\documentclass[12pt,reqno,oneside]{amsart}
\usepackage[latin1]{inputenc}
\usepackage[T1]{fontenc}
\usepackage{amsmath,amsfonts,amscd,amssymb,amsthm}
\usepackage{mathtools}
\usepackage{dsfont}
\usepackage{color}
\usepackage[mathscr]{euscript}
\usepackage{graphicx}
\usepackage{placeins}
\usepackage{extsizes}
\usepackage{float}
\usepackage{multirow}
\usepackage{pdflscape}
\usepackage{caption}
\usepackage{pgfplots}
\usepackage{subcaption}
\usepackage{mathabx}
\usepackage{hyperref}
\usepackage[final]{changes}
\usepackage{tikz}
\usetikzlibrary{decorations.pathreplacing}
\usetikzlibrary{backgrounds} 
\usetikzlibrary{calc}
\pgfdeclarelayer{bg}    % declare background layer
\pgfsetlayers{bg,main}
\usepackage{rotate}
\usepackage[noend]{algorithmic}
\usepackage{algorithm}% http://ctan.org/pkg/algorithms
\usepackage{setspace}
\usepackage{array}
\definecolor{color1bg}{HTML}{f73d28}%FA8072}
\definecolor{color2bg}{HTML}{FA8072}
\definecolor{bblue}{HTML}{00BFFF}
\definecolor{bblue2}{HTML}{00ffff}
\usetikzlibrary{arrows,calc}
\tikzset{
	>=stealth',
	help lines/.style={dashed, thick},
	axis/.style={<->},
	important line/.style={thick},
	connection/.style={thick, dotted},
}

\tikzset{
	square/.style={%
		draw=none,
		circle,
		append after command={%
			\pgfextra \draw[black] (\tikzlastnode.north-|\tikzlastnode.west) rectangle 
			(\tikzlastnode.south-|\tikzlastnode.east);\endpgfextra}
	}
}

\usetikzlibrary{shadows}
\usetikzlibrary{backgrounds}
\tikzset{
	diagonal fill/.style 2 args={fill=#2, path picture={
			\fill[#1, sharp corners] (path picture bounding box.south west) -|
			(path picture bounding box.north east) -- cycle;}},
	reversed diagonal fill/.style 2 args={fill=#2, path picture={
			\fill[#1, sharp corners] (path picture bounding box.north west) |- 
			(path picture bounding box.south east) -- cycle;}}
}
\usetikzlibrary{arrows.meta}

\makeatletter
\@addtoreset{equation}{section}
\makeatother
\newcounter{as}[section]

\graphicspath{{Images/}}

\title{Discrete Morphological Neural Networks}

\author{Diego Marcondes \and Junior Barrera}
\date{\today}

\address{Visiting Scholar, Department of Electrical and Computer Engineering, Texas A\&M University, USA and Post-doctorate Researcher, Department of Computer Science, Institute of Mathematics and University, University of São Paulo, Brazil. e-mail: \texttt{dmarcondes@ime.usp.br}}
\address{Professor, Department of Computer Science, Institute of Mathematics and University, University of São Paulo, Brazil. e-mail: \texttt{jb@ime.usp.br}}

\allowdisplaybreaks
\newtheorem{theorem}{Theorem}[section]
\newtheorem{remark}[theorem]{Remark}

\newtheorem{definition}[theorem]{Definition}

\newtheorem{corollary}[theorem]{Corollary}

\newtheorem{proposition}[theorem]{Proposition}
\newtheorem{example}[theorem]{Example}

\usepackage{lineno}
\definecolor{bblue}{rgb}{.2,0.2,.8}

\begin{document}
	\maketitle
	
\begin{abstract}
	A classical approach to designing binary image operators is Mathematical Morphology (MM). We propose the Discrete Morphological Neural Networks (DMNN) for binary image analysis to represent  W-operators and estimate them via machine learning. A DMNN architecture, which is represented by a Morphological Computational Graph, is designed as in the classical heuristic design of morphological operators, in which the designer should combine a set of MM operators and Boolean operations based on prior information and theoretical knowledge. Then, once the architecture is fixed, instead of adjusting its parameters (i.e., structural elements or maximal intervals) by hand, we propose a lattice descent algorithm (LDA) to train these parameters based on a sample of input and output images under the usual machine learning approach. We also propose a stochastic version of the LDA that is more efficient, is scalable and can obtain small error in practical problems. The class represented by a DMNN can be quite general or specialized according to expected properties of the target operator, i.e., prior information, and the semantic expressed by algebraic properties of classes of operators is a differential relative to other methods. The main contribution of this paper is the merger of the two main paradigms for designing morphological operators: classical heuristic design and automatic design via machine learning. As a proof-of-concept, we apply the DMNN to recognize the boundary of digits with noise, and we discuss many topics for future research.
\end{abstract}

\section{Introduction}

\subsection{Mathematical Morphology and W-operators}

Mathematical Morphology (MM), a theory initiated in the sixties by George Matheron \cite{matheron} and Jean Serra \cite{serra1983image1,serra1983image} at the Fontainebleau campus of the École Superieur des Mines de Paris (ESMP), has been important for the development of the field of image analysis. Lattices are the fundamental algebraic structure in MM, which studies mappings between complete lattices, and, from the sixties to the eighties, this theory evolved from Boolean to general lattice mappings \cite{serra1983image1}. These mappings are decomposed by lattice operations and two pairs of dual elementary operators: erosion and dilation, which commute, respectively, with infimum and supremum, as well as, anti-dilation and anti-erosion, which commute, respectively, with supremum-infimum and infimum-supremum. Based on these elementary operators, Banon and Barrera \cite{banon1991minimal} proposed two dual minimal representations for any lattice operator: one in terms of supremum of sup-generating operators (i.e., infimum of an erosion and an anti-dilation) and another in terms of infimum of inf-generating operators (i.e., supremum of a dilation and an anti-erosion).

A particular class of morphological operators for binary image analysis, very important in practical applications, are the W-operators \cite{barrera1996set}. Let a binary image $X$ be a subset of the image domain $E$ and let $(\mathcal{P}(E),\subseteq)$ be the Boolean lattice of the powerset of $E$. A set operator $\psi$ is a mapping from $\mathcal{P}(E)$ to $\mathcal{P}(E)$ that transforms image $X \in \mathcal{P}(E)$ into image $\psi(X)$. When $(E,+)$ is an \textit{Abelian group} with respect to a binary operation $+$, a set operator $\psi$ is a W-operator if, and only if, it is translation invariant and locally defined in a window $W \in \mathcal{P}(E)$, i.e., $\psi(X + h) = \psi(X) + h$ and 
\begin{linenomath}
	\begin{align*}
		h \in \psi(X) \iff f_{\psi}((X - h) \cap W) = 1, \text{ for } h \in E,		
	\end{align*}
\end{linenomath}
in which $f_{\psi}: \mathcal{P}(W) \to \{0,1\}$ is the characteristic function of $\psi$. In this case, the kernel of $\psi$ is $\mathcal{K}_{W}(\psi) = \{X \in \mathcal{P}(W): f_{\psi}(X) = 1\}$ and the basis of $\psi$, denoted by $\boldsymbol{B}_{W}(\psi)$, is the collection of maximal intervals in $\mathcal{K}_{W}(\psi)$. The sup-generating operators
\begin{linenomath}
	\begin{align*}
		\lambda_{[A,B]}(X) = \{h \in E: A \subseteq (X - h) \cap W \subseteq B\}, & & X \in \mathcal{P}(E),		
	\end{align*}
\end{linenomath}
associated to intervals $[A,B] \added{\coloneqq \{X \in \mathcal{P}(E): A \subseteq X \subseteq B\}} \subseteq \mathcal{P}(W)$ are the building blocks to represent any W-operator:
\begin{linenomath}
	\begin{align*}
		\psi(X) = \bigcup \left\{\lambda_{[A,B]}(X): [A,B] \in \boldsymbol{B}_{W}(\psi)\right\} & & X \in E.
	\end{align*}
\end{linenomath}
This expression has also a dual form based on inf-generating operators. In fact, any W-operator can be represented in terms of unions, intersections, compositions, and complements of erosions, dilations, anti-erosions and anti-dilations, and have several equivalent representations, but all of them can be converted into a unique basis representation \cite{banon1991minimal}. It is possible to compute incrementally the basis of any W-operator from any representation of it in terms of the elementary operators of mathematical morphology (i.e., erosion and dilation), the Boolean lattice operations, (i.e., intersection, union and complement), and the basis of the identity operator \cite{barrera1996set}.

\subsection{Design of set operators}

\deleted{From} \added{S}everal studies about mathematical morphology operators \cite{serra1983image1,serra1983image} \added{led to} \deleted{followed} the proposal of many families of operators which can be combined heuristically to solve image analysis problems. Some classes of operators include the hit-or-miss operators \cite{serra1983image1}, morphological filters \cite{serra1992overview}, and watershed segmentation \cite{beucher1982watersheds}. The research group at ESMP developed procedures to heuristically design operators, which require prior information about the problem at hand, knowledge of several classes of operators (the so-called MM toolbox), and the ability to combine them through composition and lattice operations. Once the combination of operators is performed, the designer should carefully adjust the parameters of the designed operator to obtain a solution that properly solves the practical problem.

Specialized software, whose architectures were based on basic operators and operations (i.e., erosion, dilation, infimum, supremum, complement), from which more complex operators (i.e., skeletons, filters, etc) can be built, were implemented in the nineties. In special, the researchers and engineers of the Center for Mathematical Morphology of ESMP were the pioneers in developing software and hardware for morphological image processing, and the \textit{mmach} toolbox for the Khoros system \cite{barrera1998mmach}, developed by researchers of USP, UNICAMP and INPE in Brazil, became popular at the end of the nineties. These software have had a considerable impact on the solution of a large class of image analysis problems. We refer to \cite{barrera1998mmach,dougherty2003hands,van1996fast,vincent1991efficient} for more details on practical methods for the design of MM operators.

Although the heuristic design of morphological operators had a great impact on image analysis, it is a time-consuming manual procedure which does not scale efficiently to more complex problems. Since combining basic morphological operators to form a complex image processing pipeline is not a trivial task, a natural idea is to develop methods to automatically design morphological operators. At the end of the eighties, M. Shimi\added{tt}\deleted{dt} \cite{schmitt1989mathematical} proposed a design procedure based on artificial intelligence, which combined W-operators (represented by compositions of parameterized morphological operators) and adjusted their parameters. Since the late nineties, approaches for the automatic design of morphological operators based on machine learning have been proposed, mainly by the work of Ed. Dougherty, J. Barrera and their collaborators \cite{barrera2000automatic}, and have been extensively applied in the literature with great success to solve specific problems \cite{barrera1997automatic,brun2003design,brun2004nonlinear,dellamonica2007exact,hirata2008multilevel,hirata1999design,hirata2002incremental,hirata2000iterative,hirata2002segmentation,hirata2000aperture,montagner2017staff,montagner2016kernel,santos2010information}. Many of these approaches are also heuristics, which may consider prior information about the problem at hand. However, the parameters of the operators are more efficiently adjusted by learning from data and good practical results have been obtained. More details about MM in the context of machine learning may be found in \cite{barrera2022mathematical,hirata2021machine}.

\deleted{In the nineties \cite{ritter1996introduction}, and then more recently, MM methods have been studied in connection with deep neural networks, by either combining convolutional neural networks (CNN) with a morphological operator \cite{julca2017image}, or by replacing the convolution operations of CNN with basic morphological operators, such as erosions and dilations, obtaining the so-called morphological neural networks (MNN)} \deleted{. A MNN has the general \deleted{framework} \added{structure} of neural networks, and its specificity is on the fact that the layers realize morphological operations. Although it has been seen empirically that convolutional layers could be replaced by morphological layers \cite{franchi2020deep}, and MNN have shown a better performance than CNN in some tasks \cite{hu2022learning}, MNN have lost the spirit of MM, that is a careful theoretical design of the operator based on prior information, followed by the adjustment or learning of its parameters implying a fully understanding of its properties. This detachment from classical MM made MNN opaque and as much a black-box as CNN, so no real gain in control and interpretability is obtained with MNN when compared to CNN.}

\added{Since the nineties, MM methods have been studied in connection with neural networks. The pioneering papers about morphological neural networks (MNN), such as \cite{davidson1992simulated,davidson1993morphology,davidson1990theory,ritter1996introduction}, among others, proposed neural network architectures in which the operations of multiplication and addition in a neuron are replaced by addition and maximum (or minimum), respectively. This replacement implies that the operation performed by each neuron is either an erosion or a dilation. MNN usually have the general structure of neural networks, and their specificity is on the fact that the layers realize morphological operations.}

\added{In the following years, many MNN architectures and training algorithms have been proposed for classification and image processing problems \cite{araujo2017morphological,dimitriadis2021advances,grana2001some,mondal2019morphological2,sussner2009constructive}. In special, morphological/rank neural networks have been proposed as a class of MNN which can be trained via back-propagation \cite{pessoa1996morphological,pessoa2000neural}; the dendrite MNN have been successfully applied in classification problems by partitioning the input domain by hypercubes \cite{arce2018differential,ritter2003morphological,sossa2014efficient}; and the modular MNN \cite{araujo2006modular,de2000designing} have been proposed to represent translation invariant set mappings via the minimal representation results of Banon and Barrera \cite{banon1991minimal} and Matheron \cite{matheron1974random}.  We refer to \cite{monteiro2008brief} and the references therein for a review of the early learning methods based on MNN.}

\added{More recently, MNN have been studied in the context of convolutional neural networks (CNN) and deep learning, by either combining CNN with a morphological operator \cite{julca2017image}, or by replacing the convolution operations of CNN with basic morphological operators, such as erosions and dilations \cite{franchi2020deep,groenendijk2022morphpool,mondal2019morphological}. Although it has been seen empirically that convolutional layers could be replaced by morphological layers \cite{franchi2020deep}, and MNN have shown a better performance than CNN in some tasks \cite{hu2022learning}, MNN have lost the spirit of MM, that is a careful theoretical design of the operator based on prior information, followed by the adjustment or learning of its parameters implying a fully understanding of its properties. This detachment from classical MM made MNN opaque and as much a black-box as usual neural networks, so no real gain in control and interpretability is obtained with MNN when compared to CNN.}

\subsection{Main contributions}

In this context, we propose the Discrete Morphological Neural Networks (DMNN) to represent and automatically design W-operators via machine learning. The design of a DMNN architecture, which is represented by a Morphological Computational Graph, is equivalent to the first part of the classical heuristic design of operators, in which the designer should combine basic operators and operations based on prior information and theoretical knowledge of MM. When such a combination satisfies certain axioms, it is equivalent to an architecture of a canonical DMNN. Then, once the architecture is fixed, instead of adjusting its parameters by hand, we propose the lattice \deleted{gradient} descent algorithm (L\deleted{G}DA) to train these parameters based on a sample of input and output images under the usual machine learning approach for learning W-operators. Since the L\deleted{G}DA, that is based on the U-curve algorithms \cite{u-curve3,ucurveParallel,reis2018,u-curve1}, is a combinatorial algorithm whose complexity increases exponentially with the number of computational nodes in the architecture computational graph, we also propose a stochastic version of it that, even though is suboptimal, is much more efficient, is scalable and can obtain good results in practical problems.

The main contribution of this paper is the merger of the two main paradigms for designing morphological operators: classical heuristic design and automatic design via machine learning. On the \added{one} hand, a DMNN architecture is equivalent to the manual combination of basic operators performed in heuristic design based on prior information and known properties of MM. On the other hand, the training of an architecture, or equivalently the adjustment of the parameters of the heuristic combination, is performed based on data under a machine learning framework. The adjustment of the parameters scales to more complex problems if one considers the stochastic L\deleted{G}DA, and the learning via the L\deleted{G}DA is general and not problem specific as many of the automatic design methods based on machine learning, so the DMNN is a contribution to both design paradigms. This merger is a flexible and transparent approach for automatically designing W-operators in general settings via machine learning. 

\subsection{Paper structure}

In Section \ref{Sec_lattice}, we present the main elements of Boolean lattices, in Section \ref{Sec_canonical}, we present the canonical decomposition of set operators, and in Section \ref{Sec_composition}, we summarize the main results of \cite{barrera1996set} about the composition of operators in the canonical form. Sections \ref{Sec_lattice} to \ref{Sec_composition} follow the presentation of \cite{barrera1996set}. In Section \ref{Sec_DMNN}, we define the Morphological Computation Graph and the Discrete Morphological Neural Networks, and then show that DMNN are universal representers of translation invariant and locally defined set operators. In Section \ref{Sec_CDMNN}, we particularize to the canonical DMNN, and in Section \ref{Sec_train}, we propose the lattice \deleted{gradient} descent algorithm to train them. In Section \ref{Sec_Applications}, we illustrate the proposed method with \deleted{an} \added{a proof-of-concept} application, and in Section \ref{Sec_Discussion}, we discuss the main contributions of this paper.

\section{Elements of Boolean lattices}
\label{Sec_lattice}

Let $E$ by a non-empty set, and let $W$ be a finite subset of $E$. We assume that $(E,+)$ is an \textit{Abelian group} with respect to a binary operation denoted by $+$. We denote the zero element of $E$ by $o$. Let $\mathcal{P}(W)$ be the collection of all subsets of $W$ and let $\subseteq$ be the usual inclusion relation on sets. The pair $(\mathcal{P}(W),\subseteq)$ is a complete Boolean lattice, with least and greatest element $\emptyset$ and $W$, respectively. The intersection and union of $X_{1}$ and $X_{2}$ in $\mathcal{P}(W)$ are, respectively, $X_{1} \cap X_{2}$ and $X_{1} \cup X_{2}$. 

Given $A,B \in \mathcal{P}(W)$, the sub-collection $[A,B]$ of $\mathcal{P}(W)$ defined by
\begin{linenomath}
	\begin{equation*}
		[A,B] = \{X \in \mathcal{P}(W): A \subseteq X \subseteq B\}
	\end{equation*}
\end{linenomath}
is called a closed interval, or simply interval. If $A \subseteq B$, then $A$ and $B$ are called, respectively, the left and right extremities of $[A,B]$. For all pairs $(A,B)$ such that $A \nsubseteq B$, $[A,B]$ represents the empty collection $\emptyset$.

We will denote sub-collections of $\mathcal{P}(W)$ by uppercase script letters $\mathscr{A}, \mathscr{B}, \dots, \mathscr{X}, \mathscr{Y}, \mathscr{Z}$. The collections of closed intervals will be denoted by uppercase bold face letters $\boldsymbol{A}, \boldsymbol{B},\dots, \boldsymbol{X}, \\ \boldsymbol{Y}, \boldsymbol{Z}$. 

An element of a collection of closed intervals $\boldsymbol{X}$ is called \textit{maximal in} $\boldsymbol{X}$ if no other element of $\boldsymbol{X}$ properly contains it, that is, $\forall [A,B] \in \boldsymbol{X}$,
\begin{linenomath}
	\begin{equation*}
		[A,B] \text{ is maximal in } \boldsymbol{X} \iff \forall [A^\prime,B^\prime] \in \boldsymbol{X}, [A,B] \subseteq [A^\prime,B^\prime] \implies [A,B] = [A^\prime,B^\prime].
	\end{equation*}
\end{linenomath}
The collection of maximal closed intervals in $\boldsymbol{X}$ is denoted by $\text{Max}(\boldsymbol{X})$.

Let $\mathscr{X}$ be a sub-collection of $\mathcal{P}(W)$. The collection of all maximal closed intervals contained in $\mathscr{X}$ is denoted by $\boldsymbol{M}(\mathscr{X})$ and is given by
\begin{linenomath}
	\begin{equation*}
		\boldsymbol{M}(\mathscr{X}) = \text{Max}\left(\{[A,B] \subseteq \mathcal{P}(W): [A,B] \subseteq \mathscr{X}\}\right).
	\end{equation*}
\end{linenomath}
Usually, we will denote a sub-collection and the set of maximal intervals contained in it by the same letter, for example $\boldsymbol{X} = \boldsymbol{M}(\mathscr{X})$.

We denote by $\cup\boldsymbol{X}$ the collection of all elements of $\mathcal{P}(W)$ that are elements of closed intervals in $\boldsymbol{X}$, that is
\begin{linenomath}
	\begin{equation*}
		\cup\boldsymbol{X} = \{X \in \mathcal{P}(W): X \in [A,B], [A,B] \in \boldsymbol{X}\}.
	\end{equation*}
\end{linenomath}
Note that, for any $\mathscr{X} \subseteq \mathcal{P}(W)$, $\cup \boldsymbol{X} = \mathscr{X}$. Since $W$ is finite and $\mathscr{X} \subseteq \mathcal{P}(W)$, for any $[A,B] \subseteq \mathscr{X}$, there exists $[A^\prime,B^\prime] \in \boldsymbol{M}(\mathscr{X})$ such that $[A,B] \subseteq [A^\prime,B^\prime]$, and hence any sub-collection $\mathscr{X}$ can be represented by its maximal closed intervals.

Let $\mathcal{P}(\mathcal{P}(W))$ be the collection of all sub-collections of $\mathcal{P}(W)$ and let $\subseteq$ be the usual inclusion relation on sets. The pair $(\mathcal{P}(\mathcal{P}(W)),\subseteq)$ is a complete Boolean lattice. The least and the greatest element of $\mathcal{P}(\mathcal{P}(W))$ are, respectively, $\emptyset$ and $\mathcal{P}(W)$. The intersection and the union of $\mathscr{X}_{1}$ and $\mathscr{X}_{2}$ in $\mathcal{P}(\mathcal{P}(W))$ are, respectively, $\mathscr{X}_{1} \cap \mathscr{X}_{2}$ and $\mathscr{X}_{1} \cup \mathscr{X}_{2}$. The complementary collection of a sub-collection $\mathscr{X}$ in $\mathcal{P}(\mathcal{P}(W))$, with respect to $\mathcal{P}(W)$, is defined as $\mathscr{X}^{c} = \{X \in \mathcal{P}(W): X \notin \mathscr{X}\}$.

Let $\Pi_{W} \coloneqq \{\boldsymbol{M}(\mathscr{X}): \mathscr{X} \subseteq \mathcal{P}(W)\}$ be the collection of the sub-collections of maximal intervals of $\mathcal{P}(W)$. We will define the partial order $\leq$ on the elements of $\Pi_{W}$ by setting, $\forall \boldsymbol{X}, \boldsymbol{Y} \in \Pi_{W}$,
\begin{linenomath}
	\begin{equation*}
		\boldsymbol{X} \leq \boldsymbol{Y} \iff \forall [A,B] \in \boldsymbol{X}, \exists [A^\prime,B^\prime] \in \boldsymbol{Y}: [A,B] \subseteq [A^\prime,B^\prime].
	\end{equation*}
\end{linenomath}
The \added{partially ordered set (}poset\added{)} $(\Pi_{W},\leq)$ constitutes a complete Boolean lattice, with infimum, supremum and complement operations given, respectively, by
\begin{linenomath}
	\begin{align*}
		\boldsymbol{X} \sqcap \boldsymbol{Y} = \boldsymbol{M}(\mathscr{X} \cap \mathscr{Y}) & & \boldsymbol{X} \sqcup \boldsymbol{Y} = \boldsymbol{M}(\mathscr{X} \cup \mathscr{Y}) & & \bar{\boldsymbol{X}} = \boldsymbol{M}(\mathscr{X}^{c})
	\end{align*}
\end{linenomath}
for all $\boldsymbol{X}, \boldsymbol{Y} \in \Pi_{W}$, in which $\boldsymbol{X} = \boldsymbol{M}(\mathscr{X})$ and $\boldsymbol{Y} = \boldsymbol{M}(\mathscr{Y})$. These expressions follow since the mapping $\boldsymbol{M}(\cdot)$, defined from $\mathcal{P}(\mathcal{P}(W))$ to $\Pi_{W}$, is a lattice isomorphism. The inverse of the mapping $\boldsymbol{M}(\cdot)$ is the mapping $\cup(\cdot)$. In particular, the least and the greatest element of $(\Pi_{W},\leq)$ are, respectively, $\boldsymbol{M}(\emptyset) = \{\emptyset\}$ and $\boldsymbol{M}(\mathcal{P}(W)) = \{[\emptyset,W]\}$.

Let $\mathfrak{B}_{W} \coloneqq \{0,1\}^{\mathcal{P}(W)}$ denote the set of all Boolean functions defined on $\mathcal{P}(W)$. The pair $(\mathfrak{B}_{W},\leq)$, where $\leq$ is the partial order inherited from the total order in $\{0,1\}$, constitutes a complete Boolean lattice. This lattice is isomorphic to the lattice $(\Pi_{W},\leq)$, since the mapping $F$ defined from $\Pi_{W}$ to $\mathfrak{B}_{W}$ by
\begin{linenomath}
	\begin{align*}
		F(\boldsymbol{X})(X) = \begin{cases}
			1, &\text{ if } X \in \cup \boldsymbol{X},\\
			0, &\text{ otherwise}
		\end{cases} & & X \in \mathcal{P}(W)
	\end{align*}
\end{linenomath}
is a lattice isomorphism. The inverse mapping $F$ is the mapping $F^{-1}$ defined by
\begin{linenomath}
	\begin{align*}
		F^{-1}(f) = \boldsymbol{M}\left(\{X \in \mathcal{P}(W): f(X) = 1\}\right) & & f \in \mathfrak{B}_{W}.
	\end{align*}
\end{linenomath}
Figure \ref{lattice_iso} illustrates the lattice isomorphisms between $(\mathcal{P}(\mathcal{P}(W)),\subseteq), (\Pi_{W},\leq)$ and $(\mathfrak{B}_{W},\leq)$.

\section{Canonical Decomposition of Set Operators}
\label{Sec_canonical}

A \textit{set operator} is any mapping defined from $\mathcal{P}(E)$ into itself. We will denote set operators by lowercase Greek letters $\alpha,\beta,\gamma,\delta,\epsilon,\psi,\phi,\\ \theta,\lambda,\mu,...$ and collections of set operators by uppercase Greek letters $\Psi,\Theta,\Gamma,\Lambda,\Delta,\Phi,\Omega,\Sigma,...$. The set $\Psi$ of all the operators from $\mathcal{P}(E)$ to $\mathcal{P}(E)$ inherits the complete lattice structure of $(\mathcal{P}(E),\subseteq)$ by setting, $\forall \psi_{1},\psi_{2} \in \Psi$,
\begin{linenomath}
	\begin{align}
		\label{partial_Psi}
		\psi_{1} \leq \psi_{2} \iff \psi_{1}(X) \subseteq \psi_{2}(X), \forall X \in \mathcal{P}(E).
	\end{align}
\end{linenomath}
The supremum and infimum of a subset $\Theta \subseteq \Psi$ of the complete lattice $(\Psi,\leq)$ satisfy
\begin{linenomath}
	\begin{align*}
		\left(\bigvee \Theta\right)(X) = \bigcup \{\theta(X): \theta \in \Theta\} & & X \in \mathcal{P}(E),
	\end{align*}
\end{linenomath}
and 
\begin{linenomath}
	\begin{align*}
		\left(\bigwedge \Theta\right)(X) = \bigcap \{\theta(X): \theta \in \Theta\} & & X \in \mathcal{P}(E),
	\end{align*}
\end{linenomath}
respectively.

For any $h \in E$ and $X \in \mathcal{P}(E)$, the set
\begin{linenomath}
	\begin{equation*}
		X + h \coloneqq \{x \in E: x - h \in X\}
	\end{equation*}
\end{linenomath}
is called the translation of $X$ by $h$. We may also denote $X + h$ as $X_{h}$ and, in special, $X_{o} = X$ in which $o$ is the origin of $E$. Let $X^{t}$ be the transpose of $X \in \mathcal{P}(E)$ defined as
\begin{linenomath}
	\begin{equation*}
		X^{t} = \{y \in E: y = -x,x \in X\}.
	\end{equation*}
\end{linenomath} 

A set operator $\psi$ is called \textit{translation invariant} (t.i) if, and only if, $\forall h \in E$,
\begin{linenomath}
	\begin{align*}
		\psi(X + h) = \psi(X) + h & & X \in \mathcal{P}(E).
	\end{align*}
\end{linenomath}
Let $W$ be a finite subset of $E$. A set operator $\psi$ is called \textit{locally defined within a window} $W$ if, and only if, $\forall h \in E$,
\begin{linenomath}
	\begin{align*}
		h \in \psi(X) \iff h \in \psi(X \cap W_{h}) & & X \in \mathcal{P}(E).
	\end{align*}
\end{linenomath}
Let $\Psi_{W}$ denote the collection of t.i. operators locally defined within a window $W \in \mathcal{P}(E)$ and let
\begin{linenomath}
	\begin{equation*}
		\Omega = \bigcup\limits_{\substack{W \in \mathcal{P}(E)\\|W| < \infty}} \Psi_{W}
	\end{equation*}
\end{linenomath}
be the subset of $\Psi$ of all operators from $\mathcal{P}(E)$ to $\mathcal{P}(E)$ that are t.i. and locally defined within some finite window $W \in \mathcal{P}(E)$. The elements of $\Omega$ are called $W$\textit{-operators}. For each $W \in \mathcal{P}(E)$, the pair $(\Psi_{W},\leq)$ constitutes a finite sub-lattice of the lattice $(\Psi,\leq)$.

The kernel $\mathcal{K}_{W}(\psi)$ of a $W$-operator $\psi$ is the sub-collection of $\mathcal{P}(W)$ defined by
\begin{linenomath}
	\begin{align}
		\label{kernel}
		\mathcal{K}_{W}(\psi) = \{X \in \mathcal{P}(W): o \in \psi(X)\} & & \psi \in \Psi_{W}.
	\end{align}
\end{linenomath}
The mapping $\mathcal{K}_{W}$ from $\Psi_{W}$ to $\mathcal{P}(\mathcal{P}(W))$ defined by \eqref{kernel} constitutes a lattice isomorphism between the lattices $(\Psi_{W},\leq)$ and $(\mathcal{P}(\mathcal{P}(W)),\subseteq)$. The inverse of the mapping $\mathcal{K}_{W}$ is the mapping $\mathcal{K}_{W}^{-1}$ defined by
\begin{linenomath}
	\begin{align*}
		\mathcal{K}_{W}^{-1}(\mathscr{X})(X) = \{x \in E: X_{-x} \cap W \in \mathscr{X}\} & & \mathscr{X} \in \mathcal{P}(\mathcal{P}(W)), X \in \mathcal{P}(E).  
	\end{align*}
\end{linenomath}
As a consequence of this isomorphism, we have that, $\forall \psi_{1},\psi_{2} \in \Psi_{W}$,
\begin{linenomath}
	\begin{align*}
		\mathcal{K}_{W}(\psi_{1} \wedge \psi_{2}) = \mathcal{K}_{W}(\psi_{1}) \cap \mathcal{K}_{W}(\psi_{2}) & & \text{ and } & & \mathcal{K}_{W}(\psi_{1} \vee \psi_{2}) = \mathcal{K}_{W}(\psi_{1}) \cup \mathcal{K}_{W}(\psi_{2}).
	\end{align*}
\end{linenomath}
See \cite[Proposition~4.1]{barrera1996set} for more details.

Define by $\nu$ and $\iota$ the set operators
\begin{linenomath}
	\begin{align*}
		\iota(X) = X & & \nu(X) = X^{c}
	\end{align*}
\end{linenomath}
for $X \in \mathcal{P}(E)$, called, respectively, the identity and the complement operators, in which the complement of $X$ is with respect to $E$: $X^{c} = \{x \in E: x \notin X\}$.

For $A,B \in \mathcal{P}(E)$, the operations
\begin{linenomath}
	\begin{align*}
		A \oplus B = \bigcup\limits_{b \in B} A_{b} & & \text{ and } & & A \ominus B = \bigcap\limits_{b \in B} A_{-b}
	\end{align*}
\end{linenomath}
are called, respectively, \textit{Minkowski addition} and \textit{subtraction}. For $B \in \mathcal{P}(W)$, the t.i. set operators $\delta_{B}$ and $\epsilon_{B}$ defined by
\begin{linenomath}
	\begin{align*}
		\delta_{B}(X) = X \oplus B & & X \in \mathcal{P}(E)
	\end{align*}
\end{linenomath}
and 
\begin{linenomath}
	\begin{align*}
		\epsilon_{B}(X) = X \ominus B & & X \in \mathcal{P}(E)
	\end{align*}
\end{linenomath}
are called, respectively, \textit{dilation} and \textit{erosion} by $B$. The set $B$ is called structural element.

Let $A,B \in \mathcal{P}(W)$ be such that $A \subseteq B$. The t.i. set operators $\lambda_{[A,B]}^{W}$ and $\mu_{[A,B]}^{W}$ defined by
\begin{linenomath}
	\begin{align*}
		\lambda_{[A,B]}^{W}(X) = \{x \in E: A \subseteq X_{-x} \cap W \subseteq B\} & & X \in \mathcal{P}(E)
	\end{align*}
\end{linenomath}
and 
\begin{linenomath}
	\begin{align*}
		\mu_{[A,B]}^{W}(X) = \left\{x \in E: X_{-x} \cap A^{t} \neq \emptyset \text{ or } \left[X_{-x} \cap W^{t}\right] \cup B^{t} \neq W^{t}\right\} & & X \in \mathcal{P}(E)
	\end{align*}
\end{linenomath}
are called, respectively, sup-generating  and inf-generating operators. These operators are locally defined within, respectively, $W$ and $W^{t}$. The sup-generating and inf-generating operators can be decomposed in terms of erosions and dilations, respectively, by
\begin{linenomath}
	\begin{align}
		\label{formula_supGen}
		\lambda_{[A,B]}^{W}(X) = \epsilon_{A}(X) \cap \nu\delta_{B^{tc}}(X) & & X \in \mathcal{P}(E)
	\end{align}
\end{linenomath}
and
\begin{linenomath}
	\begin{align}
		\label{formula_infGen}
		\mu_{[A,B]}^{W}(X) = \delta_{A}(X) \cup \nu\epsilon_{B^{ct}}(X) & & X \in \mathcal{P}(E)
	\end{align}
\end{linenomath}
in which the complement of $B$ is taken relative to $W$. The operators $\nu\delta_{B^{tc}}$ and $\nu\epsilon_{B^{ct}}$ are called, respectively, anti-dilation and anti-erosion.

The dual operator of $\psi$, denoted by $\psi^{\star}$, is defined as
\begin{linenomath}
	\begin{equation*}
		\psi^{\star} = \nu\psi\nu.
	\end{equation*}
\end{linenomath}
The next result, due to a specialization of the results in \cite{banon1991minimal} proved in \cite{barrera1996set}, presents two decompositions of a $W$-operator called canonical sup-decomposition and canonical inf-decomposition.

\begin{proposition}
	\label{prop_canonical_decomposition}
	Let $\psi$ be a $W$-operator. Then $\psi$ can be decomposed as
	\begin{linenomath}
		\begin{align*}
			\psi(X) = \bigcup \left\{\lambda_{[A,B]}^{W}(X): [A,B] \subset \mathcal{K}_{W}(\psi)\right\} & & X \in \mathcal{P}(E)
		\end{align*}
	\end{linenomath}
	and
	\begin{linenomath}
		\begin{align*}
			\psi(X) = \bigcap \left\{\mu_{[A^{t},B^{t}]}^{W^{t}}(X): [A,B] \subset \mathcal{K}_{W}(\psi^{\star})\right\} & & X \in \mathcal{P}(E),
		\end{align*}
	\end{linenomath}
	called, respectively, canonical sup-decomposition and canonical inf-decomposition of $\psi$.
\end{proposition}

These decompositions are quite general, but there may exist a smaller family of sup-generating or inf-generating operators that represents a same operator. Indeed, let 
\begin{linenomath}
	\begin{align*}
		\boldsymbol{B}_{W}(\psi) \coloneqq \boldsymbol{M}\left(\mathcal{K}_{W}(\psi)\right) & & \psi \in \Psi_{W}
	\end{align*}
\end{linenomath}
be the collection of all maximal closed intervals contained in $\mathcal{K}_{W}(\psi)$. We call $\boldsymbol{B}_{W}(\psi)$ the \textit{basis} of $\psi$ and obtain the following simplification of the decompositions in Proposition \ref{prop_canonical_decomposition}.

\begin{corollary}
	\label{corollary_canonical_basis}
	Let $\psi$ be a $W$-operator. Then $\psi$ can be decomposed as
	\begin{linenomath}
		\begin{align*}
			\psi(X) = \bigcup \left\{\lambda_{[A,B]}^{W}(X): [A,B] \in \boldsymbol{B}_{W}(\psi)\right\} & & X \in \mathcal{P}(E)
		\end{align*}
	\end{linenomath}
	and
	\begin{linenomath}
		\begin{align*}
			\psi(X) = \bigcap \left\{\mu_{[A^{t},B^{t}]}^{W^{t}}(X): [A,B] \in \boldsymbol{B}_{W}(\psi^{\star})\right\} & & X \in \mathcal{P}(E).
		\end{align*}
	\end{linenomath}
\end{corollary}

From Proposition \ref{prop_canonical_decomposition} and Corollary \ref{corollary_canonical_basis} follow that $W$-operators are uniquely determined, or parametrized, by their kernel or their basis. As a consequence of the isomorphisms in Figure \ref{lattice_iso}, a $W$-operator can also be represented by a Boolean function, as follows.

Let $T$ be the mapping between $\Psi_{W}$ and $\mathfrak{B}_{W}$ defined by
\begin{linenomath}
	\begin{align*}
		T(\psi)(X) = \begin{cases}
			1, & \text{ if } o \in \psi(X)\\
			0, & \text{ otherwise}.
		\end{cases} & & \psi \in \Psi_{W}, X \in \mathcal{P}(W).
	\end{align*}
\end{linenomath}
The mapping $T$ constitutes a lattice isomorphism between the complete lattices $(\Psi_{W},\leq)$ and $(\mathfrak{B}_{W},\leq)$, and its inverse $T^{-1}$ is defined by
\begin{linenomath}
	\begin{align*}
		T^{-1}(f)(X) = \left\{x \in E: f(X_{-x} \cap W) = 1\right\} & & f \in \mathfrak{B}_{W}, X \in \mathcal{P}(E).
	\end{align*}
\end{linenomath}

We note that underlying all these decomposition results are the isomorphisms between $(\Psi_{W},\leq), (\mathcal{P}(\mathcal{P}(W)),\leq), (\Pi_{W},\leq)$ and $(\mathfrak{B}_{W},\leq)$ depicted in Figure \ref{lattice_iso}.

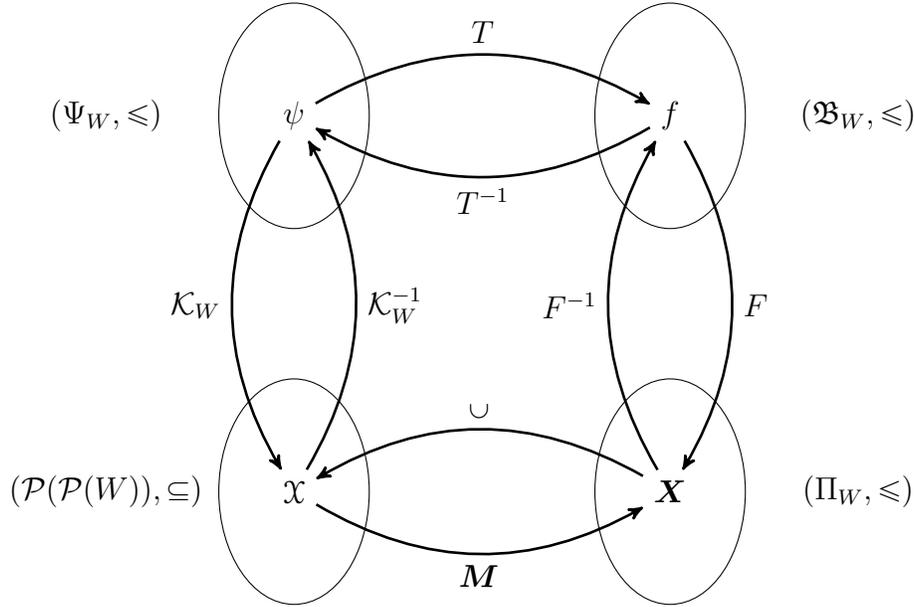
\begin{figure}[ht]
	\centering
	\begin{tikzpicture}
		\tikzstyle{hs} = [circle,draw=black, rounded corners,minimum width=3em, vertex distance=2.5cm, line width=1pt]
		\tikzstyle{hs2} = [circle,draw=black,dashed, rounded corners,minimum width=3em, vertex distance=2.5cm, line width=1pt]
		
		\draw (0,0) ellipse (1cm and 1.5cm);
		\draw (5,0) ellipse (1cm and 1.5cm);
		\draw (0,-5) ellipse (1cm and 1.5cm);
		\draw (5,-5) ellipse (1cm and 1.5cm);
		
		\node (psi) at (0,0) {$\psi$};
		\node (f) at (5,0) {$f$};
		\node (IntX) at (5,-5) {$\boldsymbol{X}$};
		\node (X) at (0,-5) {$\mathscr{X}$};
		
		\node at (-2.5,0) {$(\Psi_{W},\leq)$};
		\node at (7.5,0) {$(\mathfrak{B}_{W},\leq)$};
		\node at (7.5,-5) {$(\Pi_{W},\leq)$};
		\node at (-2.5,-5) {$(\mathcal{P}(\mathcal{P}(W)),\subseteq)$};

		\begin{scope}[line width=1pt]
			\draw[->] (psi) to[bend left]  node[sloped,above] {$T$} (f);
			\draw[->] (f) to[bend left]  node[sloped,below] {$T^{-1}$} (psi);
			
			\draw[->] (psi) to[bend right]  node[left] {$\mathcal{K}_{W}$} (X);
			\draw[->] (X) to[bend right]  node[right] {$\mathcal{K}_{W}^{-1}$} (psi);
			
			\draw[->] (f) to[bend left]  node[right] {$F$} (IntX);
			\draw[->] (IntX) to[bend left]  node[left] {$F^{-1}$} (f);
			
			\draw[->] (X) to[bend right]  node[sloped,below] {$\boldsymbol{M}$} (IntX);
			\draw[->] (IntX) to[bend right]  node[sloped,above] {$\cup$} (X);
		\end{scope}
	\end{tikzpicture}
	\caption{\footnotesize The lattice isomorphisms between representations of W-operators.} \label{lattice_iso}
\end{figure}

\section{Composition of Operators in the Canonical Form}
\label{Sec_composition}

The basic operations between set operators are supremum, infimum, and composition. In this section, we summarize the main results of \cite{barrera1996set} regarding the basis of set operators obtained via such operations.

The first result are assertions about the window of the supremum, infimum, and composition of set operators. For any $h \in E$ let
\begin{linenomath}
	\begin{align*}
		\tau_{h}(X) = X + h & & X \in \mathcal{P}(E)
	\end{align*}
\end{linenomath}
be the \textit{translation by $h$ operator}.

\begin{proposition}
	\label{prop_window_combination}
	Let $\psi \in \Psi_{W}, \psi_{1} \in \Psi_{W_{1}}$ and $\psi_{2} \in \Psi_{W_{2}}$ be t.i. operators locally defined within windows $W,W_{1},W_{2} \in \mathcal{P}(E)$, respectively. The following hold:
	\begin{itemize}
		\item[(a)] The operator $\psi$ is locally defined within any window $W^{\prime} \supseteq W$.
		\item[(b)] The operators $\psi_{1} \vee \psi_{2}$ and $\psi_{1} \wedge \psi_{2}$ are t.i. and locally defined within the window $W_{1} \cup W_{2}$.
		\item[(c)] For any $h \in E$, $\tau_{h}\psi$ is t.i. and locally defined within the window $W - h$.
		\item[(d)] For $B \in \mathcal{P}(E)$, the operators $\delta_{B}\psi$ and $\epsilon_{B}\psi$ are t.i. and locally defined, respectively, within windows $W \oplus B^{t}$ and $W \oplus B$.
		\item[(e)] The operator $\psi_{2}\psi_{1}$ is t.i. and locally defined within the window $W_{1} \oplus W_{2}$.
	\end{itemize} 
\end{proposition}

From Proposition \ref{prop_window_combination} \added{it} follows that the collection $\Omega$ of all t.i. and locally defined set operators is closed under finite composition, infimum, and supremum.

\begin{corollary}
	\label{corollary_Omega_closed}
	The set $\Omega$ of all translation invariant and locally defined set operators from $\mathcal{P}(E)$ to $\mathcal{P}(E)$ is closed under finite composition, infimum, and supremum.
\end{corollary}

The next result presents the basis of the supremum, infimum, and composition of set operators. See \cite{barrera1996set} for a proof.

\begin{proposition}
	\label{prop_basis_combination}
	Let $\psi \in \Psi_{W}, \psi_{1} \in \Psi_{W_{1}}$ and $\psi_{2} \in \Psi_{W_{2}}$ be t.i. operators locally defined within windows $W,W_{1},W_{2} \in \mathcal{P}(E)$, respectively. The following hold:
	\begin{itemize}
		\item[(a)] $\boldsymbol{B}_{W_{1} \cup W_{2}}(\psi_{1} \wedge \psi_{2}) = \boldsymbol{B}_{W_{1} \cup W_{2}}(\psi_{1}) \bigsqcap \boldsymbol{B}_{W_{1} \cup W_{2}}(\psi_{2})$
		\item[(b)] $\boldsymbol{B}_{W_{1} \cup W_{2}}(\psi_{1} \vee \psi_{2}) = \boldsymbol{B}_{W_{1} \cup W_{2}}(\psi_{1}) \bigsqcup \boldsymbol{B}_{W_{1} \cup W_{2}}(\psi_{2})$
		\item[(c)] $\boldsymbol{B}_{W}(\nu\psi) = \bar{\boldsymbol{B}}_{W}(\psi)$
		\item[(d)] $\boldsymbol{B}_{W-h}(\tau_{h}\psi) = \boldsymbol{B}_{W}(\psi) - h$
		\item[(e)] For $B \in \mathcal{P}(E)$, $\boldsymbol{B}_{W \oplus B^{t}}(\delta_{B}\psi) = \bigsqcup\limits_{b \in B} \boldsymbol{B}_{W \oplus B^{t}}(\tau_{b}\psi)$
		\item[(f)] For $B \in \mathcal{P}(E)$, $\boldsymbol{B}_{W \oplus B}(\epsilon_{B}\psi) = \bigsqcap\limits_{b \in B^{t}} \boldsymbol{B}_{W \oplus B}(\tau_{b}\psi)$
		\item[(g)] Denoting $\boldsymbol{B}_{W_{2}}(\psi_{2}) = \{[A_{i},B_{i}]: i \in \mathfrak{I}\}$, then
		\begin{linenomath}
			\begin{equation*}
				\boldsymbol{B}_{W_{1} \oplus W_{2}}(\psi_{2}\psi_{1}) = \bigsqcup \left\{\boldsymbol{B}_{W_{1} \oplus W_{2}}(\epsilon_{A_{i}}\psi_{1}) \sqcap \bar{\boldsymbol{B}}_{W_{1} \oplus W_{2}}(\delta_{B_{i}^{ct}}\psi_{1}): i \in \mathfrak{I}\right\}
			\end{equation*}
		\end{linenomath}
	\end{itemize}
\end{proposition}

\section{Discrete Morphological Neural Networks}
\label{Sec_DMNN}

A Discrete Morphological Neural Network (DMNN) can be represented as a computational graph with vertices that compute either a $W$-operator, or one of the two basic supremum and infimum operations. A DMNN realizes the $W$-operator computed by its computational graph. We start by defining the Morphological Computational Graphs that represent DMNN. Then, we define a DMNN architecture and the space of $W$-operators realized by it, and show that any $W$-operator may be non-trivially represented by a DMNN. We end this section showing that DMNN are universal representers of t.i and locally defined set operators.

\subsection{Morphological Computational Graph}

Let $\mathcal{G} = (\mathcal{V},\mathcal{E},\mathcal{C})$ be a \textit{computational graph}, in which $\mathcal{V}$ is a general set of vertices, $\mathcal{E} \subset \{(\mathfrak{v}_{1},\mathfrak{v}_{2}) \in \mathcal{V} \times \mathcal{V}: \mathfrak{v}_{1} \neq \mathfrak{v}_{2}\}$ is a set of directed edges, and $\mathcal{C}: \mathcal{V} \to \Omega \cup \{\vee,\wedge\}$ is a mapping that associates each vertex $\mathfrak{v} \in \mathcal{V}$ to a \textit{computation} given by either applying a t.i. and locally defined operator $\psi \in \Omega$ or one of the two basic operations $\{\vee,\wedge\}$.

For each $\mathfrak{v} \in \mathcal{V}$ we denote by
\begin{linenomath}
	\begin{equation*}
		\mathds{I}(\mathfrak{v}) = \left\{\mathfrak{a} \in \mathcal{V}: (\mathfrak{a},\mathfrak{v}) \in \mathcal{E}\right\}
	\end{equation*}
\end{linenomath}
the vertices in $\mathcal{V}$ connected with $\mathfrak{v}$ by an edge ending in $\mathfrak{v}$ (input vertices of $\mathfrak{v}$), and
\begin{linenomath}
	\begin{equation*}
		\mathds{O}(\mathfrak{v}) = \left\{\mathfrak{a} \in \mathcal{V}: (\mathfrak{v},\mathfrak{a}) \in \mathcal{E}\right\}
	\end{equation*}
\end{linenomath}
as the vertices in $\mathcal{V}$ connected with $\mathfrak{v}$ by an edge starting in $\mathfrak{v}$ (output vertices of $\mathfrak{v}$).

We say that $\mathcal{G}$ is a Morphological Computational Graph (MCG) when the following conditions are satisfied.

\vspace{0.25cm}

\noindent \textbf{Axioms of Morphological Computational Graphs} A computational graph $\mathcal{G} = (\mathcal{V},\mathcal{E},\mathcal{C})$ is a Morphological Computational Graph if, and only if, \deleted{all} the following conditions hold:
\begin{itemize}
	\item[] \textbf{A1}: The directed graph $(\mathcal{V},\mathcal{E})$ is acyclic and the number of vertices is finite and greater than 2, i.e., $2 < |\mathcal{V}| < \infty$;
	\item[] \textbf{A2}: There exists a unique $\mathfrak{v}_{i} \in \mathcal{V}$ with $|\mathds{I}(\mathfrak{v}_{i})| = 0$ and it computes the identity operator, i.e., $\mathcal{C}(\mathfrak{v}_{i}) = \iota$;
	\item[] \textbf{A3}: There exists a unique $\mathfrak{v}_{o} \in \mathcal{V}$ with $|\mathds{O}(\mathfrak{v}_{o})| = 0$ and it computes the identity operator, i.e., $\mathcal{C}(\mathfrak{v}_{o}) = \iota$;
	\item[] \textbf{A4}: If $\mathcal{C}(\mathfrak{v}) \in \Omega, \added{\mathfrak{v} \neq \mathfrak{v}_{i}}$, then $|\mathds{I}(\mathfrak{v})| = 1$;
	\item[] \textbf{A5}: If $\mathcal{C}(\mathfrak{v}) \in \{\vee,\wedge\}$, then $|\mathds{I}(\mathfrak{v})| \geq 2$.
	\item[] \deleted{\textbf{A6}: For all $\mathfrak{v} \in \mathcal{V}, \mathfrak{v} \neq \mathfrak{v}_{o}$, $|\mathds{O}(\mathfrak{v})| \geq 1$.}
\end{itemize}

\vspace{0.25cm}

The computation of a vertex $\mathfrak{v}$ in $\mathcal{G}$ receives as input the output of the computation of the vertices in $\mathds{I}(\mathfrak{v})$, and the output of its computation will be used as the input of the computation of the vertices in $\mathds{O}(\mathfrak{v})$. We assume there is an input vertex $\mathfrak{v}_{i}$, and an output vertex $\mathfrak{v}_{o}$, that store the input, which is an element $X \in \mathcal{P}(E)$, and output of the computational graph, respectively. These vertices cannot have, respectively, an edge ending and starting in them (\textbf{A2} and \textbf{A3}). The directed graph $(\mathcal{V},\mathcal{E})$ should be acyclic, so each computation is performed at most once in any path of $(\mathcal{V},\mathcal{E})$, and there should be at least three vertices in $\mathcal{V}$ so it is not formed solely by the input and output vertices; there must be a finite number of vertices so the computation of $\mathcal{G}$ is possible (\textbf{A1}). 

On the one hand, if $\mathcal{C}(\mathfrak{v}) \in \Omega$, then the computation of $\mathfrak{v}$ is given by applying a t.i. and locally defined set operator which has only one input, so it must hold $|\mathds{I}(\mathfrak{v})| = 1$ (\textbf{A4}). On the other hand, if $\mathcal{C}(\mathfrak{v}) \in \{\vee,\wedge\}$, then the computation of $\mathfrak{v}$ is given by taking the infimum or the supremum of the output of at least two preceding computations, so it must hold $|\mathds{I}(\mathfrak{v})| \geq 2$ (\textbf{A5}). Apart from the output vertex $\mathfrak{v}_{o}$, all other vertices should have at least one edge starting in them, so their computation is the input of at least one other vertex \added{(consequence of \textbf{A3})}. \deleted{\textbf{(A6)}} Table \ref{tab_dic} summarizes the types of vertices of a MCG.

\begin{table}[ht]
	\centering
	\caption{\footnotesize Types of vertices of a Morphological Computational Graph.} \label{tab_dic}
	\resizebox{\linewidth}{!}{\begin{tabular}{clcccl}
			\hline
			Symbol & Name & $\mathcal{C}(\mathfrak{v})$ & $|\mathds{I}(\mathfrak{v})|$ & $|\mathds{O}(\mathfrak{v})|$ & Description \\
			\hline
			\raisebox{-.5\height}{\includegraphics[scale=0.2]{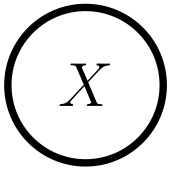}} & Input vertex $\mathfrak{v}_{i}$ & $\iota$ & 0 & $\geq 1$ & Receive and store the input $X \in \mathcal{P}(E)$ of $\mathcal{G}$ \\
			\raisebox{-.5\height}{\includegraphics[scale=0.2]{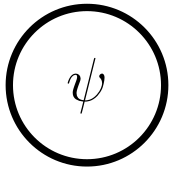}} & Set operator vertex & $\psi \in \Omega$ & $1$ & $\geq 1$ & Apply operator $\psi = \mathcal{C}(\mathfrak{v}) \in \Omega$ to its input\\
			\raisebox{-.5\height}{\includegraphics[scale=0.2]{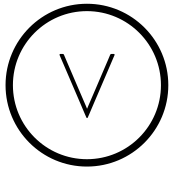}} & Supremum vertex & $\vee$ & $\geq 2$ & $\geq 1$ & Take the supremum of its inputs\\
			\raisebox{-.5\height}{\includegraphics[scale=0.2]{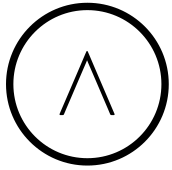}} & Infimum vertex & $\wedge$ & $\geq 2$ & $\geq 1$ & Take the infimum of its inputs\\
			\raisebox{-.5\height}{\includegraphics[scale=0.2]{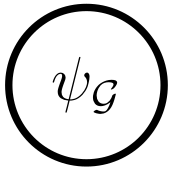}} & Output vertex $\mathfrak{v}_{o}$ & $\iota$ & 1 & 0 & Receive and store the output $\psi_{\mathcal{G}}(X)$ of $\mathcal{G}$ \\
			\hline
			\hline
	\end{tabular}}
\end{table}

Since the input of the vertex $\mathfrak{v}_{i}$ is an element of $\mathcal{P}(E)$, the output and the input of all computations performed by vertices in $\mathcal{G}$ are elements of $\mathcal{P}(E)$. Denoting by $\psi_{\mathcal{G}}(X)$ the output of vertex $\mathfrak{v}_{o}$ when the input of vertex $\mathfrak{v}_{i}$ is $X \in \mathcal{P}(E)$, we have that the MCG $\mathcal{G}$ generates a set operator $\psi_{\mathcal{G}}: \mathcal{P}(E) \to \mathcal{P}(E)$. Actually, it follows from Corollary \ref{corollary_Omega_closed} that $\psi_{\mathcal{G}}$ is actually t.i. and locally defined within a window $W_{\mathcal{G}}$.

\begin{proposition}
	\label{prop_graph_Woperator}
	Let $\mathcal{G} = (\mathcal{V},\mathcal{E},\mathcal{C})$ be a Morphological Computational Graph. Then $\psi_{\mathcal{G}}$ is translation invariant and locally defined within a finite window $W_{\mathcal{G}} \in \mathcal{P}(E)$.
\end{proposition}
\begin{proof}
	Fix $X \in \mathcal{P}(E)$ and let 
	\begin{linenomath}
		\begin{equation*}
			\Omega_{\mathcal{G}} = \left\{\psi \in \Omega: \exists \mathfrak{v} \in \mathcal{V} \text{ s.t. } \mathcal{C}(\mathfrak{v}) = \psi\right\} \subseteq \Omega
		\end{equation*}
	\end{linenomath}
	be the finite collection of t.i. and locally defined set operators applied by vertices in $\mathcal{V}$. Since $\psi_{\mathcal{G}}(X)$ is obtained by combining the set operators in $\Omega_{\mathcal{G}}$ via the composition, supremum and infimum operations, and $\Omega$ is closed under finite composition, infimum and supremum (cf. Corollary \ref{corollary_Omega_closed}), it follows that $\psi_{\mathcal{G}} \in \Omega$, and hence it is a t.i. and locally defined set operator.
\end{proof}

\begin{remark}
	\label{remark_basis}
	An algorithm to calculate a window $W_{\mathcal{G}}$ of $\psi_{\mathcal{G}}$ or its basis $\boldsymbol{B}_{W_{\mathcal{G}}}(\psi_{\mathcal{G}})$ is a consequence of the results of \cite{barrera1996set} summarized in Propositions \ref{prop_window_combination} and \ref{prop_basis_combination}. The window or the basis may be computed via the computational graph $\mathcal{G}$, but instead of applying the computations determined by $\mathcal{C}$, one combines the windows or basis of the input vertices based on Propositions \ref{prop_window_combination} and \ref{prop_basis_combination} to obtain the window or basis of the set operator defined as the output of each vertex. At the end of the computations, the window or basis stored in the output vertex will be that of $\psi_{\mathcal{G}}$.
\end{remark}

We are in position to define the DMNN represented by a MCG.

\begin{definition}
	Let $\mathcal{G} = (\mathcal{V},\mathcal{E},\mathcal{C})$ be a Morphological Computational Graph. We define the Discrete Morphological Neural Network represented by $\mathcal{G}$ as the translation invariant and locally defined set operator $\psi_{\mathcal{G}}$.
\end{definition}

We present four examples of MCG which, respectively, trivially represents a $\psi \in \Omega$, represents $\psi$ via its canonical sup-generating or inf-generating decomposition, represents an alternate-sequential filter, and represents the composition of an alternate-sequential filter and a W-operator.

\begin{example}[Trivial DMNN]
	\normalfont \label{example_MDNN_trivial}
	Fix $\psi \in \Omega$ and let $\mathcal{G}$ be the MCG with $\mathcal{V} = \{\mathfrak{v}_{i},\mathfrak{v},\mathfrak{v}_{o}\}$, $\mathcal{E} = \{(\mathfrak{v}_{i},\mathfrak{v}),(\mathfrak{v},\mathfrak{v}_{o})\}$ and $\mathcal{C}(\mathfrak{v}) = \psi$. This MCG, depicted in Figure \ref{fig_trivial}, is such that $\psi = \psi_{\mathcal{G}}$ and illustrates that any $\psi \in \Omega$ can be trivially represented by a DMNN.
	
	\begin{figure}[ht]
		\centering
		\begin{tikzpicture}[scale=1]
			\node[circle,draw=black,minimum size=30pt,line width=0.5mm] (ei) at (0,0) {$X$};
			\node[circle,draw=black,minimum size=30pt,line width=0.5mm] (e) at (2.5,0) {$\psi$};
			\node[circle,draw=black,minimum size=30pt,line width=0.5mm] (eo) at (5,0) {$\psi_{\mathcal{G}}$};
			
			\draw[->,line width=0.5mm] (ei) -- (e);
			\draw[->,line width=0.5mm] (e) -- (eo);
		\end{tikzpicture}
		\caption{\footnotesize Trivial DMNN representation of $\psi \in \Omega$.} \label{fig_trivial}
	\end{figure}
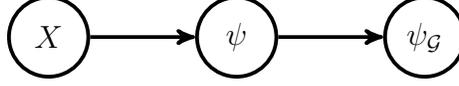
	
	\hfill$\blacksquare$
\end{example}

\begin{example}[Sup-generating and inf-generating DMNN]
	\normalfont \label{example_MDNN_supGenerating}
	
	Let $\psi \in \Psi_{W}$. Due to the canonical \added{sup} \deleted{inf} and \deleted{sup}\added{inf}-generating decomposition of $\psi$ (cf. Proposition \ref{prop_canonical_decomposition}), it can be represented by a DMNN with a vertex for each element in $\boldsymbol{B}_{W}(\psi)$ \added{or $\boldsymbol{B}_{W}(\psi^{\star})$}, that applies the respective \added{sup} \deleted{inf} and \deleted{sup}\added{inf}-generating operator to the input $X \in \mathcal{P}(E)$, and \added{a supremum or infimum} \deleted{an infimum or supremum} vertex that combines their output, respectively. The \added{sup} \deleted{inf} and \deleted{sup}\added{inf}-generating DMNN are represented in Figure \ref{fig_supGenDMNN}. This example illustrates that any t.i. and locally defined set operator can be non-trivially represented by a DMNN.
	
	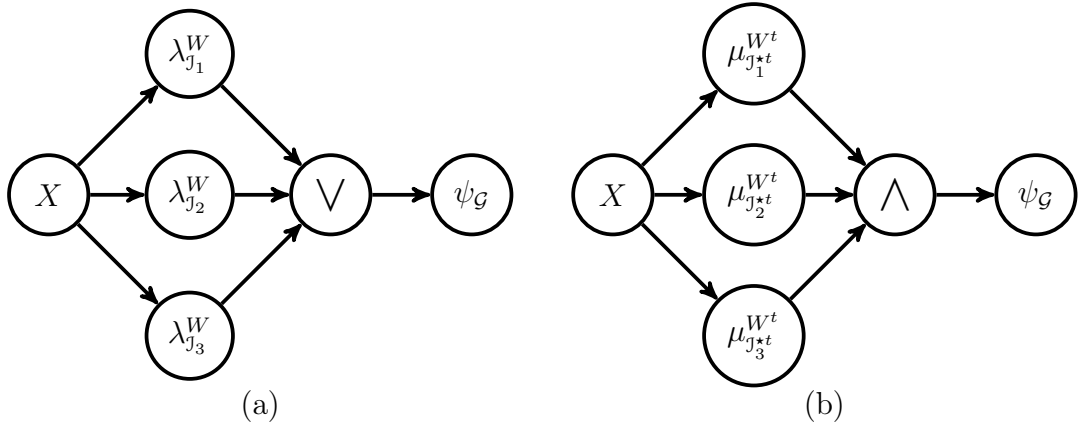
\begin{figure}[ht]
		\centering
		\begin{tikzpicture}[scale=0.75]
			\node[circle,draw=black,minimum size=30pt,line width=0.5mm] (ei) at (0,0) {$X$};
			
			\node[circle,draw=black,minimum size=30pt,line width=0.5mm] (e1) at (2.5,2.5) {$\lambda_{\mathscr{I}_{1}}^{W}$};
			\node[circle,draw=black,minimum size=30pt,line width=0.5mm] (e2) at (2.5,0) {$\lambda_{\mathscr{I}_{2}}^{W}$};
			\node[circle,draw=black,minimum size=30pt,line width=0.5mm] (e3) at (2.5,-2.5) {$\lambda_{\mathscr{I}_{3}}^{W}$};
			
			\node[circle,draw=black,minimum size=30pt,line width=0.5mm] (esup) at (5,0) {$\bigvee$};
			
			\node[circle,draw=black,minimum size=30pt,line width=0.5mm] (eo) at (7.5,0) {$\psi_{\mathcal{G}}$};
			
			\node (a) at (3.75,-3.75) {(a)};
			
			\node[circle,draw=black,minimum size=30pt,line width=0.5mm] (ei2) at (10,0) {$X$};
			
			\node[circle,draw=black,minimum size=30pt,line width=0.5mm] (e12) at (12.5,2.5) {$\mu_{\mathscr{I}_{1}^{\star t}}^{W^{t}}$};
			\node[circle,draw=black,minimum size=30pt,line width=0.5mm] (e22) at (12.5,0) {$\mu_{\mathscr{I}_{2}^{\star t}}^{W^{t}}$};
			\node[circle,draw=black,minimum size=30pt,line width=0.5mm] (e32) at (12.5,-2.5) {$\mu_{\mathscr{I}_{3}^{\star t}}^{W^{t}}$};
			
			\node[circle,draw=black,minimum size=30pt,line width=0.5mm] (esup2) at (15,0) {$\bigwedge$};
			
			\node[circle,draw=black,minimum size=30pt,line width=0.5mm] (eo2) at (17.5,0) {$\psi_{\mathcal{G}}$};
			
			\node (b) at (13.75,-3.75) {(b)};
			
			\draw[->,line width=0.5mm] (ei) -- (e1);
			\draw[->,line width=0.5mm] (ei) -- (e2);
			\draw[->,line width=0.5mm] (ei) -- (e3);
			\draw[->,line width=0.5mm] (e1) -- (esup);
			\draw[->,line width=0.5mm] (e2) -- (esup);
			\draw[->,line width=0.5mm] (e3) -- (esup);
			\draw[->,line width=0.5mm] (esup) -- (eo);
			
			\draw[->,line width=0.5mm] (ei2) -- (e12);
			\draw[->,line width=0.5mm] (ei2) -- (e22);
			\draw[->,line width=0.5mm] (ei2) -- (e32);
			\draw[->,line width=0.5mm] (e12) -- (esup2);
			\draw[->,line width=0.5mm] (e22) -- (esup2);
			\draw[->,line width=0.5mm] (e32) -- (esup2);
			\draw[->,line width=0.5mm] (esup2) -- (eo2);
		\end{tikzpicture}
		\caption{\footnotesize (a) Sup-generating and (b) inf-generating DMNN representation of a $\psi \in \Psi_{W}$ with $\boldsymbol{B}_{W}(\psi) = \{\mathscr{I}_{1},\mathscr{I}_{2},\mathscr{I}_{3}\}$ \added{and $\boldsymbol{B}_{W}(\psi) = \{\mathscr{I}_{1}^{\star},\mathscr{I}_{2}^{\star},\mathscr{I}_{3}^{\star}\}$}.} \label{fig_supGenDMNN}
	\end{figure}
	
	\hfill$\blacksquare$
\end{example}

\begin{example}[Alternate-sequential filters DMNN]
	\normalfont \label{example_ASF}
	
	The DMNN are specially useful to represent operators that can be decomposed sequentially as the composition of set operators. A special class of such operators are the alternate-sequential filters (ASF) which are defined as the interleaved composition of openings and closings. Fixed a structural element $B \in \mathcal{P}(E)$, the opening $\gamma_{B}$ and the closing $\phi_{B}$ are the set operators defined, respectively, as
	\begin{linenomath}
		\begin{align}
			\label{formula_open_close}
			\gamma_{B} = \delta_{B}\epsilon_{B} & & \phi_{B} = \epsilon_{B}\delta_{B},
		\end{align}
	\end{linenomath}
	that are, respectively, an erosion followed by a dilation and a dilation followed by an erosion. Openings and closings are increasing and idempotent set operators. The openings are anti-extensive, i.e., $\gamma_{B}(X) \subseteq X$, and the closings are extensive, i.e., $X \subseteq \phi_{B}(X)$. See \cite[Section~5.4]{serra1983image} for more details.
	
	Let $\gamma_{B_{1}},\phi_{B_{1}},\dots,\gamma_{B_{n}},\phi_{B_{n}}$ be a sequence of $n$ openings and $n$ closings. The set operator $\psi$ given by
	\begin{linenomath}
		\begin{equation*}
			\psi = \phi_{B_{n}}\gamma_{B_{n}} \cdots \phi_{B_{1}}\gamma_{B_{1}}
		\end{equation*}
	\end{linenomath}
	is called an alternate-sequential filter. Figure \ref{fig_ASF} presents three DMNN representations of $\psi$ in which, respectively, the vertices are a composition of an opening and a closing, the vertices are openings or closings, and the vertices are erosions or dilations.
	
	\begin{figure}[ht]
		\centering
		\begin{tikzpicture}[scale=1]
			\node[circle,draw=black,minimum size=30pt,line width=0.5mm] (ei) at (0,0) {$X$};
			\node[circle,draw=black,minimum size=30pt,line width=0.5mm] (e1) at (4,0) {$\gamma\phi_{B_{1}}$};
			\node[circle,draw=black,minimum size=30pt,line width=0.5mm] (e2) at (8,0) {$\gamma\phi_{B_{2}}$};
			\node[circle,draw=black,minimum size=30pt,line width=0.5mm] (eo) at (12,0) {$\psi_{\mathcal{G}}$};
			
			\node (blank) at (0,-1) {};
			
			\draw[->,line width=0.5mm] (ei) -- (e1);
			\draw[->,line width=0.5mm] (e1) -- (e2);
			\draw[->,line width=0.5mm] (e2) -- (eo);
		\end{tikzpicture}
		
		\begin{tikzpicture}[scale=1]
			\node[circle,draw=black,minimum size=30pt,line width=0.5mm] (ei) at (0,0) {$X$};
			\node[circle,draw=black,minimum size=30pt,line width=0.5mm] (e1) at (2.5,0) {$\gamma_{B_{1}}$};
			\node[circle,draw=black,minimum size=30pt,line width=0.5mm] (e2) at (5,0) {$\phi_{B_{1}}$};
			\node[circle,draw=black,minimum size=30pt,line width=0.5mm] (e3) at (7.5,0) {$\gamma_{B_{2}}$};
			\node[circle,draw=black,minimum size=30pt,line width=0.5mm] (e4) at (10,0) {$\phi_{B_{2}}$};
			\node[circle,draw=black,minimum size=30pt,line width=0.5mm] (eo) at (12.5,0) {$\psi_{\mathcal{G}}$};
			
			\node (blank) at (0,-1) {};
			
			\draw[->,line width=0.5mm] (ei) -- (e1);
			\draw[->,line width=0.5mm] (e1) -- (e2);
			\draw[->,line width=0.5mm] (e2) -- (e3);
			\draw[->,line width=0.5mm] (e3) -- (e4);
			\draw[->,line width=0.5mm] (e4) -- (eo);
		\end{tikzpicture}
		
		\begin{tikzpicture}[scale=1]
			\node[circle,draw=black,minimum size=30pt,line width=0.5mm] (ei) at (0,0) {$X$};
			\node[circle,draw=black,minimum size=30pt,line width=0.5mm] (e1) at (1.5,0) {$\epsilon_{B_{1}}$};
			\node[circle,draw=black,minimum size=30pt,line width=0.5mm] (e2) at (3,0) {$\delta_{B_{1}}$};
			\node[circle,draw=black,minimum size=30pt,line width=0.5mm] (e3) at (4.5,0) {$\delta_{B_{1}}$};
			\node[circle,draw=black,minimum size=30pt,line width=0.5mm] (e4) at (6,0) {$\epsilon_{B_{1}}$};
			\node[circle,draw=black,minimum size=30pt,line width=0.5mm] (e5) at (7.5,0) {$\epsilon_{B_{2}}$};
			\node[circle,draw=black,minimum size=30pt,line width=0.5mm] (e6) at (9,0) {$\delta_{B_{2}}$};
			\node[circle,draw=black,minimum size=30pt,line width=0.5mm] (e7) at (10.5,0) {$\delta_{B_{2}}$};
			\node[circle,draw=black,minimum size=30pt,line width=0.5mm] (e8) at (12,0) {$\epsilon_{B_{2}}$};
			\node[circle,draw=black,minimum size=30pt,line width=0.5mm] (eo) at (13.5,0) {$\psi_{\mathcal{G}}$};
			
			\draw[->,line width=0.5mm] (ei) -- (e1);
			\draw[->,line width=0.5mm] (e1) -- (e2);
			\draw[->,line width=0.5mm] (e2) -- (e3);
			\draw[->,line width=0.5mm] (e3) -- (e4);
			\draw[->,line width=0.5mm] (e4) -- (e5);
			\draw[->,line width=0.5mm] (e5) -- (e6);
			\draw[->,line width=0.5mm] (e6) -- (e7);
			\draw[->,line width=0.5mm] (e7) -- (e8);
			\draw[->,line width=0.5mm] (e8) -- (eo);
		\end{tikzpicture}
		\caption{\footnotesize Three DMNN representations of a same alternate-sequential filter. The operators $\gamma\phi_{B_{1}}$ and $\gamma\phi_{B_{2}}$ denote, respectively, the compositions $\phi_{B_{1}}\gamma_{B_{1}}$ and $\phi_{B_{2}}\gamma_{B_{2}}$.} \label{fig_ASF}
	\end{figure}
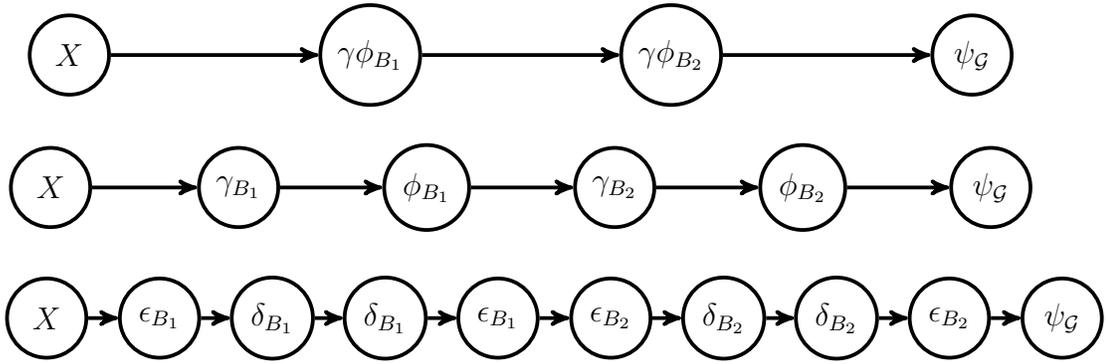 
	
	\hfill$\blacksquare$
\end{example}

\begin{example}[Composition of an ASF and a W-operator]
	\normalfont \label{example_ASFSupGen}
	
	Figure \ref{fig_ASF_SG} presents a DMNN representation of a composition of an ASF $\phi_{A}\gamma_{A}$ and a $\psi \in \Psi_{W}$ with three elements in its basis. This example illustrates how elements of sequential and sup-generating architectures may be combined to obtain more complex DMNN.

	\begin{figure}[ht]
		\centering
		\begin{tikzpicture}[scale=0.75]
			\node[circle,draw=black,minimum size=30pt,line width=0.5mm] (ei) at (-2.5,0) {$X$};
			
			\node[circle,draw=black,minimum size=30pt,line width=0.5mm] (eASF) at (0,0) {$\gamma\phi_{A}$};
			
			\node[circle,draw=black,minimum size=30pt,line width=0.5mm] (e1) at (2.5,2.5) {$\lambda_{\mathscr{I}_{1}}^{W}$};
			\node[circle,draw=black,minimum size=30pt,line width=0.5mm] (e2) at (2.5,0) {$\lambda_{\mathscr{I}_{2}}^{W}$};
			\node[circle,draw=black,minimum size=30pt,line width=0.5mm] (e3) at (2.5,-2.5) {$\lambda_{\mathscr{I}_{3}}^{W}$};
			
			\node[circle,draw=black,minimum size=30pt,line width=0.5mm] (esup) at (5,0) {$\bigvee$};
			
			\node[circle,draw=black,minimum size=30pt,line width=0.5mm] (eo) at (7.5,0) {$\psi_{\mathcal{G}}$};
			
			\draw[->,line width=0.5mm] (ei) -- (eASF);
			\draw[->,line width=0.5mm] (eASF) -- (e1);
			\draw[->,line width=0.5mm] (eASF) -- (e2);
			\draw[->,line width=0.5mm] (eASF) -- (e3);
			\draw[->,line width=0.5mm] (e1) -- (esup);
			\draw[->,line width=0.5mm] (e2) -- (esup);
			\draw[->,line width=0.5mm] (e3) -- (esup);
			\draw[->,line width=0.5mm] (esup) -- (eo);
		\end{tikzpicture}
		\caption{\footnotesize A DMNN representation of the composition of an ASF with structuring element $A \in \mathcal{P}(W)$ and a W-operator $\psi \in \Psi_{W}$ with $\boldsymbol{B}_{W}(\psi) = \{\mathscr{I}_{1},\mathscr{I}_{2},\mathscr{I}_{3}\}$.} \label{fig_ASF_SG}
	\end{figure}
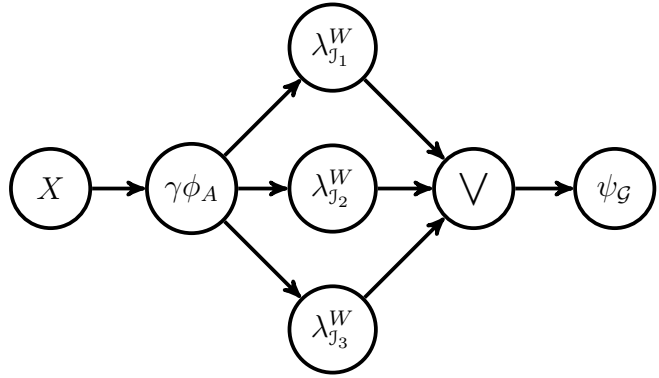
	
	\hfill$\blacksquare$
\end{example}

The examples above illustrate the flexibility of a DMNN in representing any t.i. and locally defined set operator. Such a flexibility allows, for instance, the representation of $\psi$ by a DMNN that is more efficiently computed via a MCG $\mathcal{G}$ or that has fewer \textit{parameters}. Another feature of this flexibility is the possibility of defining a class of set operators through restrictions on the set of all MCG. Indeed, since the mapping $\psi_{\mathcal{G}}$, from the collection of all MCG to $\Omega$, is surjective, restrictions on the MCG structure may generate restrictions on $\Omega$. In the next section, we formalize the restricted classes of MCG, which are the architectures of DMNN.

\subsection{DMNN Architectures}

Let $\mathcal{A} = (\mathcal{V},\mathcal{E},\mathcal{F})$ in which $\mathcal{V}$ is a set of vertices, $\mathcal{E} \subset \mathcal{V} \times \mathcal{V}$ is a set of ordered edges and $\mathcal{F} \subseteq \left(\Omega \cup \{\vee,\wedge\}\right)^{\mathcal{V}}$ is a collection of functions from $\mathcal{V}$ to $\Omega \cup \{\vee,\wedge\}$. We say that $\mathcal{A}$ is the architecture of a DMNN if it satisfies the following condition.  

\vspace{0.25cm}

\noindent \textbf{Axiom of Discrete Morphological Neural Network Architectures} A triple $\mathcal{A} = (\mathcal{V},\mathcal{E},\mathcal{F})$ is a Discrete Morphological Neural Network architecture if, and only if, $(\mathcal{V},\mathcal{E},\mathcal{C})$ is a Morphological Computational Graph for all $\mathcal{C} \in \mathcal{F}$.

\vspace{0.25cm}

A DMNN architecture is actually a collection of MCG with same graph $(\mathcal{V},\mathcal{E})$ and computation map $\mathcal{C}$ in $\mathcal{F}$. Since a DMNN architecture represents a collection of MCG, it actually represents a collection of t.i. and locally defined set operators (cf. Proposition \ref{prop_graph_Woperator}). For an architecture $\mathcal{A} = (\mathcal{V},\mathcal{E},\mathcal{F})$, let
\begin{linenomath}
	\begin{equation*}
		\mathbb{G}(\mathcal{A}) = \left\{\mathcal{G} = (\mathcal{V},\mathcal{E},\mathcal{C}): \mathcal{C} \in \mathcal{F}\right\}
	\end{equation*}
\end{linenomath}
be the collection of MCG generated by $\mathcal{A}$. We say that $\mathcal{G} \in \mathbb{G}(\mathcal{A})$ is a realization of architecture $\mathcal{A}$ and we define
\begin{linenomath}
	\begin{equation*}
		\mathcal{H}(\mathcal{A}) = \left\{\psi \in \Omega: \psi = \psi_{\mathcal{G}}, \mathcal{G} \in \mathbb{G}(\mathcal{A})\right\}
	\end{equation*}
\end{linenomath}
as the collection of t.i. and locally defined set operators that can be realized by $\mathcal{A}$.

We present three examples of DMNN architectures.

\begin{example}[Sup-generating DMNN architecture]
	\normalfont \label{example_supGen_architecture}
	
	Figure \ref{fig_supGenDMNN_architecture} presents two examples of sup-generating architectures, which are characterized by the supremum of sup-generating set operators. In both examples, we consider the graph $(\mathcal{V},\mathcal{E})$ such that $\mathcal{V} = \{\mathfrak{v}_{i},\mathfrak{v}_{1},\mathfrak{v}_{2},\mathfrak{v}_{3},\mathfrak{v}_{\vee},\mathfrak{v}_{o}\}$ and 
	\begin{linenomath}
		\begin{equation*}
			\mathcal{E} = \{(\mathfrak{v}_{i},\mathfrak{v}_{1}),(\mathfrak{v}_{i},\mathfrak{v}_{2}),(\mathfrak{v}_{i},\mathfrak{v}_{3}),(\mathfrak{v}_{1},\mathfrak{v}_{\vee}),(\mathfrak{v}_{2},\mathfrak{v}_{\vee}),(\mathfrak{v}_{3},\mathfrak{v}_{\vee}),(\mathfrak{v}_{\vee},\mathfrak{v}_{o})\}.
		\end{equation*}
	\end{linenomath}	
	The architecture $\mathcal{A}_{1} = (\mathcal{V},\mathcal{E},\mathcal{F}_{1})$ is such that, for all $\mathcal{C} \in \mathcal{F}$, $\mathcal{C}(\mathfrak{v}_{i}) = \mathcal{C}(\mathfrak{v}_{o}) = \iota$, $\mathcal{C}(\mathfrak{v}_{\vee}) = \vee$ and
	\begin{linenomath}
		\begin{align*}
			\mathcal{C}(\mathfrak{v}_{j}) = \lambda_{[A,B]}^{W} \text{ with } A,B \in \mathcal{P}(W), A \subseteq B & & j = 1,2,3,
		\end{align*}
	\end{linenomath}
	which realizes the $W$-operators with at most three elements in their basis, that is,
	\begin{linenomath}
		\begin{equation*}
			\mathcal{H}(\mathcal{A}_{1}) = \left\{\psi \in \Psi_{W}: |\boldsymbol{B}_{W}(\psi)| \leq 3\right\}.
		\end{equation*}
	\end{linenomath}
	Denoting
	\begin{linenomath}
		\begin{equation*}
			\Lambda_{W} = \{\lambda_{[A,B]}^{W}: A,B \in \mathcal{P}(W), A \subseteq B\}
		\end{equation*}
	\end{linenomath}
	as the sup-generating set operators locally defined within $W$, it follows that
	\begin{linenomath}
		\begin{equation*}
			\mathcal{F}_{1} = \{\iota\} \times \Lambda_{W} \times \Lambda_{W} \times \Lambda_{W} \times \{\vee\} \times \{\iota\},
		\end{equation*}
	\end{linenomath}
	so $\mathcal{F}_{1}$ is the Cartesian product of elements in $\mathcal{P}(\Psi_{W}) \cup \{\{\vee\},\{\wedge\}\}$. This means that, under this architecture, the computation of each vertex can be chosen independently of one another.
	
	Denoting
	\begin{linenomath}
		\begin{equation*}
			\Sigma_{W} = \{\epsilon_{A}^{W}: A \in \mathcal{P}(W)\}
		\end{equation*}
	\end{linenomath}
	as the erosions with structural elements in $\mathcal{P}(W)$, the architecture $\mathcal{A}_{2}$ is such that
	\begin{linenomath}
		\begin{equation*}
			\mathcal{F}_{2} = \{\iota\} \times \Sigma_{W} \times \Sigma_{W} \times \Sigma_{W} \times \{\vee\} \times \{\iota\},
		\end{equation*}
	\end{linenomath}
	so it realizes the more restricted class of increasing $W$-operators with at most three elements in their basis, that is,
	\begin{linenomath}
		\begin{equation*}
			\mathcal{H}(\mathcal{A}_{2}) = \left\{\psi \in \Psi_{W}: \psi \text{ is increasing},|\boldsymbol{B}_{W}(\psi)| \leq 3\right\},
		\end{equation*}
	\end{linenomath}
	which can be represented by the supremum of three erosions. Even though $\mathcal{A}_{1}$ and $\mathcal{A}_{2}$ have the same graph, $\mathcal{H}(\mathcal{A}_{2}) \subseteq \mathcal{H}(\mathcal{A}_{1})$ so $\mathcal{A}_{2}$ is a restriction of $\mathcal{A}_{1}$.
	
	\begin{figure}[ht]
		\centering
		\begin{tikzpicture}[scale=0.75]
			\node[circle,draw=black,minimum size=30pt,line width=0.5mm] (ei) at (0,0) {$X$};
			
			\node[circle,draw=black,minimum size=30pt,line width=0.5mm] (e1) at (2.5,2.5) {$\lambda_{\cdot}^{W}$};
			\node[circle,draw=black,minimum size=30pt,line width=0.5mm] (e2) at (2.5,0) {$\lambda_{\cdot}^{W}$};
			\node[circle,draw=black,minimum size=30pt,line width=0.5mm] (e3) at (2.5,-2.5) {$\lambda_{\cdot}^{W}$};
			
			\node[circle,draw=black,minimum size=30pt,line width=0.5mm] (esup) at (5,0) {$\bigvee$};
			
			\node[circle,draw=black,minimum size=30pt,line width=0.5mm] (eo) at (7.5,0) {$\psi_{\cdot}$};
			
			\node (a) at (3.75,-3.75) {$\mathcal{A}_{1}$};
			
			\node[circle,draw=black,minimum size=30pt,line width=0.5mm] (ei2) at (10,0) {$X$};
			
			\node[circle,draw=black,minimum size=30pt,line width=0.5mm] (e12) at (12.5,2.5) {$\epsilon_{\cdot}^{W}$};
			\node[circle,draw=black,minimum size=30pt,line width=0.5mm] (e22) at (12.5,0) {$\epsilon_{\cdot}^{W}$};
			\node[circle,draw=black,minimum size=30pt,line width=0.5mm] (e32) at (12.5,-2.5) {$\epsilon_{\cdot}^{W}$};
			
			\node[circle,draw=black,minimum size=30pt,line width=0.5mm] (esup2) at (15,0) {$\bigvee$};
			
			\node[circle,draw=black,minimum size=30pt,line width=0.5mm] (eo2) at (17.5,0) {$\psi_{\cdot}$};
			
			\node (b) at (13.75,-3.75) {$\mathcal{A}_{2}$};
			
			\draw[->,line width=0.5mm] (ei) -- (e1);
			\draw[->,line width=0.5mm] (ei) -- (e2);
			\draw[->,line width=0.5mm] (ei) -- (e3);
			\draw[->,line width=0.5mm] (e1) -- (esup);
			\draw[->,line width=0.5mm] (e2) -- (esup);
			\draw[->,line width=0.5mm] (e3) -- (esup);
			\draw[->,line width=0.5mm] (esup) -- (eo);
			
			\draw[->,line width=0.5mm] (ei2) -- (e12);
			\draw[->,line width=0.5mm] (ei2) -- (e22);
			\draw[->,line width=0.5mm] (ei2) -- (e32);
			\draw[->,line width=0.5mm] (e12) -- (esup2);
			\draw[->,line width=0.5mm] (e22) -- (esup2);
			\draw[->,line width=0.5mm] (e32) -- (esup2);
			\draw[->,line width=0.5mm] (esup2) -- (eo2);
		\end{tikzpicture}
		\caption{\footnotesize Two DMNN architectures with the same graph that represent sup-generating decompositions. When we exchange a parameter by a dot, we mean that it is free, so $\lambda_{\cdot}^{W}$ represents any sup-generating operator locally defined within $W$ and $\epsilon_{\cdot}^{W}$ represents any erosion locally defined within $W$. Architecture $\mathcal{A}_{1}$ realizes all $\psi \in \Psi_{W}$ with at most 3 elements in the basis, i.e., $|\boldsymbol{B}_{W}(\psi)| \leq 3$, while architecture $\mathcal{A}_{2}$ realizes all $\psi \in \Psi_{W}$ that can be written as the supremum of three erosions, which are actually the increasing operators in $\Psi_{W}$ with at most three elements in their basis. See Figure \ref{fig_supGenDMNN} (a) for a realization of architecture $\mathcal{A}_{1}$.} \label{fig_supGenDMNN_architecture}
	\end{figure}
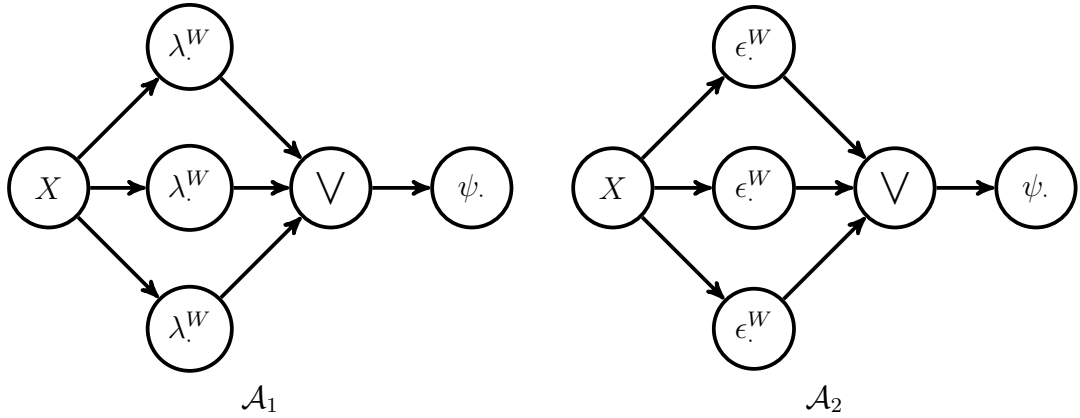
	
	\hfill$\blacksquare$
\end{example}

\begin{example}[Sequential DMNN architecture]
	\normalfont \label{example_sequential_architecture}
	
	Figure \ref{fig_sequential_architecture} presents three examples of sequential DMNN architectures characterized by the absence of supremum and infimum vertices. The architectures $\mathcal{A}_{1}$ and $\mathcal{A}_{2}$ have the same graph $(\mathcal{V},\mathcal{E})$ with $\mathcal{V} = \{\mathfrak{v}_{i},\mathfrak{v}_{1},\mathfrak{v}_{2},\mathfrak{v}_{3},\mathfrak{v}_{4},\mathfrak{v}_{o}\}$ and $\mathcal{E}$ linking the vertices in sequence. For fixed windows $W_{1}, W_{2} \in \mathcal{P}(E)$, the architecture $\mathcal{A}_{1}$ is such that
	\begin{linenomath}
		\begin{equation*}
			\mathcal{F}_{1} = \{\iota\} \times \Psi_{W_{1}} \times \Psi_{W_{1}} \times \Psi_{W_{2}} \times \Psi_{W_{2}} \times \{\iota\}
		\end{equation*}
	\end{linenomath}
	so it realizes the operators in $\Psi_{W}$, with $W = W_{1} \oplus W_{1} \oplus W_{2} \oplus W_{2}$, that can be written as the composition of two operators in $\Psi_{W_{1}}$ composed with two operators in $\Psi_{W_{2}}$, that is,
	\begin{linenomath}
		\begin{equation*}
			\mathcal{H}(\mathcal{A}_{1}) = \left\{\psi \in \Psi_{W}: \psi = \psi^{\prime W_{2}}\psi^{W_{2}}\psi^{\prime W_{1}}\psi^{W_{1}}; \psi^{W_{i}},\psi^{\prime W_{i}} \in \Psi_{W_{i}}, i = 1,2\right\}.
		\end{equation*}
	\end{linenomath}
	This is another example in which the computation of the vertices in the architecture can be chosen independently, since $\mathcal{F}_{1}$ is a Cartesian product.
	
	In the architecture $\mathcal{A}_{2}$, the functions $\mathcal{C}$ in $\mathcal{F}_{2}$ are such that
	\begin{linenomath}
		\begin{equation*}
			\begin{cases}
				\mathcal{C}(\mathfrak{v}_{i}) = \mathcal{C}(\mathfrak{v}_{o}) = \iota\\
				\mathcal{C}(\mathfrak{v}_{1}) = \gamma_{A}, \mathcal{C}(\mathfrak{v}_{2}) = \phi_{A}, \text{ for a } A \in \mathcal{P}(W_{1}) \\
				\mathcal{C}(\mathfrak{v}_{3}) = \gamma_{B}, \mathcal{C}(\mathfrak{v}_{4}) = \phi_{B}, \text{ for a } B \in \mathcal{P}(W_{2})
			\end{cases}
		\end{equation*}
	\end{linenomath}
	so $\mathcal{A}_{2}$ realizes the subclass of the ASF in $\mathcal{H}(\mathcal{A}_{1})$, that is,
	\begin{linenomath}
		\begin{equation*}
			\mathcal{H}(\mathcal{A}_{2}) = \left\{\psi \in \mathcal{H}(\mathcal{A}_{1}): \psi \text{ is an alternate-sequential filter}\right\}.
		\end{equation*}
	\end{linenomath}
	See Figure \ref{fig_ASF} for a realization of this architecture. In this example, $\mathcal{F}_{2}$ is not a Cartesian product, so the computation of its vertices are not independent. Indeed, the structural element of the first and last pairs of opening and closing have the same structural element, so in this example the vertices share a parameter. However, this sub-collection of set operators could also be represented by the architecture $\mathcal{A}_{3}$ in Figure \ref{fig_sequential_architecture} in which the interior vertices compute a composition of an opening and a closing with a same structuring element and the vertices do not share parameters. Even though $\mathcal{H}(\mathcal{A}_{2}) = \mathcal{H}(\mathcal{A}_{3})$, $\mathcal{F}_{3}$ is a Cartesian product, so the computation of the vertices can be chosen independently.
	
	This is another example of architectures with the same graph, but such that $\mathcal{H}(\mathcal{A}_{3}) = \mathcal{H}(\mathcal{A}_{2}) \subseteq \mathcal{H}(\mathcal{A}_{1})$, so $\mathcal{A}_{2}$ is a restriction of $\mathcal{A}_{1}$. This is also an example of a DMNN that can be represented by two equivalent architectures.
	
	\begin{figure}[ht]
		\centering
		\begin{tikzpicture}[scale=1]
			\node[circle,draw=black,minimum size=30pt,line width=0.5mm] (ei) at (0,0) {$X$};
			\node[circle,draw=black,minimum size=30pt,line width=0.5mm] (e1) at (2.5,0) {$\psi^{W_{1}}$};
			\node[circle,draw=black,minimum size=30pt,line width=0.5mm] (e2) at (5,0) {$\psi^{\prime W_{1}}$};
			\node[circle,draw=black,minimum size=30pt,line width=0.5mm] (e3) at (7.5,0) {$\psi^{W_{2}}$};
			\node[circle,draw=black,minimum size=30pt,line width=0.5mm] (e4) at (10,0) {$\psi^{\prime W_{2}}$};
			\node[circle,draw=black,minimum size=30pt,line width=0.5mm] (eo) at (12.5,0) {$\psi_{\mathcal{\cdot}}$};
			
			\node (blank) at (6.25,-1) {$\mathcal{A}_{1}$};
			
			\node[circle,draw=black,minimum size=30pt,line width=0.5mm] (ei2) at (0,-2) {$X$};
			\node[circle,draw=black,minimum size=30pt,line width=0.5mm] (e12) at (2.5,-2) {$\gamma_{\dot{A}}^{W_{1}}$};
			\node[circle,draw=black,minimum size=30pt,line width=0.5mm] (e22) at (5,-2) {$\phi_{\dot{A}}^{W_{1}}$};
			\node[circle,draw=black,minimum size=30pt,line width=0.5mm] (e32) at (7.5,-2) {$\gamma_{\dot{B}}^{W_{2}}$};
			\node[circle,draw=black,minimum size=30pt,line width=0.5mm] (e42) at (10,-2) {$\phi_{\dot{B}}^{W_{2}}$};
			\node[circle,draw=black,minimum size=30pt,line width=0.5mm] (eo2) at (12.5,-2) {$\psi_{\mathcal{\cdot}}$};
			
			\node (blank) at (6.25,-3) {$\mathcal{A}_{2}$};
			
			\node[circle,draw=black,minimum size=30pt,line width=0.5mm] (ei3) at (0,-4) {$X$};
			\node[circle,draw=black,minimum size=30pt,line width=0.5mm] (e13) at (4,-4) {$\gamma\phi_{\dot{A}}^{W_{1}}$};
			\node[circle,draw=black,minimum size=30pt,line width=0.5mm] (e23) at (8,-4) {$\gamma\phi_{\dot{B}}^{W_{2}}$};
			\node[circle,draw=black,minimum size=30pt,line width=0.5mm] (eo3) at (12,-4) {$\psi_{\mathcal{\cdot}}$};
			
			\node (blank) at (6.25,-5) {$\mathcal{A}_{3}$};
			
			\draw[->,line width=0.5mm] (ei) -- (e1);
			\draw[->,line width=0.5mm] (e1) -- (e2);
			\draw[->,line width=0.5mm] (e2) -- (e3);
			\draw[->,line width=0.5mm] (e3) -- (e4);
			\draw[->,line width=0.5mm] (e4) -- (eo);
			
			\draw[->,line width=0.5mm] (ei2) -- (e12);
			\draw[->,line width=0.5mm] (e12) -- (e22);
			\draw[->,line width=0.5mm] (e22) -- (e32);
			\draw[->,line width=0.5mm] (e32) -- (e42);
			\draw[->,line width=0.5mm] (e42) -- (eo2);
			
			\draw[->,line width=0.5mm] (ei3) -- (e13);
			\draw[->,line width=0.5mm] (e13) -- (e23);
			\draw[->,line width=0.5mm] (e23) -- (eo3);
		\end{tikzpicture}
		\caption{\footnotesize Three DMNN with a sequential architecture. The notation $\psi^{W}, \gamma^{W}_{\dot{A}}$ and $\phi^{W}_{\dot{A}}$ refer to any operator in $\Psi_{W}$, any opening in $\Psi_{W}$ and any closing in $\Psi_{W}$, respectively. A letter with a dot on top is used to represent the openings and closings with the same structural element.  We denote by $\gamma\phi_{\dot{A}}^{W_{1}}$ and $\gamma\phi_{\dot{B}}^{W_{2}}$ the compositions $\phi_{\dot{A}}^{W_{1}}\gamma_{\dot{A}}^{W_{1}}$ and $\phi_{\dot{B}}^{W_{2}}\gamma_{\dot{B}}^{W_{2}}$, respectively. The architecture $\mathcal{A}_{1}$ realizes the operators in $\Psi_{W}$ with $W = W_{1} \oplus W_{1} \oplus W_{2} \oplus W_{2}$ that can be written as the composition of four operators in $\Psi_{W_{1}}, \Psi_{W_{1}}, \Psi_{W_{2}}$ and $\Psi_{W_{2}}$, respectively. The architectures $\mathcal{A}_{2}$ and $\mathcal{A}_{3}$ realize the ASF given by the composition of, respectively, an opening and a closing in $\Psi_{W_{1}}$, and an opening and a closing in $\Psi_{W_{2}}$. Each pair of opening and closing share the structural element.} \label{fig_sequential_architecture}
	\end{figure}
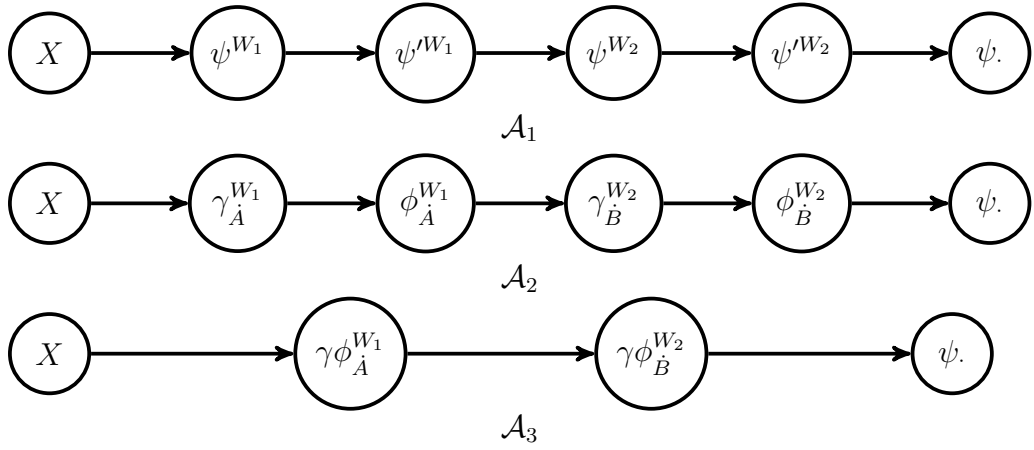 
	
	\hfill$\blacksquare$
\end{example}

\begin{example}[\deleted{Sequential} \added{Composition of} ASF and W-operator]
	\normalfont \label{example_ASFSupGen_LatticeF}
	
	Figure \ref{fig_ASF_SG_architecture} presents an example of an architecture which composes an ASF with a W-operator. This architecture has a graph $(\mathcal{V},\mathcal{E})$ such that $\mathcal{V} = \{\mathfrak{v}_{i},\mathfrak{v}_{1},\mathfrak{v}_{2},\mathfrak{v}_{3},\mathfrak{v}_{4},\mathfrak{v}_{\vee},\mathfrak{v}_{o}\}$ and 
	\begin{linenomath}
		\begin{equation*}
			\mathcal{E} = \{(\mathfrak{v}_{i},\mathfrak{v}_{1}),(\mathfrak{v}_{1},\mathfrak{v}_{2}),(\mathfrak{v}_{1},\mathfrak{v}_{3}),(\mathfrak{v}_{1},\mathfrak{v}_{4}),(\mathfrak{v}_{2},\mathfrak{v}_{\vee}),(\mathfrak{v}_{3},\mathfrak{v}_{\vee}),(\mathfrak{v}_{4},\mathfrak{v}_{\vee}),(\mathfrak{v}_{\vee},\mathfrak{v}_{o})\}.
		\end{equation*}
	\end{linenomath}	
	The architecture is such that, for all $\mathcal{C} \in \mathcal{F}$, $\mathcal{C}(\mathfrak{v}_{i}) = \mathcal{C}(\mathfrak{v}_{o}) = \iota$, $\mathcal{C}(\mathfrak{v}_{\vee}) = \vee$ and
	\begin{linenomath}
		\begin{align*}
			& \mathcal{C}(\mathfrak{v}_{1}) = \phi_{A}\gamma_{A} \text{ with } A \in \mathcal{P}(W) & & \\
			& \mathcal{C}(\mathfrak{v}_{j}) = \lambda_{[A_{j},B_{j}]}^{W} \text{ with } A_{j},B_{j} \in \mathcal{P}(W), A_{j} \subseteq B_{j} & & j = 2,3,4,
		\end{align*}
	\end{linenomath}
	which realizes the composition of a one\added{-}layer ASF and a $W$-operator with at most three elements in its basis. Observe that, denoting
	\begin{linenomath}
		\begin{equation*}
			\Gamma\Phi_{W} = \{\phi_{A}\gamma_{A}: A \in \mathcal{P}(W)\}
		\end{equation*}
	\end{linenomath}
	as the one layer ASF with structuring element in $W$, in this example we have
	\begin{linenomath}
		\begin{equation*}
			\mathcal{F} = \{\iota\} \times \Gamma\Phi_{W} \times \Lambda_{W} \times \Lambda_{W} \times \Lambda_{W} \times \{\vee\} \times \{\iota\},
		\end{equation*}
	\end{linenomath}
	so $\mathcal{F}$ is the Cartesian product of elements in $\mathcal{P}(\Psi_{W}) \cup \{\{\vee\},\{\wedge\}\}$, and the computation of each vertex can be chosen independently.

	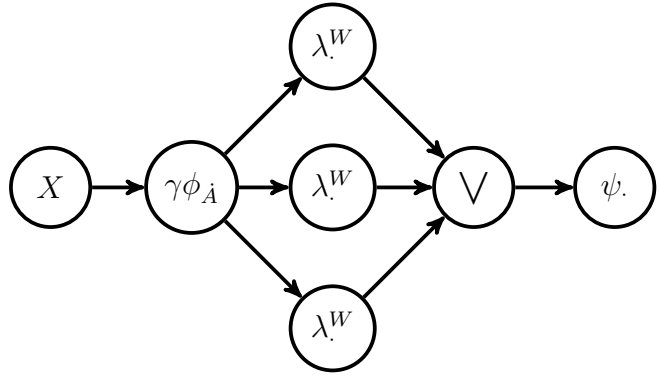
\begin{figure}[ht]
		\centering
		\begin{tikzpicture}[scale=0.75]
			\node[circle,draw=black,minimum size=30pt,line width=0.5mm] (ei) at (-2.5,0) {$X$};
			
			\node[circle,draw=black,minimum size=30pt,line width=0.5mm] (eASF) at (0,0) {$\gamma\phi_{\dot{A}}$};
			
			\node[circle,draw=black,minimum size=30pt,line width=0.5mm] (e1) at (2.5,2.5) {$\lambda_{\cdot}^{W}$};
			\node[circle,draw=black,minimum size=30pt,line width=0.5mm] (e2) at (2.5,0) {$\lambda_{\cdot}^{W}$};
			\node[circle,draw=black,minimum size=30pt,line width=0.5mm] (e3) at (2.5,-2.5) {$\lambda_{\cdot}^{W}$};
			
			\node[circle,draw=black,minimum size=30pt,line width=0.5mm] (esup) at (5,0) {$\bigvee$};
			
			\node[circle,draw=black,minimum size=30pt,line width=0.5mm] (eo) at (7.5,0) {$\psi_{\cdot}$};
			
			\draw[->,line width=0.5mm] (ei) -- (eASF);
			\draw[->,line width=0.5mm] (eASF) -- (e1);
			\draw[->,line width=0.5mm] (eASF) -- (e2);
			\draw[->,line width=0.5mm] (eASF) -- (e3);
			\draw[->,line width=0.5mm] (e1) -- (esup);
			\draw[->,line width=0.5mm] (e2) -- (esup);
			\draw[->,line width=0.5mm] (e3) -- (esup);
			\draw[->,line width=0.5mm] (esup) -- (eo);
		\end{tikzpicture}
		\caption{\footnotesize Example of a DMNN architecture which composes a one layer ASF and a W-operator with at most three intervals in its basis.} \label{fig_ASF_SG_architecture}
	\end{figure}
	
	\hfill$\blacksquare$
\end{example}

\subsection{DMNN are universal representers}

The semantic expressiveness of DMNN allows to represent any collection of t.i. and locally defined set operators by a DMNN. This is the result of the next theorem.

\begin{theorem}
	\label{theorem_universal_representer}
	Let $W \in \mathcal{P}(E)$ be a finite set. For any $\Theta \in \mathcal{P}(\Psi_{W})$, there exists a Discrete Morphological Neural Network architecture $\mathcal{A}_{\Theta}$ such that
	\begin{linenomath}
		\begin{equation*}
			\mathcal{H}(\mathcal{A}_{\Theta}) = \Theta.
		\end{equation*}
	\end{linenomath}
\end{theorem}
\begin{proof}
	Fix a finite set $W \in \mathcal{P}(E)$ and a sub-collection $\Theta \in \mathcal{P}(\Psi_{W})$. We will present a sup-generating DMNN architecture that generates $\Theta$. Consider the architecture $\mathcal{A}_{\Theta} = (\mathcal{V},\mathcal{E},\mathcal{F}_{\Theta})$ with a graph $(\mathcal{V},\mathcal{E})$ such that
	\begin{linenomath}
		\begin{equation*}
			\mathcal{V} = \left\{\mathfrak{v}_{i},\mathfrak{v}_{1},\dots,\mathfrak{v}_{2^{|W|-1}},\mathfrak{v}_{\vee},\mathfrak{v}_{o}\right\}
		\end{equation*}
	\end{linenomath}
	and
	\begin{linenomath}
		\begin{equation*}
			\mathcal{E} = \left\{(\mathfrak{v}_{i},\mathfrak{v}_{j}): j = 1,\dots,2^{|W| - 1}\right\} \bigcup \left\{(\mathfrak{v}_{j},\mathfrak{v}_{\vee}): j = 1,\dots,2^{|W| - 1}\right\} \bigcup \{(\mathfrak{v}_{\vee},\mathfrak{v}_{o})\}.
		\end{equation*}
	\end{linenomath}
	This graph is analogous to that of Figure \ref{fig_supGenDMNN} (a), but with $2^{|W| - 1}$ vertices computing sup-generating set operators instead of three.
	
	For each $\psi \in \Theta$, denote the elements in its basis by
	\begin{linenomath}
		\begin{equation*}
			\boldsymbol{B}_{W}(\psi) = \{\mathscr{I}_{1},\dots,\mathscr{I}_{|\boldsymbol{B}_{W}(\psi)|}\}
		\end{equation*}
	\end{linenomath}
	and observe that $|\boldsymbol{B}_{W}(\psi)| \leq 2^{|W| - 1}$. Consider the function $\mathcal{C}_{\psi}: \mathcal{V} \to \Psi_{W} \cup \{\vee,\wedge\}$ given by $\mathcal{C}_{\psi}(\mathfrak{v}_{i}) = \mathcal{C}_{\psi}(\mathfrak{v}_{o}) = \iota$, $\mathcal{C}_{\psi}(\mathfrak{v}_{\vee}) = \vee$ and
	\begin{linenomath}
		\begin{equation*}
			\mathcal{C}_{\psi}(\mathfrak{v}_{j}) = \begin{cases}
				\lambda_{\mathscr{I}_{j}}^{W}, & \text{ if } 1 \leq j \leq |\boldsymbol{B}_{W}(\psi)|\\
				\emptyset, & \text{ if } j > |\boldsymbol{B}_{W}(\psi)|
			\end{cases}
		\end{equation*}
	\end{linenomath}
	for $j = 1,\dots, 2^{|W| - 1}$, in which $\emptyset$ stands for the constant set operator equal to $\emptyset$. Observe that the MCG $\mathcal{G}_{\psi} = (\mathcal{V},\mathcal{E},\mathcal{C}_{\psi})$ generates a W-operator $\psi_{\mathcal{G}_{\psi}}$ satisfying $\psi_{\mathcal{G}_{\psi}} = \psi$. Considering
	\begin{linenomath}
		\begin{equation*}
			\mathcal{F}_{\Theta} = \left\{\mathcal{C}_{\psi}: \psi \in \Theta\right\}
		\end{equation*}
	\end{linenomath}
	we have that
	\begin{linenomath}
		\begin{equation*}
			\mathcal{H}(\mathcal{A}_{\Theta}) = \left\{\psi_{\mathcal{G}_{\psi}}: \psi \in \Theta\right\} = \Theta
		\end{equation*}
	\end{linenomath}
	as desired.
\end{proof}

\begin{remark}
	We could readily obtain an inf-generating DMNN architecture that represents a given $\Theta \in \mathcal{P}(\Psi_{W})$ following steps analogous to the proof of Theorem \ref{theorem_universal_representer}.
\end{remark}

\begin{remark}
	A sub-collection $\Theta \in \mathcal{P}(\Psi_{W})$ may be represented by other DMNN architectures besides a sup-generating and an inf-generating architecture, as  is the case of the sub-collections represented by the DMNN architectures in Examples \ref{example_sequential_architecture} and \ref{example_ASFSupGen_LatticeF}. Theorem \ref{theorem_universal_representer} established the existence of a DMNN architecture that represents $\Theta$, but not its uniqueness, which clearly does not hold.
\end{remark}

Although the sup-generating representation of a $\Theta \in \mathcal{P}(\Psi_{W})$ via a DMNN derived on the proof of Theorem \ref{theorem_universal_representer} is not in principle a new result, since it is a direct consequence of \cite{banon1991minimal}, the result of Theorem \ref{theorem_universal_representer} is pragmatic. \deleted{since} \added{It} guarantees that any prior information about the set operator that best solves a given problem \added{may be expressed by a DMNN architecture as long as it} \deleted{which} can be translated into a sub-collection $\Theta \in \mathcal{P}(\Psi_{W})$ containing operators with properties believed to be satisfied by the best operator. \deleted{may be expressed by a DMNN architecture.}

\section{Canonical Discrete Morphological Neural Networks}
\label{Sec_CDMNN}

Dilations and erosions are the canonical operators, and the supremum, infimum and complement are the canonical operations, of Mathematical Morphology. Indeed, any W-operator can be represented by combining the supremum, infimum, and complement of compositions of erosions and dilations. This fact is a consequence of Corollary \ref{corollary_canonical_basis} and the representation of sup-generating and inf-generating operators by \eqref{formula_supGen} and \eqref{formula_infGen}, respectively. Furthermore, openings, closings and ASF can also be simply represented by the composition of dilations and erosions (cf. \eqref{formula_open_close}). 

We define the Canonical Discrete Morphological Neural Networks (CDMNN) as those in which the computations are canonical operations or can be simply decomposed \deleted{on} \added{into} canonical operators. For each finite subset $W \in \mathcal{P}(E)$, denote
\begin{linenomath}
	\begin{align}
		\label{basic_classes} \nonumber
		&\Delta_{W} = \left\{\delta_{A}: A \in \mathcal{P}(W)\right\} & & \Sigma_{W} = \left\{\epsilon_{A}: A \in \mathcal{P}(W)\right\}\\
		&\Gamma_{W} = \left\{\gamma_{A}: A \in \mathcal{P}(W)\right\} & & \Phi_{W} = \left\{\phi_{A}: A \in \mathcal{P}(W)\right\}\\ \nonumber
		&\Gamma\Phi_{W} = \left\{\phi_{A}\gamma_{A}: A \in \mathcal{P}(W)\right\}  && \Lambda_{W} = \left\{\lambda_{[A,B]}: [A,B] \subseteq \mathcal{P}(W), A \subseteq B\right\}  \\ \nonumber
		& \mathcal{M}_{W} = \left\{\mu_{[A,B]}: [A,B] \subseteq \mathcal{P}(W), A \subseteq B\right\} & & \mathfrak{N} = \{\nu\}	
	\end{align}
\end{linenomath}
as the class of the dilations, erosions, openings, closings, one layer ASF, sup-generating, inf-generating and complement operators, respectively, with structuring elements in $\mathcal{P}(W)$ or intervals in $\mathcal{P}(\mathcal{P}(W))$. 

\begin{definition}
	\label{def_canonical_DMNN}
	A Discrete Morphological Neural Network architecture $\mathcal{A} = (\mathcal{V},\mathcal{E},\mathcal{F})$ is canonical if the following hold\deleted{s}:
	\begin{itemize}
		\item[\textbf{C1}] There exist two subsets of \deleted{edges} \added{vertices} $\mathcal{V}_{\vee}, \mathcal{V}_{\wedge} \subseteq \mathcal{V}$ such that, for all $\mathcal{C} \in \mathcal{F}$,
		\begin{linenomath}
			\begin{align*}
				\begin{cases}
					\mathcal{C}(\mathfrak{v}) = \vee, \forall \mathfrak{v} \in \mathcal{V}_{\vee}\\
					\mathcal{C}(\mathfrak{v}) = \wedge, \forall \mathfrak{v} \in \mathcal{V}_{\wedge}\\
					\mathcal{C}(\mathfrak{v}) \in \Omega_{\mathfrak{v}} \subseteq \Omega, \forall \mathfrak{v} \in \mathcal{V}_{\Omega},
				\end{cases}
			\end{align*}
		\end{linenomath}
		in which $\Omega_{\mathfrak{v}}$ is a class of t.i. and locally defined operators and $\mathcal{V}_{\Omega} \coloneqq \mathcal{V}\setminus\left(\mathcal{V}_{\vee} \cup \mathcal{V}_{\wedge}\right)$. 
		\item[\textbf{C2}] For every $\mathfrak{v} \in \mathcal{V}_{\Omega}$ there exists a finite subset $W_{\mathfrak{v}} \in \mathcal{P}(E)$ such that
		\begin{linenomath}
			\begin{equation*}
				\Omega_{\mathfrak{v}} \in \{\Delta_{W_{\mathfrak{v}}},\Sigma_{W_{\mathfrak{v}}},\Gamma_{W_{\mathfrak{v}}},\Phi_{W_{\mathfrak{v}}},\Gamma\Phi_{W_{\mathfrak{v}}},\Lambda_{W_{\mathfrak{v}}},\mathcal{M}_{W_{\mathfrak{v}}},\mathfrak{N}\}.
			\end{equation*}
		\end{linenomath}
	\end{itemize}		
\end{definition}

A special case of canonical DMNN is the unrestricted canonical DMNN, in which $\mathcal{F}$ is a Cartesian product of $\{\vee\}, \{\wedge\}$ and the sets defined in \eqref{basic_classes}.

\begin{definition}
	\label{def_unrestricted_DMNN}
	A canonical Discrete Morphological Neural Network architecture $\mathcal{A} = (\mathcal{V},\mathcal{E},\\ \mathcal{F})$ is unrestricted if
	\begin{linenomath}
		\begin{equation*}
			\mathcal{F} = \prod\limits_{\mathfrak{v} \in \mathcal{V}_{\vee}} \{\vee\} \bigtimes \prod\limits_{\mathfrak{v} \in \mathcal{V}_{\wedge}} \{\wedge\} \bigtimes \prod\limits_{\mathfrak{v} \in \mathcal{V}_{\Omega}} \Omega_{\mathfrak{v}} \subseteq \left(\Omega \cup \{\vee,\wedge\}\right)^{\mathcal{V}}.
		\end{equation*}
	\end{linenomath}
\end{definition}

The architectures in Figures \ref{fig_supGenDMNN_architecture} and \ref{fig_ASF_SG_architecture}, and the architectures $\mathcal{A}_{2}$ and $\mathcal{A}_{3}$ in Figure \ref{fig_sequential_architecture}, are canonical since the vertices that compute the supremum are the same over all realizations of the architectures and the morphological operators computed by the remaining vertices are erosions, dilations, openings, closings, one layer ASF, sup-generating or inf-generating. Furthermore, these architectures, except $\mathcal{A}_{2}$ in Figure \ref{fig_sequential_architecture}, are unrestricted since the parameters of the operators, which are structural elements or intervals, are not shared among vertices. The architecture $\mathcal{A}_{1}$ in Figure \ref{fig_sequential_architecture} is not canonical, since there are four vertices that compute general W-operators.

Observe that all vertices of a CDMNN compute either a canonical operation (supremum, infimum, or complement), or an operator that can be decomposed as the supremum, infimum, and complement of erosions and dilations, what justifies the \textit{canonical} nomenclature. Furthermore, any CDMNN could be represented by a MCG in which the vertices compute only erosions, dilations, supremum, infimum or complement, but that may share parameters. This fact is illustrated in Figure \ref{fig_ASF} in which a realization of a sequential CDMNN can be represented by the composition of two ASF that do not share parameters, or by the composition of dilations and erosions that share parameters. We chose the representation of CDMNN by MCG satisfying \textbf{C1} and \textbf{C2} in Definition \ref{def_canonical_DMNN} since it is more concise than a representation via erosions, dilations, supremum, infimum, and complement only, and allows the definition of unrestricted CDMNN when there is no parameter sharing.

\subsection{The lattice associated to a CDMNN architecture}

Fix a canonical DMNN architecture $\mathcal{A} = (\mathcal{V},\mathcal{E},\mathcal{F})$. We will consider a partial order $\leq$ in $\mathcal{F}$ based on the \textit{parameters} that represent the functions $\mathcal{C} \in \mathcal{F}$. More precisely, each $\mathcal{C} \in \mathcal{F}$ can be represented by a vector of sets and intervals that represent the structural elements of its vertices that compute erosions, dilations, openings, closings and one layer ASF, and the intervals of its vertices that compute inf-generating and sup-generating operators. \deleted{and} The collection $\mathcal{F}$ inherits the partial order of the inclusion of sets and intervals.

For example, if each $\mathcal{C} \in \mathcal{F}$ is represented by a vector $(A^{\mathcal{C}}_{1},\dots,A^{\mathcal{C}}_{p}) \in \mathcal{P}(W)^{p}, p \geq 1,$ of sets, then
\begin{linenomath}
	\begin{equation*}
		\mathcal{C}_{1} \leq \mathcal{C}_{2} \iff A_{j}^{\mathcal{C}_{1}} \subseteq A_{j}^{\mathcal{C}_{2}}, \forall j = 1,\dots,p, 
	\end{equation*}
\end{linenomath}
and $\mathcal{F}$ inherits the partial order of the inclusion of its parameters. ASF architectures (cf. $\mathcal{A}_{2}$ and $\mathcal{A}_{3}$ in Figure \ref{fig_sequential_architecture}) and increasing sup-generating architectures (cf. $\mathcal{A}_{2}$ in Figure \ref{fig_supGenDMNN_architecture}) are examples in which the functions in $\mathcal{F}$ may be parameterized by a vector of sets, which represent the structuring elements of erosions, openings, or closings. As another example, if each $\mathcal{C} \in \mathcal{F}$ is represented by a vector $([A^{\mathcal{C}}_{1},B^{\mathcal{C}}_{1}],\dots,[A^{\mathcal{C}}_{p},B^{\mathcal{C}}_{p}]) \in \mathcal{P}(\mathcal{P}(W))^{p}, p \geq 1,$ of intervals, then
\begin{linenomath}
	\begin{equation}
		\label{partial_inclusion_intervals}
		\mathcal{C}_{1} \leq \mathcal{C}_{2} \iff [A_{j}^{\mathcal{C}_{1}},B_{j}^{\mathcal{C}_{1}}] \subseteq [A_{j}^{\mathcal{C}_{2}},B_{j}^{\mathcal{C}_{2}}], \forall j = 1,\dots,p, 
	\end{equation}
\end{linenomath}
and $\mathcal{F}$ inherits the partial order of the inclusion of its parameters. Sup-generating architectures (cf. $\mathcal{A}_{1}$ in Figure \ref{fig_supGenDMNN_architecture}) are examples in which the functions in $\mathcal{F}$ may be parameterized by a vector of intervals, which represent the intervals of the sup-generating operators. Finally, if each $\mathcal{C} \in \mathcal{F}$ is represented by a vector $([A^{\mathcal{C}}_{1},B^{\mathcal{C}}_{1}],\dots,[A^{\mathcal{C}}_{p_{I}},B^{\mathcal{C}}_{p_{I}}],C^{\mathcal{C}}_{1},\dots,C^{\mathcal{C}}_{p_{S}}) \in \mathcal{P}(\mathcal{P}(W))^{p_{I}} \times \mathcal{P}(W)^{p_{S}}, p_{I},p_{S} \geq 1,$ of intervals and sets, then
\begin{linenomath}
	\begin{equation*}
		\mathcal{C}_{1} \leq \mathcal{C}_{2} \iff \begin{cases}
			[A_{j}^{\mathcal{C}_{1}},B_{j}^{\mathcal{C}_{1}}] \subseteq [A_{j}^{\mathcal{C}_{2}},B_{j}^{\mathcal{C}_{2}}], &\forall j = 1,\dots,p_{I}\\
			C_{j}^{\mathcal{C}_{1}} \subseteq C_{j}^{\mathcal{C}_{2}}, &\forall j = 1,\dots,p_{S}. 
		\end{cases}
	\end{equation*}
\end{linenomath}
The architecture in Figure \ref{fig_ASF_SG_architecture} is an example of this case, in which each $\mathcal{C} \in \mathcal{F}$ is represented by $(A,\mathscr{I}_{1},\mathscr{I}_{2},\mathscr{I}_{3})$. The set $A$ is the structuring element of the one layer ASF, and the intervals are the parameters of the three sup-generating operators.

Therefore, for any CDMNN architecture, $\mathcal{F}$ inherits a partial order $\leq$ from that of the structuring elements (sets) of the openings, closings, dilations, erosions and ASF, and from that of the intervals of the sup-generating and inf-generating operators, that it computes.

The poset $(\mathcal{F},\leq)$ is not, in general, isomorphic to $(\mathcal{H}(\mathcal{A}),\leq)$ since the mapping $\psi_{\cdot}$, that maps each function $\mathcal{C} \in \mathcal{F}$ to the operator $\psi_{\mathcal{C}} \in \mathcal{H}(\mathcal{A})$ generated by the MCG $\mathcal{G} = (\mathcal{V},\mathcal{E},\mathcal{C})$, is not injective when two realizations of the architecture $\mathcal{A}$ generate the same set operator. This is the case, for example, of the architecture $\mathcal{A}_{1}$ in Figure \ref{fig_supGenDMNN_architecture} in which an operator $\psi$ is generated by all realizations of $\mathcal{A}_{1}$ where the union of the intervals of the sup-generating operators is equal to its kernel $\mathcal{K}_{W}(\psi)$ (cf. Proposition \ref{prop_canonical_decomposition}). Actually, $\mathcal{F}$ is usually an overparametrization of $\mathcal{H}(\mathcal{A})$.

Although the representation of the operators in $\mathcal{H}(\mathcal{A})$ by $\mathcal{F}$ is not unique and $(\mathcal{F},\leq)$ is not isomorphic to $(\mathcal{H}(\mathcal{A}),\leq)$, it has the advantage that the elements and partial order of $\mathcal{F}$ are completely known a priori, since $\mathcal{F}$ is fixed and the partial order is that of inclusion of sets/intervals, while the elements of $\mathcal{H}(\mathcal{A})$ and its partial order may need to be computed. Therefore, training the DMNN by minimizing an empirical loss in $(\mathcal{F},\leq)$ to obtain a set operator that \textit{fits} empirical data may be computationally cheaper than minimizing this loss in $\mathcal{H}(\mathcal{A})$ directly, since its elements and partial order should be computed (see Remark \ref{remark_FnotH} for more details.). This observation points to the fact that a DMNN is a tool, or a \textit{computational machine}, to learn set operators in a fixed collection $\Theta = \mathcal{H}(\mathcal{A})$ and not a learning model in itself.

The suboptimal minimization of an empirical loss in $(\mathcal{F},\leq)$ could be performed via a greedy algorithm that runs through this poset by locally minimizing the loss. This algorithm depends on the definition of the neighborhoods of $(\mathcal{F},\leq)$.

\subsection{Neighborhoods of $(\mathcal{F},\leq)$}

In view of the discussion above, we define the neighborhood of $\mathcal{C} \in \mathcal{F}$ as the points in $\mathcal{F}$ at a distance one from it. The poset $(\mathcal{F},\leq)$ generates a directed acyclic graph whose vertices are $\mathcal{F}$ and the edges connect every pair of subsequent elements, which are such that $\mathcal{C}_{1} \leq \mathcal{C}_{2}$ and if $\mathcal{C}_{1} \leq \mathcal{C} \leq \mathcal{C}_{2}, \mathcal{C} \in \mathcal{F}$, then either $\mathcal{C}_{1} = \mathcal{C}$ or $\mathcal{C}_{2} = \mathcal{C}$, with orientation from $\mathcal{C}_{1}$ to $\mathcal{C}_{2}$. In this setting, the distance between elements $\mathcal{C}_{1}, \mathcal{C}_{2} \in \mathcal{F}$, with $\mathcal{C}_{1} \leq \mathcal{C}_{2}$, denoted by $d(\mathcal{C}_{1},\mathcal{C}_{2}) = d(\mathcal{C}_{2},\mathcal{C}_{1})$, is the length of the shortest path from $\mathcal{C}_{1}$ to $\mathcal{C}_{2}$.

\begin{definition}
	\label{def_neighborhood}
	The neighborhood in $\mathcal{F}$ of each $\mathcal{C} \in \mathcal{F}$ is defined as
	\begin{linenomath}
		\begin{equation*}
			\mathcal{N}(\mathcal{C}) = \left\{\mathcal{C}^{\prime} \in \mathcal{F}: d(\mathcal{C},\mathcal{C}^{\prime}) = 1\right\}.
		\end{equation*}
	\end{linenomath}
\end{definition}

\added{We emphasize that $\mathcal{C}$ is not an element of $\mathcal{N}(\mathcal{C})$.} We exemplify $\mathcal{F}$ and its neighborhoods for an ASF architecture, a sup-generating architecture and the composition of an ASF and a W-operator.

\begin{example}
	\normalfont \label{example_ASF_LatticeF}
	
	Let $\mathcal{A}$ be the ASF architecture $\mathcal{A}_{3}$ in Figure \ref{fig_sequential_architecture}. In this example, $\mathcal{F}$ is isomorphic to the Boolean lattice $\mathcal{P}(W)^{2}$ and an element $(A,B) \in \mathcal{P}(W)^{2}$ realizes the ASF with two pairs of openings and closings in which the first has structuring element $A$ and the second $B$. Denote by $\mathcal{C}_{A,B} \in \mathcal{F}$ the element of $\mathcal{F}$ associated to $(A,B)$. We have that
	\begin{linenomath}
		\begin{equation*}
			\mathcal{N}(\mathcal{C}_{A,B}) = \left\{\mathcal{C}_{C,B} \in \mathcal{F}: d(C,A) = 1\right\} \cup \left\{\mathcal{C}_{A,C} \in \mathcal{F}: d(C,B) = 1\right\},
		\end{equation*}
	\end{linenomath}
	in which $d(C,A)$ and $d(C,B)$ are the respective distance in $\mathcal{P}(W)$, so a neighbor of $\mathcal{C}_{A,B}$ is obtained by adding or removing one point from $A$ or $B$. This is an example in which the poset $(\mathcal{F},\leq)$ is actually a Boolean lattice, which is partially depicted in Figure \ref{fig_F_ASF}.
	\hfill$\blacksquare$
\end{example}

\begin{example}
	\normalfont \label{example_SupGen_LatticeF}
	
	Let $\mathcal{A}$ be the sup-generating architecture $\mathcal{A}_{1}$ in Figure \ref{fig_supGenDMNN_architecture}. In this example, $\mathcal{F}$ is isomorphic to a subset of the Boolean lattice $\mathcal{P}(\mathcal{P}(W))^{3}$ and each element $(\mathscr{I}_{1},\mathscr{I}_{2},\mathscr{I}_{3})$ of this subset represents the sup-generating MCG given by the supremum of three sup-generating operators with intervals $(\mathscr{I}_{1},\mathscr{I}_{2},\mathscr{I}_{3})$. Denote by $\mathcal{C}_{\mathscr{I}_{1},\mathscr{I}_{2},\mathscr{I}_{3}} \in \mathcal{F}$ the element of $\mathcal{F}$ associated to $(\mathscr{I}_{1},\mathscr{I}_{2},\mathscr{I}_{3})$. We have that
	\begin{linenomath}
		\begin{align*}
			\mathcal{N}(\mathcal{C}_{\mathscr{I}_{1},\mathscr{I}_{2},\mathscr{I}_{3}}) =& \left\{\mathcal{C}_{\mathscr{I},\mathscr{I}_{2},\mathscr{I}_{3}} \in \mathcal{F}: d(\mathscr{I},\mathscr{I}_{1}) = 1\right\} \cup \left\{\mathcal{C}_{\mathscr{I}_{1},\mathscr{I},\mathscr{I}_{3}} \in \mathcal{F}: d(\mathscr{I},\mathscr{I}_{2}) = 1\right\}\\
			& \cup \left\{\mathcal{C}_{\mathscr{I}_{1},\mathscr{I}_{2},\mathscr{I}} \in \mathcal{F}: d(\mathscr{I},\mathscr{I}_{3}) = 1\right\},
		\end{align*}
	\end{linenomath}
	in which $d(\mathscr{I},\mathscr{I}_{j}), j = 1,2,3$, is the respective distance in the poset of the intervals in $\mathcal{P}(\mathcal{P}(W))$, so a neighbor of $\mathcal{C}_{\mathscr{I}_{1}, \mathscr{I}_{2}, \mathscr{I}_{3}}$ is obtained by adding or removing a certain point from one extremity of one of its three intervals.
	\hfill$\blacksquare$
\end{example}

\begin{example}
	\normalfont \label{example_ASF_SG_LatticeF}
	
	Let $\mathcal{A}$ be the architecture in Figure \ref{fig_ASF_SG_architecture}. In this example, $\mathcal{F}$ is isomorphic to a subset of the Boolean lattice $\mathcal{P}(W) \times \mathcal{P}(\mathcal{P}(W))^{3}$ and each element $(A,\mathscr{I}_{1},\mathscr{I}_{2},\mathscr{I}_{3})$ of this subset realizes the MCG given by the composition of a one layer ASF with structuring element $A$ and the supremum of three sup-generating operators with intervals $(\mathscr{I}_{1},\mathscr{I}_{2},\mathscr{I}_{3})$. Denote by $\mathcal{C}_{A,\mathscr{I}_{1},\mathscr{I}_{2},\mathscr{I}_{3}} \in \mathcal{F}$ the element of $\mathcal{F}$ associated to $(A,\mathscr{I}_{1},\mathscr{I}_{2},\mathscr{I}_{3})$. We have that
	\begin{linenomath}
		\begin{align*}
			\mathcal{N}(\mathcal{C}_{A,\mathscr{I}_{1},\mathscr{I}_{2},\mathscr{I}_{3}}) =& \left\{\mathcal{C}_{B,\mathscr{I}_{1},\mathscr{I}_{2},\mathscr{I}_{3}} \in \mathcal{F}: d(A,B) = 1\right\} \cup \left\{\mathcal{C}_{A,\mathscr{I},\mathscr{I}_{2},\mathscr{I}_{3}} \in \mathcal{F}: d(\mathscr{I},\mathscr{I}_{1}) = 1\right\}\\
			& \cup \left\{\mathcal{C}_{A,\mathscr{I}_{1},\mathscr{I},\mathscr{I}_{3}} \in \mathcal{F}: d(\mathscr{I},\mathscr{I}_{2}) = 1\right\} \cup \left\{\mathcal{C}_{A,\mathscr{I}_{1},\mathscr{I}_{2},\mathscr{I}} \in \mathcal{F}: d(\mathscr{I},\mathscr{I}_{3}) = 1\right\},
		\end{align*}
	\end{linenomath}
	in which the distances are on the respective posets of sets and intervals, so a neighbor of $\mathcal{C}_{A,\mathscr{I}_{1},\mathscr{I}_{2},\mathscr{I}_{3}}$ is obtained by adding or removing one point from the structuring element $A$ or from one extremity of one of its three intervals.
	\hfill$\blacksquare$
\end{example}

\section{Training Canonical Discrete Morphological Neural Networks}
\label{Sec_train}

A training algorithm for a CDMNN should receive an architecture and a sample of pairs of input and target sets, and return a realization of the architecture that well \textit{approximates} the set operation applied to the input sets to obtain the respective target set. For a $N \geq 1$, let
\begin{linenomath}
	\begin{equation*}
		\mathcal{D}_{N} = \left\{(X_{1},Y_{1}),\dots,(X_{N},Y_{N})\right\}
	\end{equation*}
\end{linenomath}
be a sample of $N$ input sets $X$ and target sets $Y$. We assume there exists a finite set $F \in \mathcal{P}(E)$ such that $X_{j}, Y_{j} \subseteq F, j = 1,\dots,N$. For example, if $E = \mathbb{Z}^{2}$ and $X,Y$ represent binary images with $d \times d$ pixels, then we could consider $F = \{-d/2,\dots,d/2\}^{2}$. 

It is assumed that each $Y$ is, apart from a random noise, obtained from $X$ by applying a set operator $\psi^{\star} \in \Omega$. Fixed a DMNN architecture $\mathcal{A}$, the main objective of the learning is to estimate $\psi^{\star}$ by a $\hat{\psi} \in \mathcal{H}(\mathcal{A})$ which, we hope, is such that $\hat{\psi} \approx \psi^{\star}$.

The main paradigm of learning is empirical risk minimization \cite{vapnik1998}, in which $\hat{\psi}$ is obtained by minimizing the empirical loss incurred when $\hat{\psi}(X)$ is employed to approximate $Y$ in sample $\mathcal{D}_{N}$. Common losses in this setting are the absolute loss and the intersection of union loss. The absolute loss is defined as
\begin{linenomath}
	\begin{equation}
		\label{Aerror}
		\ell_{a}(X,Y;\psi) = \frac{\left|\left\{x \in F: x \in \left(Y^{c} \cap \psi(X)\right) \cup \left(Y \cap [\psi(X)]^{c}\right)\right\}\right|}{|F|},
	\end{equation}
\end{linenomath}
that is the proportion of the points in $F$ which are in $Y$ but not in $\psi(X)$ or that are in $\psi(X)$ but not in $Y$. The intersection \added{over} \deleted{of} union (IoU) loss is defined as
\begin{linenomath}
	\begin{equation*}
		\ell_{IoU}(X,Y;\psi) = 1 - \frac{|Y \cap \psi(X)|}{|Y \cup \psi(X)|},
	\end{equation*}
\end{linenomath}
that is the proportion of points in $Y \cup \psi(X)$ that are not in $Y \cap \psi(X)$. The IoU loss is more suitable for form or object detection tasks, while the absolute loss is suitable for general image transformation tasks. When the loss function is not of importance, we denote it simply by $\ell$.

Fix an architecture $\mathcal{A} = (\mathcal{V},\mathcal{E},\mathcal{F})$ and, for each $\mathcal{C} \in \mathcal{F}$, denote by $\psi_{\mathcal{C}} \coloneqq \psi_{\mathcal{G}}$ \deleted{as} the set operator generated by MCG $\mathcal{G} = (\mathcal{V},\mathcal{E},\mathcal{C})$. We define the mean empirical loss of each realization $\mathcal{C} \in \mathcal{F}$ of the DMNN architecture $\mathcal{A}$ as
\begin{linenomath}
	\begin{equation*}
		L_{\mathcal{D}_{N}}(\mathcal{C}) = \frac{1}{N} \sum_{j=1}^{N} \ell(X_{j},Y_{j};\psi_{\mathcal{C}}).
	\end{equation*}
\end{linenomath}

The optimal minimization of $L_{\mathcal{D}_{N}}$ in $\mathcal{F}$ is a highly combinatorial problem whose complexity is exponential on the number $|\mathcal{V}|$ of computational vertices, specially when $\mathcal{F}$ is the Cartesian product of subsets of Boolean lattices, as is the case of unrestricted CDMNN. To circumvent the computational complexity, we propose an iterative algorithm that performs a greedy search of $\mathcal{F}$ and returns a \textit{local minimum} of $L_{\mathcal{D}_{N}}$ in $\mathcal{F}$ that, although it may not be a minimizer of $L_{\mathcal{D}_{N}}$, \added{it} may be suitable for the application at hand. The proposed algorithm is the instantiation of a U-curve algorithm, that was initially proposed to minimize U-shaped functions in Boolean lattices \cite{u-curve3,ucurveParallel,reis2018,u-curve1}. The proposed algorithm is a lattice version of the gradient descent algorithm (GDA) \cite{ruder2016overview}, and we call it the \textit{lattice \deleted{gradient} descent algorithm} (L\deleted{G}DA). As in the GDA, we may add stochasticity to meaningfully decrease the problem complexity and do not get stuck on bad local \deleted{minimums} \added{minima}. 

In the following sections, we present the details of the L\deleted{G}DA and its stochastic version (SL\deleted{G}DA). 

\subsection{Lattice \deleted{gradient} Descent Algorithm for training CDMNN}

The L\deleted{G}DA performs an iterative greedy search of $\mathcal{F}$, at each step jumping to the point in the neighborhood of the current point with the least mean empirical loss, and stopping after a predefined number of steps (epochs). This algorithm is analogous to the gradient descent algorithm to minimize differentiable functions in $\mathbb{R}^{p}, p \geq 1,$ since both algorithms go, at each time, to the direction with the least function value. In the gradient descent algorithm, this direction is opposed to the function gradient, while in the L\deleted{G}DA this direction is a chain of $(\mathcal{F},\leq)$ that passes through the current point that contains its neighbor with least mean empirical loss.

The L\deleted{G}DA \deleted{algorithm} is formalized in Algorithm \ref{A1}. It is initialized at a point $\mathcal{C} \in \mathcal{F}$ and a predetermined number of training epochs is fixed. \deleted{The initial point is stored as the point with minimum empirical loss visited so far.} For each epoch, $\mathcal{C}$ is updated to a point in its neighborhood with the least empirical loss. If this point has an empirical loss lesser than that of the point with minimum loss visited so far, the current point is stored as the minimum. After all epochs, the algorithm returns the visited point with minimum empirical loss.

\begin{algorithm}[ht]
	\centering
	\caption{Lattice \deleted{gradient} Descent Algorithm for training a CDMNN.}
	\label{A1}
	\begin{algorithmic}[1]
		\ENSURE $\mathcal{C} \in \mathcal{F}$, Epochs
		\STATE $L_{min} \gets L_{\mathcal{D}_{N}}(\mathcal{C})$
		\STATE $\widehat{\mathcal{C}} \gets \mathcal{C}$ 
		\FOR{run $\in \{1,\dots,\text{Epochs}\}$}
		\STATE $\mathcal{C} \gets \mathcal{C}^{\prime} \text{ s.t. } \ \mathcal{C}^{\prime} \in \mathcal{N}(\mathcal{C}) \text{ and } L_{\mathcal{D}_{N}}(\mathcal{C}^\prime)  = \min\{L_{\mathcal{D}_{N}}(\mathcal{C}^{\prime\prime}): \mathcal{C}^{\prime\prime} \in \mathcal{N}(\mathcal{C})\}$
		\IF{$L_{\mathcal{D}_{N}}(\mathcal{C}) < L_{min}$}
		\STATE $L_{min} \gets L_{\mathcal{D}_{N}}(\mathcal{C})$
		\STATE $\widehat{\mathcal{C}} \gets \mathcal{C}$ 
		\ENDIF
		\ENDFOR
		\RETURN{$\widehat{\mathcal{C}}$}
	\end{algorithmic}
\end{algorithm}

\begin{remark}
	\label{remark_FnotH}
	On the one hand, since the poset $(\mathcal{F},\leq)$ is known a priori, for any $\mathcal{C} \in \mathcal{F}$ the set $\mathcal{N}(\mathcal{C})$ is known so the complexity of minimizing $L_{\mathcal{D}_{N}}$ in $\mathcal{N}(\mathcal{C})$ should be $|\mathcal{N}(\mathcal{C})| \mathcal{O}(L_{\mathcal{D}_{N}})$, in which $\mathcal{O}(L_{\mathcal{D}_{N}})$ is the computational complexity of $L_{\mathcal{D}_{N}}$. On the other hand, if the L\deleted{G}DA \deleted{algorithm} \deleted{were} \added{was} applied directly on poset $(\mathcal{H}(\mathcal{A}),\leq)$, fixed a $\psi \in \mathcal{H}(\mathcal{A})$, the computation of its neighborhood in $(\mathcal{H}(\mathcal{A}),\leq)$ would be problem-specific and could have a complexity that is not linear on the number of neighbors. Therefore, sub-optimally minimizing the empirical loss in $\mathcal{F}$ via the L\deleted{G}DA should be less computationally complex than doing so in $\mathcal{H}(\mathcal{A})$.
\end{remark}

Although the returned DMNN parameter $\widehat{\mathcal{C}}$ may not be a minimizer of $L_{\mathcal{D}_{N}}$ in $\mathcal{F}$, it may have a low enough mean empirical loss and well approximate the actual minimizer. Furthermore, Algorithm \ref{A1} may be run starting from many initial points $\mathcal{C}$ and lead to distinct local \deleted{minimums} \added{minima} $\widehat{\mathcal{C}}$, with mean empirical loss tending to that of the minimizer as more initial points are considered. We illustrate the L\deleted{G}DA \deleted{algorithm} in a sequential architecture.

\begin{example}[Training an ASF architecture]
	\normalfont \label{example_training_ASF}
	
	Consider the ASF architecture $\mathcal{A}_{3}$ in Figure \ref{fig_sequential_architecture} \deleted{for} \added{in} which $\mathcal{F}$ is isomorphic to the Boolean lattice $\mathcal{P}(W)^{2}$ (cf. Example \ref{example_ASF_LatticeF}). In Figure \ref{fig_F_ASF} we present selected elements of this Boolean lattice when $E = \mathbb{Z}^{2}$ and $W = \{-1,0,1\}^{2}$ is the three by three square centered at the origin of $E$. The number on top of the elements in Figure \ref{fig_F_ASF} is the mean empirical loss in a sample $\mathcal{D}_{N}$ of the set operator generated by the respective realization of the architecture. The colored edges represent the path of the L\deleted{G}DA starting from the empty sets after three epochs. The edges of paths that remove points from the structuring elements were omitted for a better visualization.
	
	In this example, each element in $\mathcal{F}$ has $18$ neighbors, which are obtained by adding or removing a point from one of its two sets, so the complexity of each step of the L\deleted{G}DA is at most eighteen times the complexity of applying the realized set operator to each input point in sample $\mathcal{D}_{N}$ and comparing the result with the respective target set to compute the mean empirical loss.
	
	\begin{figure}[ht]
		\centering
		\includegraphics[width=\textwidth]{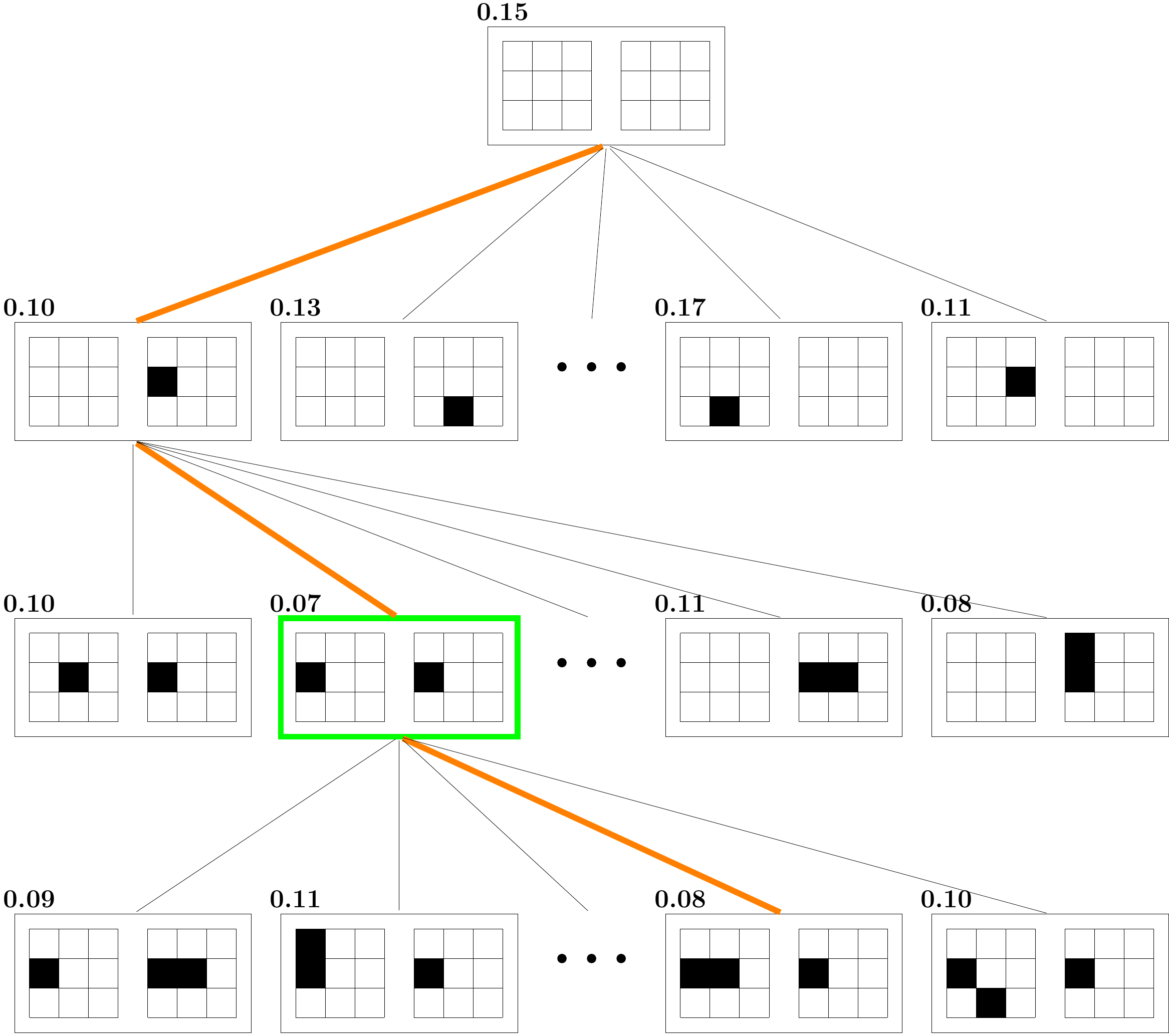}
		\caption{\footnotesize Selected elements in the lattice $(\mathcal{F},\leq)$ of the ASF architecture $\mathcal{A}_{3}$ in Figure \ref{fig_sequential_architecture} when $E = \mathbb{Z}^{2}$ and $W = \{-1,0,1\}^{2}$. The number on top of the elements is the mean empirical loss of the set operator realized by them, and the colored edges represent the path of the L\deleted{G}DA starting from the empty sets after three epochs. The edges pointing to sets obtained by removing a point from a set were omitted for a better visualization. The returned element $\hat{\mathcal{C}}$ is outlined in green.} \label{fig_F_ASF}
	\end{figure}	
	\hfill$\blacksquare$
\end{example}

Even though the L\deleted{G}DA does not perform an exhaustive search of $\mathcal{F}$, Algorithm \ref{A1} may be too complex and may not be feasible to solve practical problems. Furthermore, it may get stuck in local \deleted{minimums} \added{minima} or realize a periodical path, visiting the same points multiple times. Indeed, since at each step it jumps to the neighbor with the least empirical loss, it may jump back to the previous point if it is the neighbor with the least empirical loss, and keep jumping between these two points during the remaining epochs. Inspired by the success of stochastic gradient descent algorithms, we propose the Stochastic Lattice \deleted{gradient} Descent Algorithm (SL\deleted{G}DA) to decrease the computational complexity of the L\deleted{G}DA and mitigate the risk of getting stuck at local \deleted{minimums} \added{minima}.

\subsection{Stochastic Lattice \deleted{gradient} Descent Algorithm for training CDMNN}

The L\deleted{G}DA has two computational bottlenecks: the cost of calculating $L_{\mathcal{D}_{N}}(\mathcal{C})$ and the number of neighbors in $\mathcal{N}(\mathcal{C})$. In order to calculate $L_{\mathcal{D}_{N}}(\mathcal{C})$ it is necessary to compute the graph $(\mathcal{V},\mathcal{E},\mathcal{C})$ for all the $N$ inputs $\{X_{1},\dots,X_{N}\}$. Since it may be expensive to compute this graph, specially if the dimensionality of $X$ is too great\footnote{We define the dimensionality of $X$ as the size $|F|$ of $F$.}, the calculation of $L_{\mathcal{D}_{N}}$ may be an \deleted{impeditive} \added{impediment} of Algorithm \ref{A1} for great values of $N$. Moreover, if the number of neighbors of each point $\mathcal{C} \in \mathcal{F}$ is too great, then Algorithm \ref{A1} may not be computable since at each step it should calculate the mean empirical loss of every point in a neighborhood. We propose the SL\deleted{G}DA to circumvent these bottlenecks by considering sample batches to calculate the empirical loss and by sampling neighbors of a point at each batch. We also expect with this algorithm to avoid periodical paths and getting stuck at local \deleted{minimums} \added{minima}.

The SL\deleted{G}DA is formalized in Algorithm \ref{A2}. The initial point $\mathcal{C} \in \mathcal{F}$, a batch size\footnote{We assume that $N/b$ is an integer to easy notation. If this is not the case, the last batch will contain less than $b$ points.} $b$, the number $n$ of neighbors to be sampled at each step, and the number of training epochs is fixed. \deleted{The initial point is stored as the point with minimum empirical loss visited so far.} For each epoch, the sample $\mathcal{D}_{N}$ is randomly partitioned in $N/b$ batches $\{\tilde{\mathcal{D}}^{(1)}_{b},\dots,\tilde{\mathcal{D}}^{(N/b)}_{b}\}$. For each batch $\tilde{\mathcal{D}}^{(j)}_{b}$, $n$ neighbors of $\mathcal{C}$ are sampled and $\mathcal{C}$ is updated to the sampled neighbor with the least empirical loss $L_{\tilde{\mathcal{D}}^{(j)}_{b}}$, that is calculated on the sample batch $\tilde{\mathcal{D}}^{(j)}_{b}$. Observe that $\mathcal{C}$ is updated at each batch, so during an epoch, it is updated $N/b$ times. At the end of each epoch, the empirical loss $L_{\mathcal{D}_{N}}(\mathcal{C})$ of $\mathcal{C}$ on the whole sample $\mathcal{D}_{N}$ is compared with the loss of the point with the least empirical loss visited so far at the end of an epoch, and it is stored as this point if its empirical loss is lesser. After the predetermined number of epochs, the algorithm returns the point with the least empirical loss on the whole sample $\mathcal{D}_{N}$ visited at the end of an epoch.

\begin{algorithm}[ht]
	\centering
	\caption{Stochastic Lattice \deleted{gradient} Descent Algorithm for training CDMNN.}
	\label{A2}
	\begin{algorithmic}[1]
		\ENSURE $\mathcal{C} \in \mathcal{F}, n, b, Epochs$		
		\STATE $L_{min} \gets L_{\mathcal{D}_{N}}(\mathcal{C})$
		\STATE $\widehat{\mathcal{C}} \gets \mathcal{C}$ 
		\FOR{run $\in \{1,\dots,\text{Epochs}\}$}
		\STATE $\{\tilde{\mathcal{D}}^{(1)}_{b},\dots,\tilde{\mathcal{D}}^{(N/b)}_{b}\} \gets \text{SampleBatch}(\mathcal{D}_{N},b)$
		\FOR{$j \in \{1,\dots,N/b\}$}
		\STATE $\tilde{N}(\mathcal{C}) \gets \text{SampleNeighbors}(\mathcal{C},n)$	
		\STATE $\mathcal{C} \gets \mathcal{C}^{\prime} \text{ s.t. } \ \mathcal{C}^{\prime} \in \tilde{N}(\mathcal{C}) \text{ and } L_{\tilde{\mathcal{D}}^{(j)}_{b}}(\mathcal{C}^\prime)  = \min\{L_{\tilde{\mathcal{D}}^{(j)}_{b}}(\mathcal{C}^{\prime\prime}): \mathcal{C}^{\prime\prime} \in \tilde{N}(\mathcal{C})\}$
		\ENDFOR
		\IF{$L_{\mathcal{D}_{N}}(\mathcal{C}) < L_{min}$}
		\STATE $L_{min} \gets L_{\mathcal{D}_{N}}(\mathcal{C})$
		\STATE $\widehat{\mathcal{C}} \gets \mathcal{C}$ 
		\ENDIF
		\ENDFOR
		\RETURN{$\widehat{\mathcal{C}}$}
	\end{algorithmic}
\end{algorithm}

The computational complexity of each epoch of Algorithm \ref{A2} is of order $N/b\left[n \mathcal{O}(L_{\tilde{\mathcal{D}}^{(j)}_{b}})\right]\\ + \mathcal{O}(L_{\mathcal{D}_{N}})$ so it is controlled by the batch size $b$ and the number of sampled neighbors $n$. The SL\deleted{G}DA could be applied to find an initial point for the L\deleted{G}DA, or when the L\deleted{G}DA gets stuck at a local minimum or on a periodical path. Moreover, the L\deleted{G}DA and SL\deleted{G}DA may be combined by, for example, taking sample batches, but considering all neighbors of a point without sampling. In the next section, we illustrate the algorithms on \deleted{an applied} \added{a proof-of-concept} example.

\section{Boundary recognition of digits with noise}
\label{Sec_Applications}

In this section, we apply the SL\deleted{G}DA and L\deleted{G}DA to train CDMNN to recognize the boundary of \added{noisy} \deleted{noised} digits. The algorithms were implemented in \textbf{R} \cite{R} and are available as a \textit{package} at \url{https://github.com/dmarcondes/DMNN}. The basic morphological operators are calculated by the functions in the \textbf{mmand} package \cite{mmand}. The training was performed on a \added{server} \deleted{personal computer} with processor \added{Intel(R) Xeon(R) Platinum 8358 CPU @ 2.60GHz x 16 and 32 GB of RAM} \deleted{13th Gen Intel Core i7-1355U x 12 and 16 GB of ram}. The calculation of the empirical loss of the neighbors in $\mathcal{N}(\mathcal{C})$ or $\tilde{N}(\mathcal{C})$ was parallelized in \added{16} \deleted{12} cores.

The training sample is composed of ten pairs of $56 \times 56$ images, one of each digit from \deleted{0} \added{zero} to \deleted{9} \added{nine}, and the validation sample is composed of another ten pairs of $56 \times 56$ images, one of each digit from \deleted{0} \added{zero} to \deleted{9} \added{nine}. The input images are black and white digits with \added{salt-and-pepper} \deleted{pepper and salt} noise, and the output images are the boundary of the input digit. The input and output images in the validation sample are presented in Figure \ref{fig_SLDA_best_digits}.

We trained ten CDMNN architectures, which we denote by the operators computed in their layers. We use the notation $kSGd$ to mean a layer that computes the supremum of $k$ sup-generating operators locally defined within $W_{d} \coloneqq \{-(d-1)/2,\dots,(d-1)/2\}^2$ and we denote by $ASFd$ a layer that calculates an ASF with structuring element in $\mathcal{P}(W_{d})$. For example, we denote by $16SG3$ the architecture given by the supremum of \deleted{16} \added{sixteen} sup-generating operators locally defined within $W_{3} = \{-1,0,1\}^2$ and by $ASF3-16SG3$ the architecture whose output is the supremum of \deleted{16} \added{sixteen} sup-generating operators locally defined in $W_{3}$ applied to the output of an ASF with structuring element in $\mathcal{P}(W_{3})$. The ten considered architectures are presented in Table \ref{table_res_SLDA_digits} and illustrated in Figure \ref{fig_archs_digits}. In the figure, we omitted the operators parameters and the orientation of the graph, that is always from vertex $X$ to vertex $\psi$, for a better visualization.

\begin{figure}[ht]
	
	\begin{subfigure}{0.475\linewidth}
		\includegraphics[width=\linewidth]{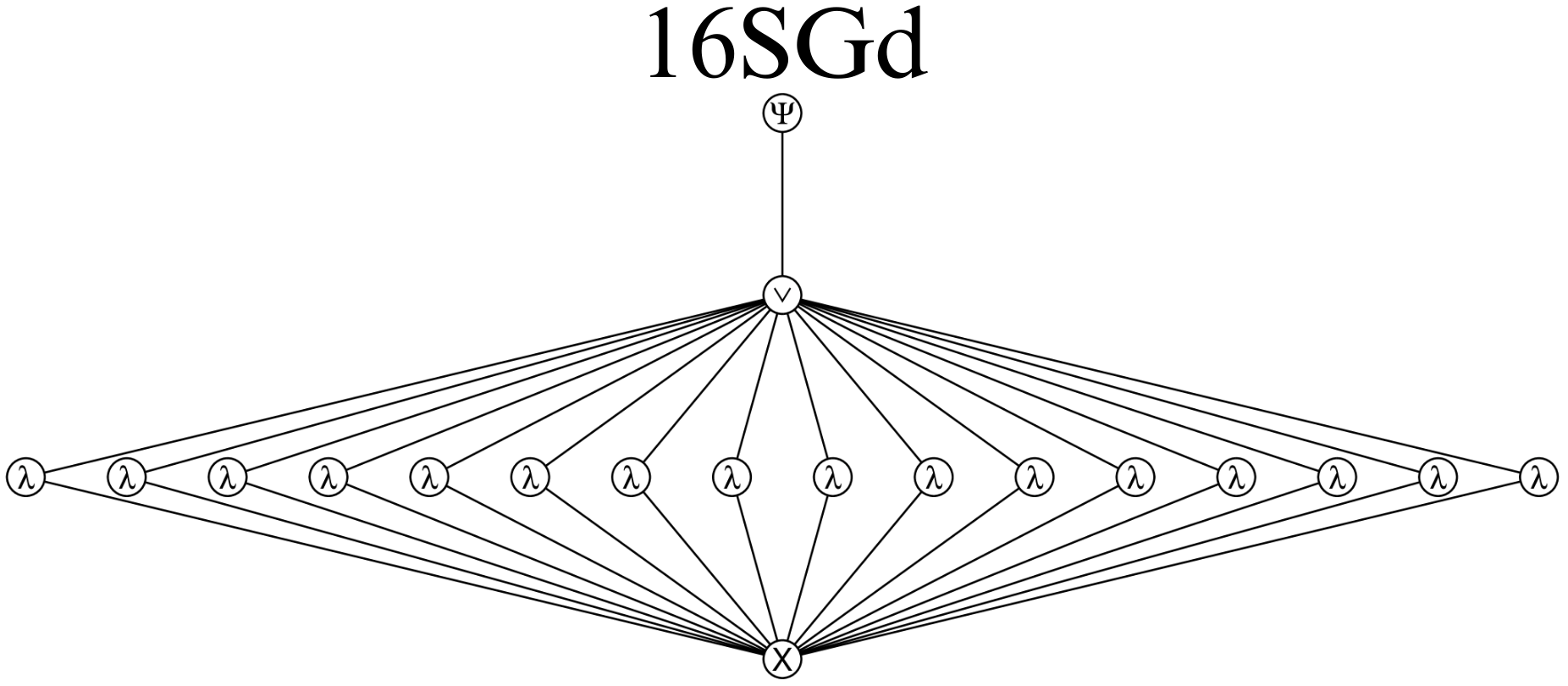}
	\end{subfigure}\hfill
	\begin{subfigure}{0.475\linewidth}
		\includegraphics[width=\linewidth]{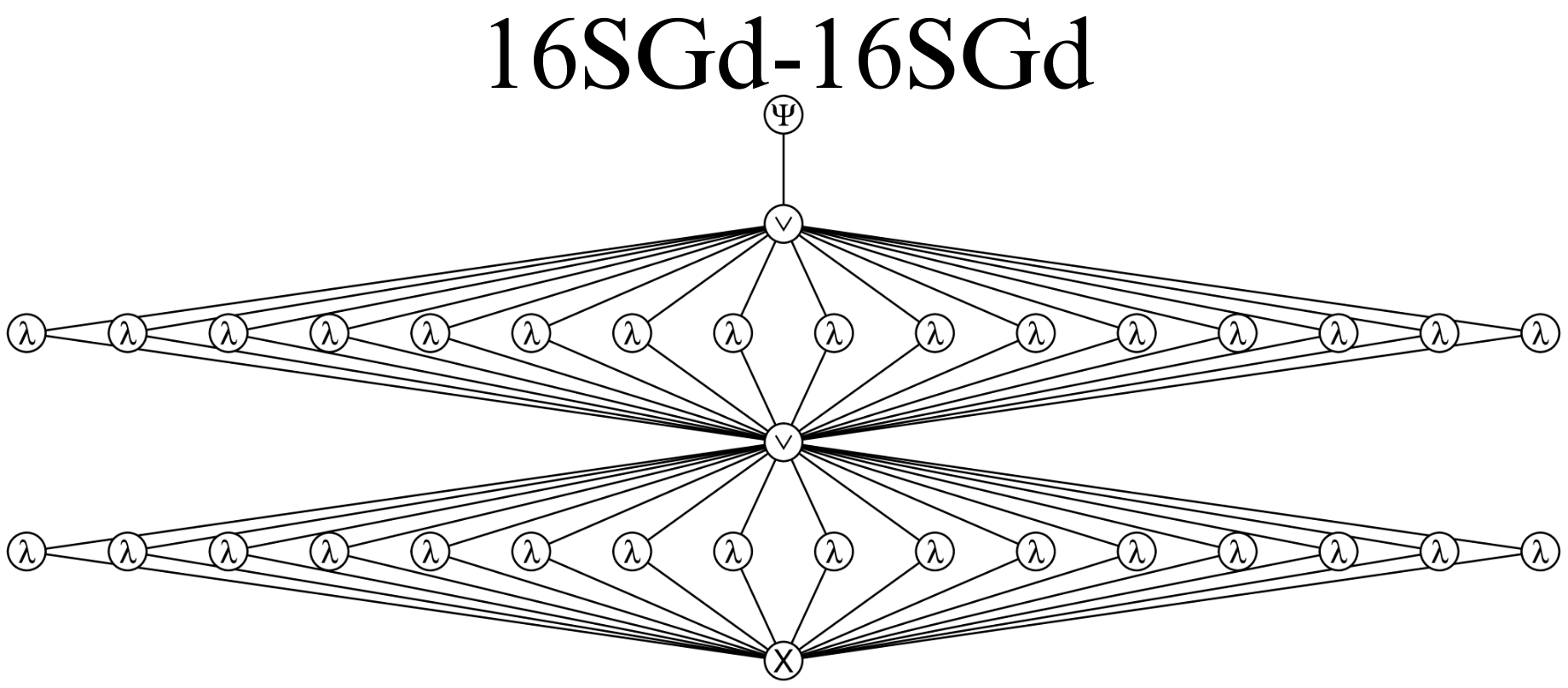}
	\end{subfigure}
	
	\medskip
	\begin{subfigure}{0.475\linewidth}
		\includegraphics[width=\linewidth]{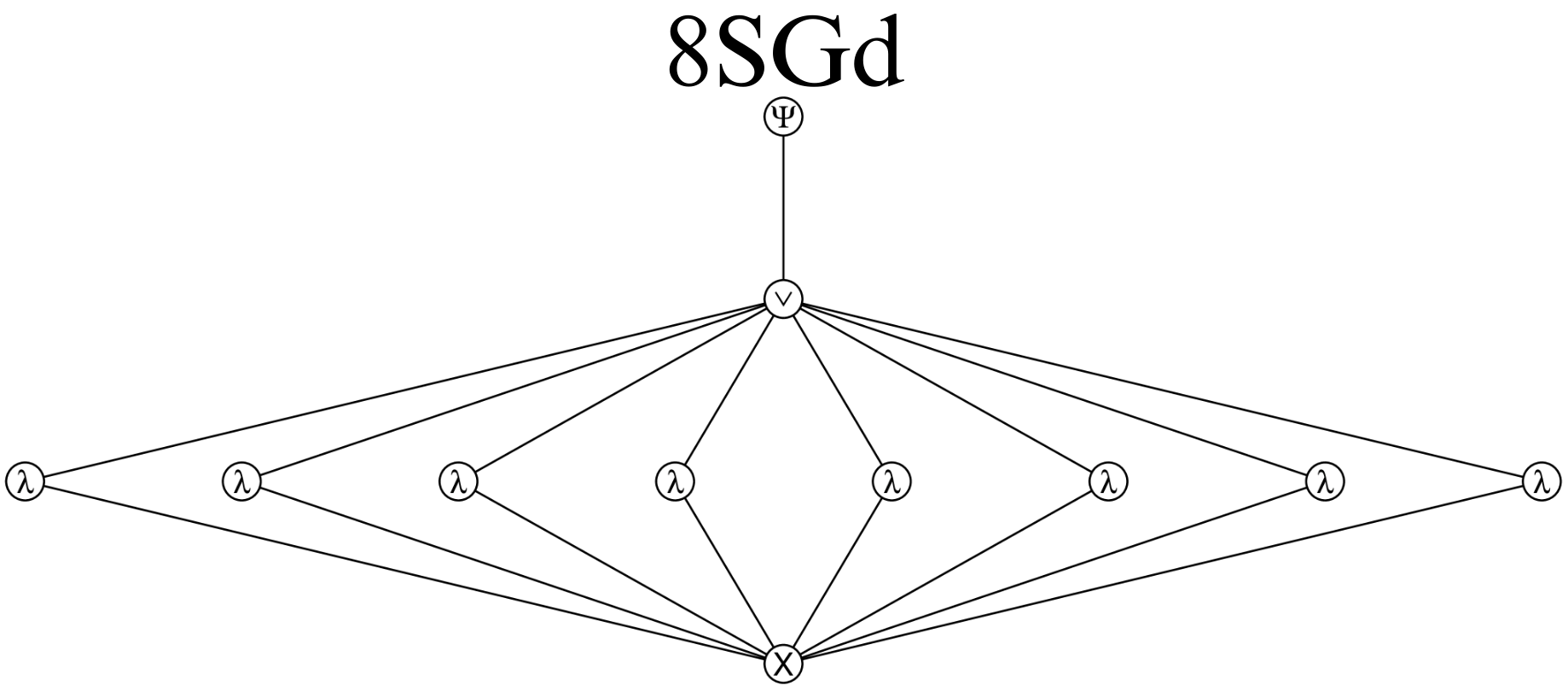}
	\end{subfigure}\hfill
	\begin{subfigure}{0.475\linewidth}
		\includegraphics[width=\linewidth]{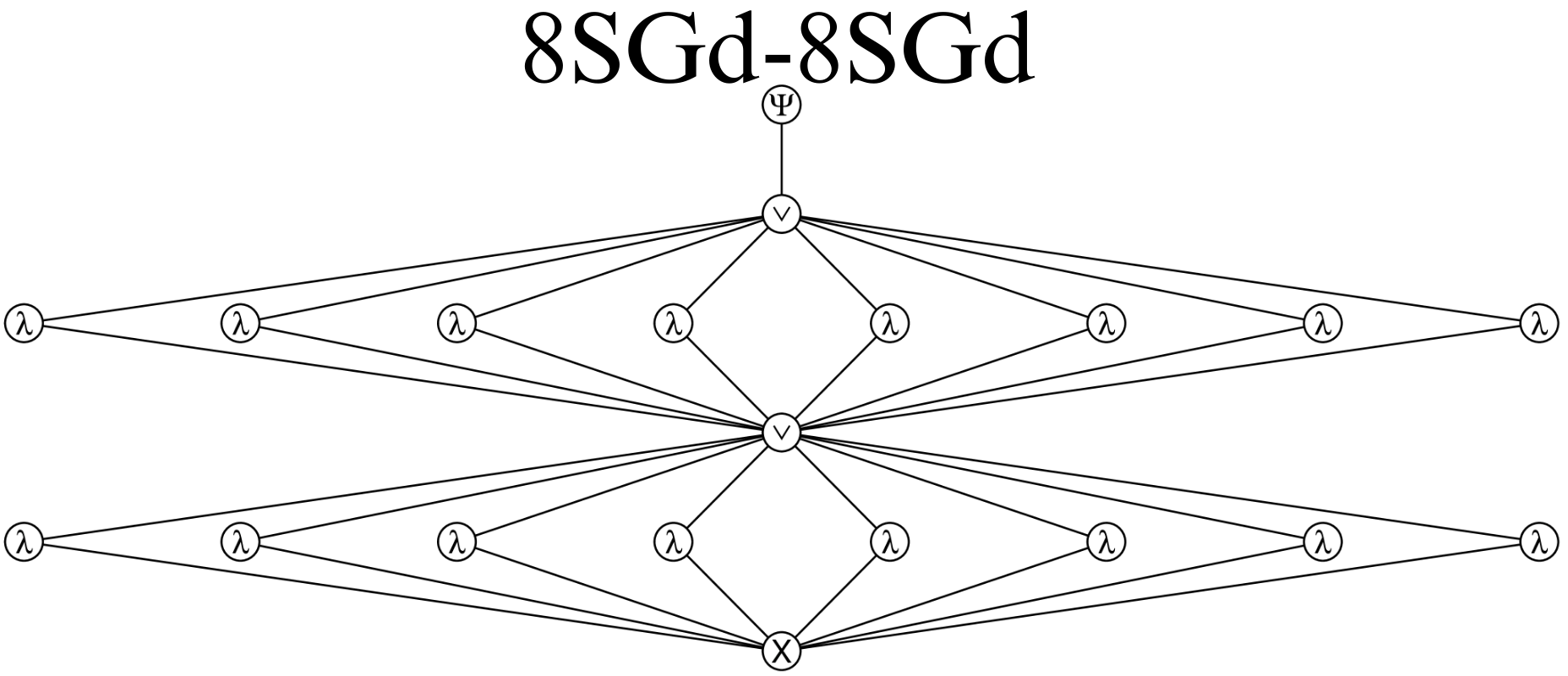}
	\end{subfigure}
	
	\medskip
	\begin{subfigure}{0.475\linewidth}
		\includegraphics[width=\linewidth]{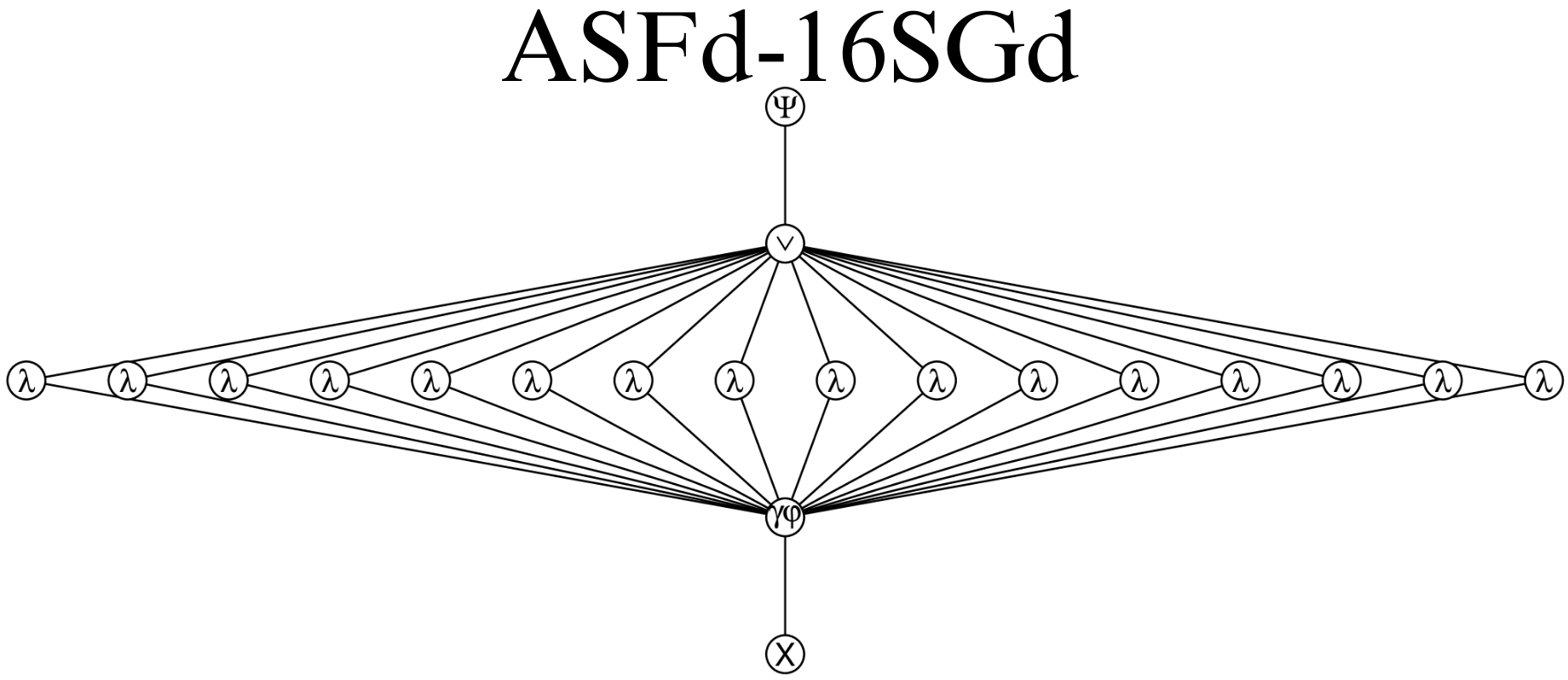}
	\end{subfigure}\hfill
	\begin{subfigure}{0.475\linewidth}
		\includegraphics[width=\linewidth]{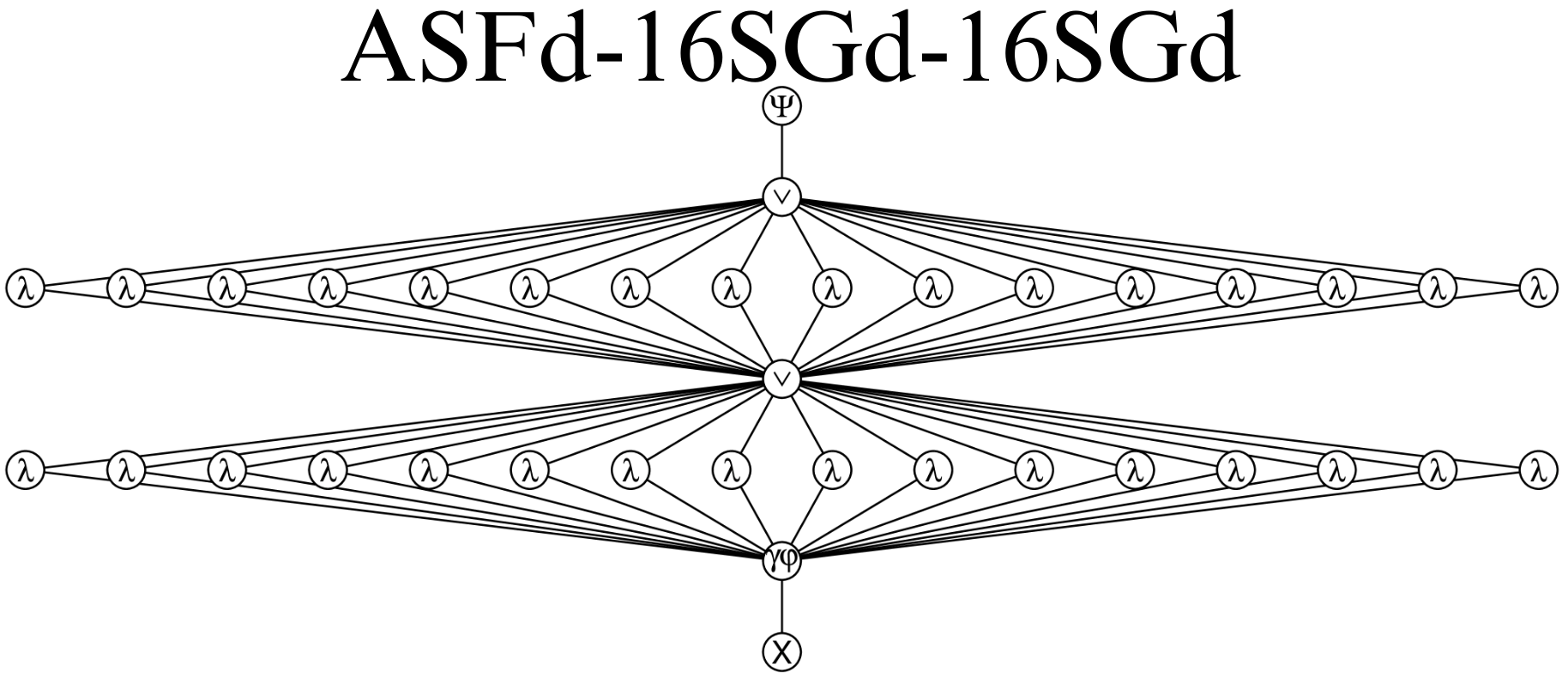}
	\end{subfigure}
	
	\medskip
	\begin{subfigure}{0.475\linewidth}
		\includegraphics[width=\linewidth]{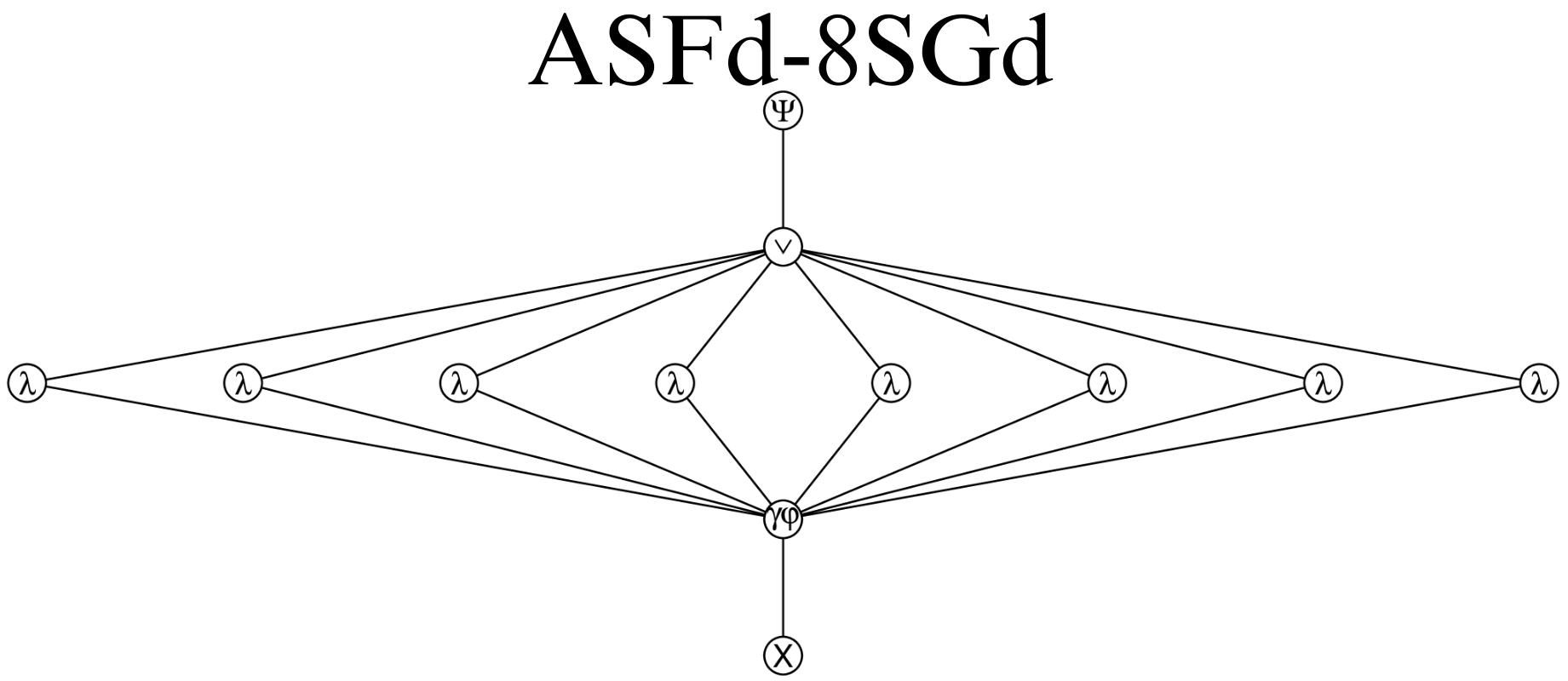}
	\end{subfigure}\hfill
	\begin{subfigure}{0.475\linewidth}
		\includegraphics[width=\linewidth]{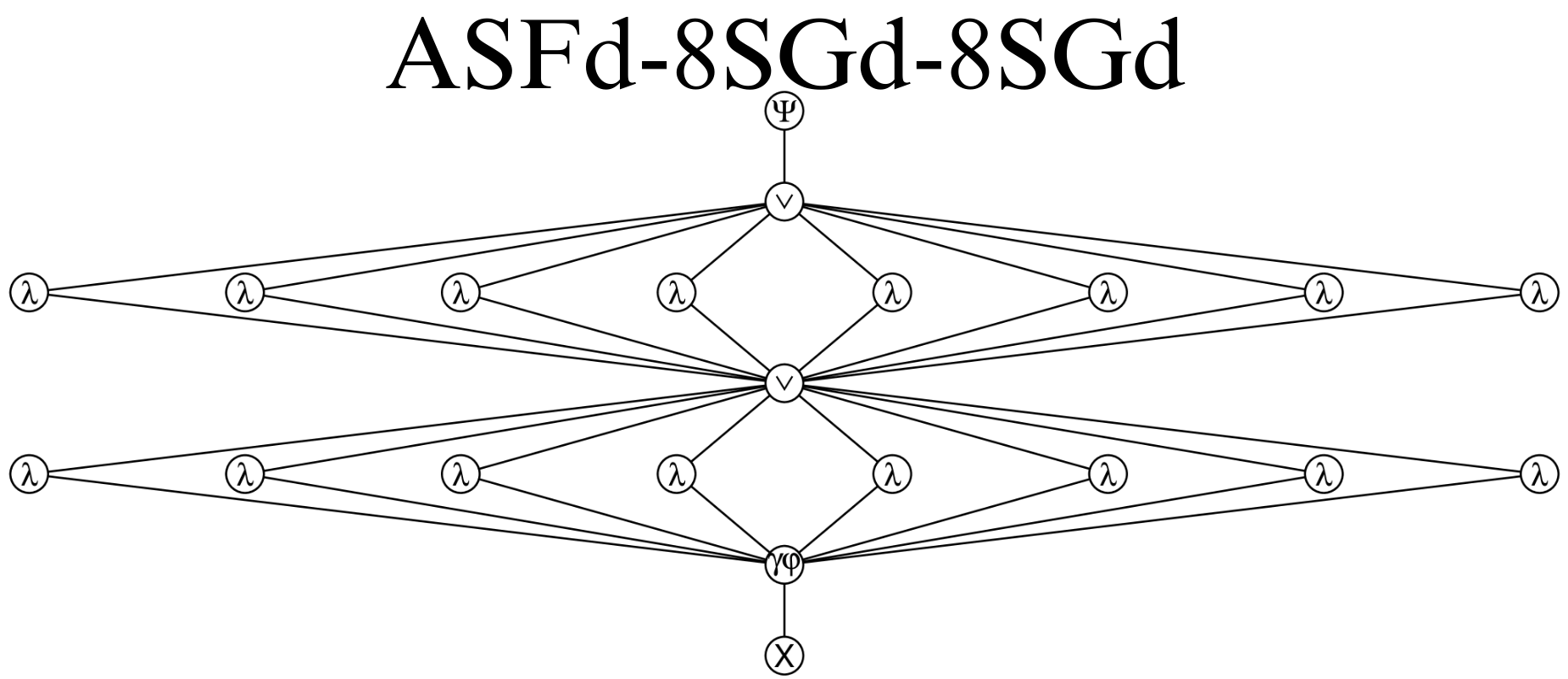}
	\end{subfigure}
	\caption{\footnotesize CDMNN architectures trained to recognize the boundary of digits with noise. The parameters of the operators and the orientation of the edges, that is, from the bottom to the top, were omitted for a better visualization.}
	\label{fig_archs_digits}
\end{figure}

The ten architectures were trained for $1{,}000$ epochs with the SL\deleted{G}DA sampling \deleted{$10$} \added{sixteen} neighbors and considering batch sizes of \deleted{$1, 5$ and $10$} \added{one, five and ten} (whole sample). We considered the IoU loss function and the parameters were initialized so each vertex compute an operator given by a random perturbation of the identity operator. We considered the same initial point of each architecture for the SL\deleted{G}DA with batch sizes of \deleted{$1, 5$ and $10$} \added{one, five and ten}. Finally, when more than one sampled neighbor had \deleted{a} \added{the} same least empirical loss in a step of the algorithm, the next point of the algorithm was uniformly sampled from these neighbors. \added{Each training scenario was replicated ten times starting from ten distinct initial points. We present the minimum, mean and standard deviation of the results among the ten replications.}

The results are presented in Table \ref{table_res_SLDA_digits}. First, we observe that the architectures with higher mean IoU loss on the validation sample are those in which there are operators locally defined within $W_{5}$. In these cases, the class of operators generated by the architecture is more complex and the respective poset $(\mathcal{F},\leq)$ has more points, so more samples and more epochs should be necessary to properly train these more complex architectures. It is also interesting to note that, for example, even though the architectures ASF3-16SG3-16SG3 and ASF3-16SG5 generate W-operators locally defined within a same window, the performance of ASF3-16SG3-16SG3 was significantly better, since it had \added{in average} a lesser validation loss. This can be explained by the fact that the collection of W-operators generated by ASF3-16SG3-16SG3 is more restricted than that generated by ASF3-16SG5, so fewer epochs and samples are necessary to properly train it.

Figure \ref{fig_SLDA_train_digits} presents, for each architecture and batch size, the mean empirical IoU at each epoch, \added{of each repetition and the median over the repetitions}. We see in this figure and in Table \ref{table_res_SLDA_digits} that there were cases in which the minimum was visited in later epochs and others in which it was visited in early epochs. On the one hand, when the batch size equals \deleted{$5$ or $10$} \added{five or ten}, the empirical IoU usually decreases abruptly from time to time until it reaches the minimum visited and, once a low value of the empirical loss is attained, it does not vary significantly in later epochs. On the other hand, when the batch size is \deleted{$1$} \added{one}, the empirical IoU decreases quickly in the early epochs and is highly unstable during later epochs, coming back to high values and not significantly decreasing the empirical loss. \deleted{Indeed, apart from the architecture ASF3-8SG5, the minimum when the batch size is $5$ or $10$ was achieved, with few exceptions, between epochs $700$ and $1,000$, while when the batch size was 1 the minimum was achieved at as early as the epoch $133$.}

The lowest \added{minimum} validation loss \added{among the repetitions} was obtained when training with a batch size of \added{five for all architectures}. \deleted{$1$ for two architectures, and with a batch size of $5$ and $10$ for four architectures each, and} The lowest validation loss overall was obtained \added{by a repetition of} \deleted{for} the architecture \deleted{ASF3-16SG3-16SG3} \added{ASF3-8SG3-8SG3} trained with a batch size of \added{five.} \deleted{1.} The $1,000$ training epochs took from around \deleted{$7$ to $27$} \added{11 to 45} minutes, and training with a batch size of \deleted{$1$} \added{one} took longer in all cases. 

\added{The results in Table \ref{table_res_SLDA_digits} evidence that some stochasticity (batch size of five) aids the algorithm to not get stuck at bad local minima, while too much stochasticity (batch size of one) can be actually detrimental. We observe that the standard deviation over the repetitions is lesser for the batch size one in the majority of cases, so the initial point has less influence on the algorithm in this case. However, considering a batch size of one may lead to a worse local minima than that reached with batch size five, as evidenced by the fact that the minimum validation loss among the repetitions was always attained with a batch size of five. But there is a nuance: even though a batch size five can attain better minima, in average it might attain a local minima with greater validation loss than that attained with batch sizes one and ten.}

\begin{table}[ht]
	\centering
	\caption{\footnotesize Results \added{in the form Minimum - Mean (Standard Deviation) among the ten repetitions} of the training of ten CDMNN architectures to recognize the boundary of digits with noise by the SL\deleted{G}DA sampling \added{$16$} \deleted{10} neighbors and considering sample batches of size $b = 1, b = 5$ and $b = 10$, and the IoU loss. For each architecture and batch size, we present the minimum empirical loss, the empirical loss on the validation sample of the trained CDMNN, the total training time, the time it took to visit the minimum and the epochs it took to visit the minimum.} \label{table_res_SLDA_digits}
	\resizebox{\linewidth}{!}{\begin{tabular}{lllllll}
			\hline
			Arch & b & $L_{\mathcal{D}_{N}}(\widehat{\mathcal{C}})$ & Validation loss & Total time (m) & Time to minimum (m) & Epochs to minimum \\ 
			\hline
			\multirow{3}{*}{16sg3} & 1 & 0.054 - 0.062 (0.005) & 0.069 - 0.073 (0.004) & 20.178 - 33.887 (11.613) & 7.823 - 22.159 (11.372) & 272 - 655.2 (221.684) \\ 
			& 5 & 0.048 - 0.067 (0.019) & 0.066 - 0.083 (0.017) & 11.631 - 22.245 (8.816) & 8.721 - 17.425 (7.701) & 514 - 787.6 (146.957) \\ 
			& 10 & 0.053 - 0.076 (0.013) & 0.071 - 0.093 (0.013) & 11.17 - 21.496 (8.799) & 9.453 - 16.794 (7.023) & 562 - 789.2 (114.652) \\ 
			\hline \multirow{3}{*}{16sg3\_16sg3} & 1 & 0.051 - 0.076 (0.017) & 0.067 - 0.092 (0.018) & 24.473 - 47.18 (18.985) & 6.677 - 25.473 (22.451) & 220 - 504.1 (275.083) \\ 
			& 5 & 0.028 - 0.081 (0.032) & 0.048 - 0.099 (0.03) & 18.604 - 36.492 (14.737) & 14.154 - 29.561 (15.65) & 474 - 791.2 (170.885) \\ 
			& 10 & 0.049 - 0.111 (0.053) & 0.072 - 0.127 (0.05) & 18.432 - 35.154 (14.1) & 18.375 - 33.58 (13.485) & 825 - 959.3 (50.844) \\ 
			\hline \multirow{3}{*}{16sg5} & 1 & 0.134 - 0.194 (0.025) & 0.143 - 0.201 (0.025) & 25.468 - 44.241 (15.362) & 2.104 - 20.745 (21.405) & 75 - 409.8 (310.944) \\ 
			& 5 & 0.087 - 0.107 (0.025) & 0.099 - 0.12 (0.023) & 19.598 - 32.65 (10.055) & 8.189 - 24.132 (9.782) & 374 - 742.4 (203.628) \\ 
			& 10 & 0.103 - 0.133 (0.023) & 0.121 - 0.148 (0.025) & 19.673 - 31.204 (9.657) & 11.528 - 22.207 (10.578) & 462 - 711.7 (234.54) \\ 
			\hline \multirow{3}{*}{8sg3} & 1 & 0.063 - 0.065 (0.001) & 0.072 - 0.074 (0.001) & 18.274 - 28.602 (8.571) & 4.594 - 15.427 (6.268) & 105 - 576.9 (237.392) \\ 
			& 5 & 0.055 - 0.07 (0.012) & 0.072 - 0.082 (0.012) & 8.471 - 16.402 (6.483) & 2.904 - 10.066 (7.259) & 117 - 604.7 (276.737) \\ 
			& 10 & 0.062 - 0.086 (0.02) & 0.073 - 0.095 (0.019) & 7.819 - 15.017 (6.019) & 3.009 - 11.332 (5.723) & 391 - 736.7 (185.475) \\ 
			\hline \multirow{3}{*}{8sg3\_8sg3} & 1 & 0.08 - 0.103 (0.017) & 0.079 - 0.11 (0.02) & 20.621 - 35.227 (11.869) & 2.63 - 13.339 (9.163) & 50 - 424.5 (297.31) \\ 
			& 5 & 0.034 - 0.077 (0.035) & 0.06 - 0.093 (0.031) & 12.192 - 23.075 (9.093) & 4.695 - 19.25 (9.707) & 387 - 821.9 (227.375) \\ 
			& 10 & 0.046 - 0.108 (0.054) & 0.066 - 0.125 (0.05) & 11.956 - 22.134 (8.752) & 7.256 - 16.198 (6.797) & 490 - 752.2 (186.995) \\ 
			\hline \multirow{3}{*}{8sg5} & 1 & 0.171 - 0.241 (0.076) & 0.173 - 0.246 (0.075) & 21.907 - 33.627 (9.656) & 0.853 - 16.397 (14.526) & 21 - 494.9 (389.379) \\ 
			& 5 & 0.133 - 0.177 (0.048) & 0.144 - 0.19 (0.047) & 12.915 - 20.851 (6.806) & 1.701 - 13.759 (6.921) & 104 - 650.8 (296.229) \\ 
			& 10 & 0.144 - 0.211 (0.048) & 0.158 - 0.221 (0.046) & 11.857 - 19.738 (6.282) & 2.223 - 7.486 (3.114) & 185 - 415.3 (230.014) \\ 
			\hline \multirow{3}{*}{asf3\_16sg3\_16sg3} & 1 & 0.047 - 0.078 (0.031) & 0.064 - 0.088 (0.027) & 24.812 - 46.998 (17.418) & 6.766 - 22.113 (13.934) & 118 - 514.9 (328.685) \\ 
			& 5 & 0.036 - 0.095 (0.077) & 0.051 - 0.107 (0.077) & 19.514 - 37.025 (14.135) & 7.46 - 28.158 (14.379) & 347 - 753.7 (256.399) \\ 
			& 10 & 0.049 - 0.128 (0.071) & 0.061 - 0.139 (0.071) & 20.408 - 35.461 (13.193) & 9.389 - 26.275 (10.645) & 463 - 746.4 (190.977) \\ 
			\hline \multirow{3}{*}{asf3\_16sg5} & 1 & 0.155 - 0.202 (0.04) & 0.166 - 0.207 (0.037) & 25.628 - 43.738 (14.173) & 0.359 - 11.664 (13.297) & 8 - 250 (263.799) \\ 
			& 5 & 0.057 - 0.13 (0.043) & 0.071 - 0.139 (0.042) & 20.448 - 32.518 (9.423) & 9.108 - 23.211 (7.316) & 455 - 716.7 (137.711) \\ 
			& 10 & 0.059 - 0.149 (0.058) & 0.072 - 0.158 (0.055) & 19.372 - 30.975 (9.155) & 7.435 - 23.176 (9.326) & 361 - 743.4 (186.098) \\ 
			\hline \multirow{3}{*}{\textbf{asf3\_8sg3\_8sg3}} & 1 & 0.051 - 0.069 (0.012) & 0.061 - 0.078 (0.013) & 21.516 - 35.002 (10.969) & 5.99 - 23.121 (12.349) & 209 - 670.2 (305.001) \\ 
			& \textbf{5} & \textbf{0.036 - 0.068 (0.028)} & \textbf{0.046 - 0.081 (0.025)} & \textbf{12.668 - 23.337 (8.8)} & \textbf{3.951 - 17.977 (8.403)} & \textbf{279 - 769.5 (232.824)} \\ 
			& 10 & 0.07 - 0.145 (0.093) & 0.079 - 0.156 (0.091) & 11.863 - 22.06 (8.242) & 11.039 - 18.389 (5.671) & 637 - 860.9 (144.989) \\ 
			\hline \multirow{3}{*}{asf3\_8sg5} & 1 & 0.119 - 0.214 (0.066) & 0.125 - 0.219 (0.065) & 21.957 - 33.581 (9.231) & 0.261 - 15.937 (12.809) & 12 - 430 (256.373) \\ 
			& 5 & 0.084 - 0.157 (0.042) & 0.093 - 0.165 (0.04) & 13.226 - 21.095 (6.523) & 4.991 - 17.598 (7.722) & 390 - 812.8 (217.689) \\ 
			& 10 & 0.115 - 0.224 (0.101) & 0.128 - 0.231 (0.097) & 12.184 - 19.81 (6.286) & 3.607 - 11.009 (8.381) & 189 - 565.2 (339.365) \\
			\hline
	\end{tabular}}
\end{table}

\begin{figure}[ht]
	\centering
	\includegraphics[width=\linewidth]{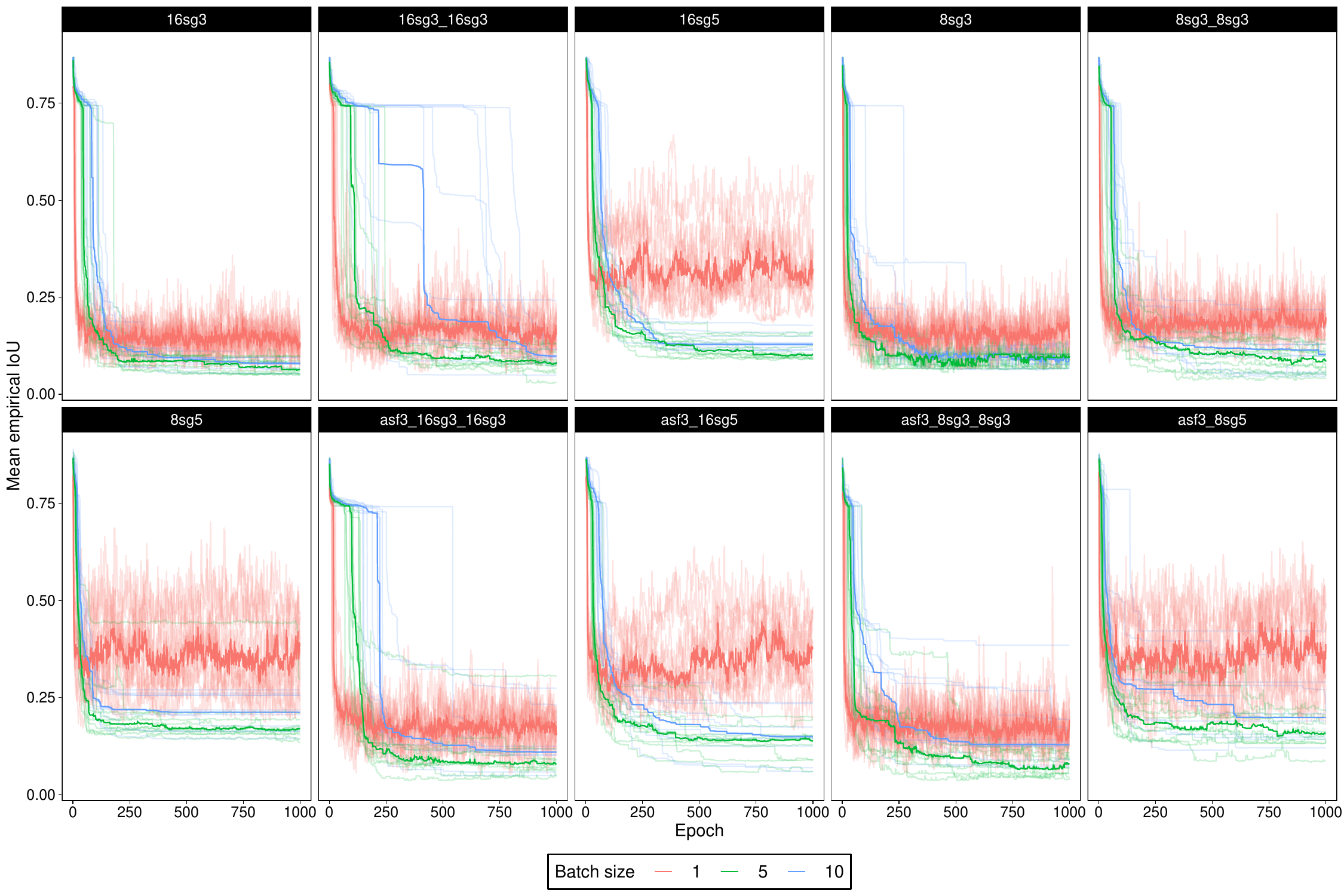}
	\caption{\footnotesize Mean empirical IoU at each epoch of the SL\deleted{G}DA for each architecture and batch size. \added{The transparent lines represent the repetitions from distinct initial points and the solid lines represent the median among the repetitions.}} \label{fig_SLDA_train_digits}
\end{figure}

The trained CDMNN that obtained the lesser validation loss \deleted{(ASF3-16SG3-16SG3 with $b = 1$)} \added{(a repetition of ASF3-8SG3-8SG3 with $b = 5$)} was fine trained via the L\deleted{G}DA for another \deleted{$100$} \added{$1,000$} epochs with batch sizes of \added{one, five and ten} \deleted{$1$, $5$ and $10$} (whole sample). Considering the CDMNN trained via the SL\deleted{G}DA as the initial point, the L\deleted{G}DA took \deleted{around 21} \added{between 70 and 80} minutes to fine train this architecture during \deleted{$100$} \added{$1,000$} epochs, \deleted{but} \added{and} achieved a lesser training and validation losses for \added{the batch sizes five and ten,} \deleted{all batch sizes,} as presented in Table \ref{table_res_refined_SLDA_digits}. The lowest validation loss was achieved with the \added{batch size 5.} \deleted{whole sample as a batch, that is the L\deleted{G}DA described in Algorithm \ref{A1}, and the total training time of this architecture was around 48 minutes.} We see in Figure \ref{fig_SLDA_refine_digits} a behavior analogous to that of Figure  \ref{fig_SLDA_train_digits}, since the empirical loss with batch sizes of $5$ and $10$ decreased abruptly during the training, while that of a sample batch size of $1$ was highly unstable during training. 

We also retrained the \deleted{ASF3-16SG3-16SG3} \added{ASF3-8SG3-8SG3} architecture from scratch via the L\deleted{G}DA with batch sizes of \deleted{$1, 5$ and $10$} \added{one, five and ten} during $1,000$ epochs and obtained the results in Table \ref{table_res_refined_SLDA_digits}. The training epochs took \added{around one hour} \deleted{at least 3 hours} to complete, and both the training and validation loss of the trained DMNN were \added{slightly lesser than that obtained by training via the SLDA, but were} higher than that obtained \deleted{via the SL\deleted{G}DA and} by refining the SL\deleted{G}DA result via the L\deleted{G}DA. In this example, the \deleted{re was no} \added{slight} gain in \added{performance by} training via the L\deleted{G}DA \added{may not be worth the increase on training time when compared to training with the SLDA,} and the best result was attained by refining via the L\deleted{G}DA the CDMNN trained via the SL\deleted{G}DA.

\begin{table}[ht]
	\centering
	\caption{\footnotesize Results of the fine training of the CDMNN architecture \deleted{ASF3-16SG3-16SG3} \added{ASF3-8SG3-8SG3} via \deleted{$100$} \added{$1,000$} epochs of the L\deleted{G}DA starting from the parameters trained by the SL\deleted{G}DA with batch size of \deleted{$1$} \added{five} and of retraining the CDMNN architecture \deleted{ASF3-16SG3-16SG3} \added{ASF3-8SG3-8SG3} from scratch via $1,000$ epochs of the L\deleted{G}DA starting from a random perturbation of the identity operator. \added{The architecture was retrained from ten distinct initial points and the results are presented in the form Minimum - Mean (Standard Deviation) among the ten repetitions.} The fine training and retraining were performed considering sample batches of size $b = 1, b = 5$ and $b = 10$, and the IoU loss. For each training type and batch size, we present the minimum empirical loss, the empirical loss on the validation sample of the trained CDMNN, the number of training epochs, the total training time, the time it took to visit the minimum and the epochs it took to visit the minimum.} \label{table_res_refined_SLDA_digits}
	\resizebox{\linewidth}{!}{\begin{tabular}{l|cccccc}
			\hline
			Type & b & $L_{\mathcal{D}_{N}}(\widehat{\mathcal{C}})$ & Val. loss &  Total time (m) & Time to min (m) & Epochs to min \\ 
			\hline
			\multirow{3}{*}{\textbf{Refine}} & 1 & 0.037 & 0.046 & 79.497 & 0.129 & 1 \\ 
			& \textbf{5} & \textbf{0.029} & \textbf{0.042} & \textbf{73.995} & \textbf{67.649} & \textbf{913} \\ 
			& 10 & 0.032 & 0.043 & 70.979 & 49.739 & 704 \\ 
			\hline
			\multirow{3}{*}{Retrain} & 1 & 0.027 - 0.057 (0.02) & 0.045 - 0.072 (0.018) & 80.8 - 84.202 (3.441) & 6.225 - 61.005 (25.824) & 69 - 724.9 (298.998) \\ 
			& 5 & 0.036 - 0.076 (0.029) & 0.047 - 0.09 (0.03) & 65.34 - 67.609 (1.764) & 26.876 - 54.857 (11.983) & 406 - 812.9 (174.766) \\ 
			& 10 & 0.033 - 0.106 (0.066) & 0.049 - 0.124 (0.065) & 63.081 - 65.672 (1.331) & 29.682 - 49.569 (11.388) & 462 - 755.1 (167.964) \\
			\hline
	\end{tabular}}
\end{table}

\begin{figure}[ht]
	\centering
	\includegraphics[width=\linewidth]{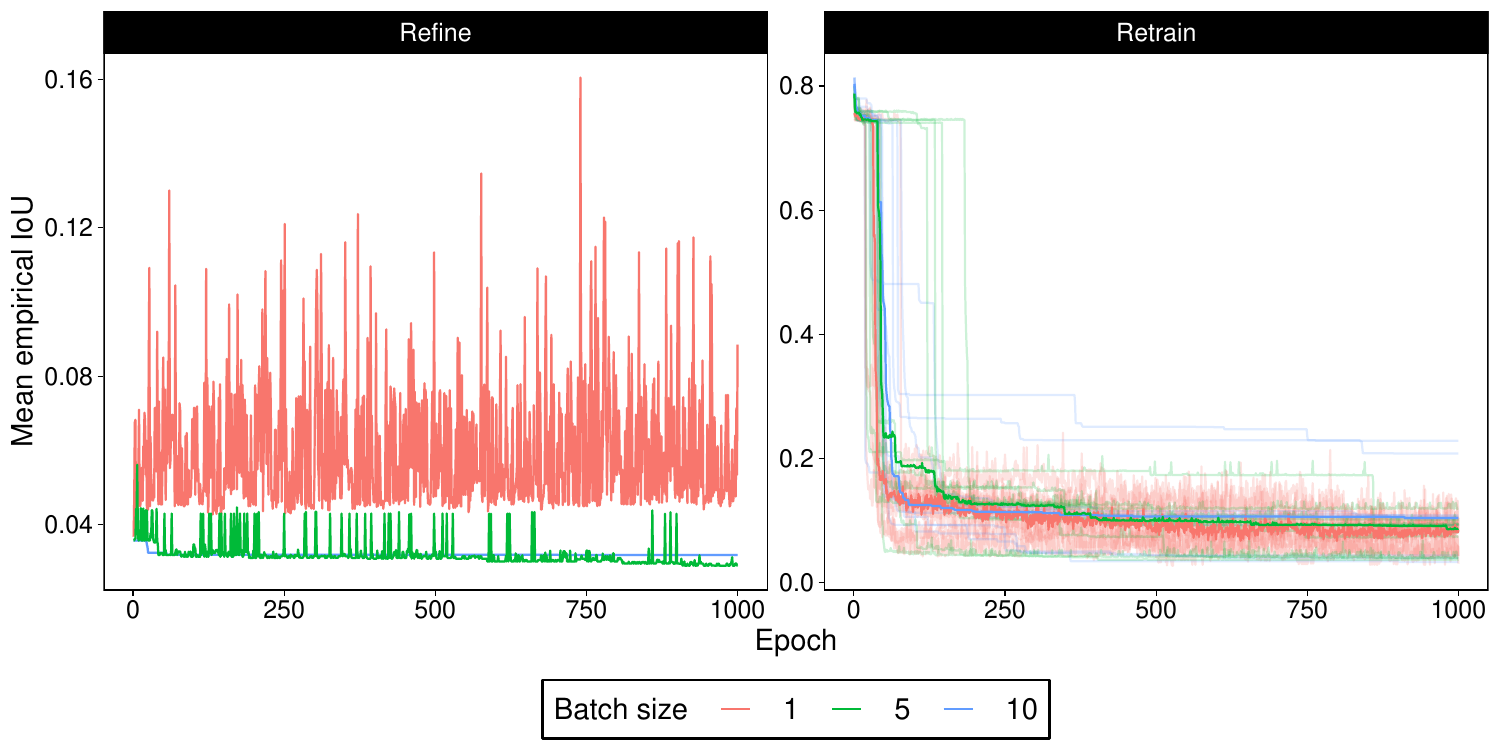}
	\caption{\footnotesize Mean empirical IoU at each epoch of the L\deleted{G}DA when fine training and retraining the architecture \deleted{ASF3-16SG3-16SG3} \added{ASF3-8SG3-8SG3} for each batch size. \added{The transparent lines represent the repetitions from distinct initial points and the solid lines represent the median among the repetitions.}} \label{fig_SLDA_refine_digits}
\end{figure}

The image predicted by the fine trained CDMNN \deleted{ASF3-16SG3-16SG3} \added{ASF3-8SG3-8SG3} with a batch size of \deleted{$10$} \added{five} for each input image in the validation sample is presented in Figure \ref{fig_SLDA_best_digits}. As expected from prior information about the problem at hand, an architecture that applies an ASF at the first layer should have a better performance than architectures that do not, since an ASF has the property of removing noise. We see in Figure \ref{fig_SLDA_best_digits} that the first layer of the trained CDMNN is indeed removing the noise. After this noise removal, the second layer, given by a W-operator represented by the supremum of \added{eight} \deleted{16} sup-generating operators, is \added{actually already recognizing the boundary of the digit.} \deleted{crudely recognizing the boundary.} The third layer, also given by a W-operator represented by the supremum of \added{eight} \deleted{16} sup-generating operators, \added{does not have a meaningful effect since the boundary had already been recognized by the second layer after the first one removed the noise.} \deleted{is refining the boundary of the digit roughly recognized by the first W-operator, obtaining a good visual result.}

\begin{figure}[ht]
	\centering
	\includegraphics[width=\linewidth]{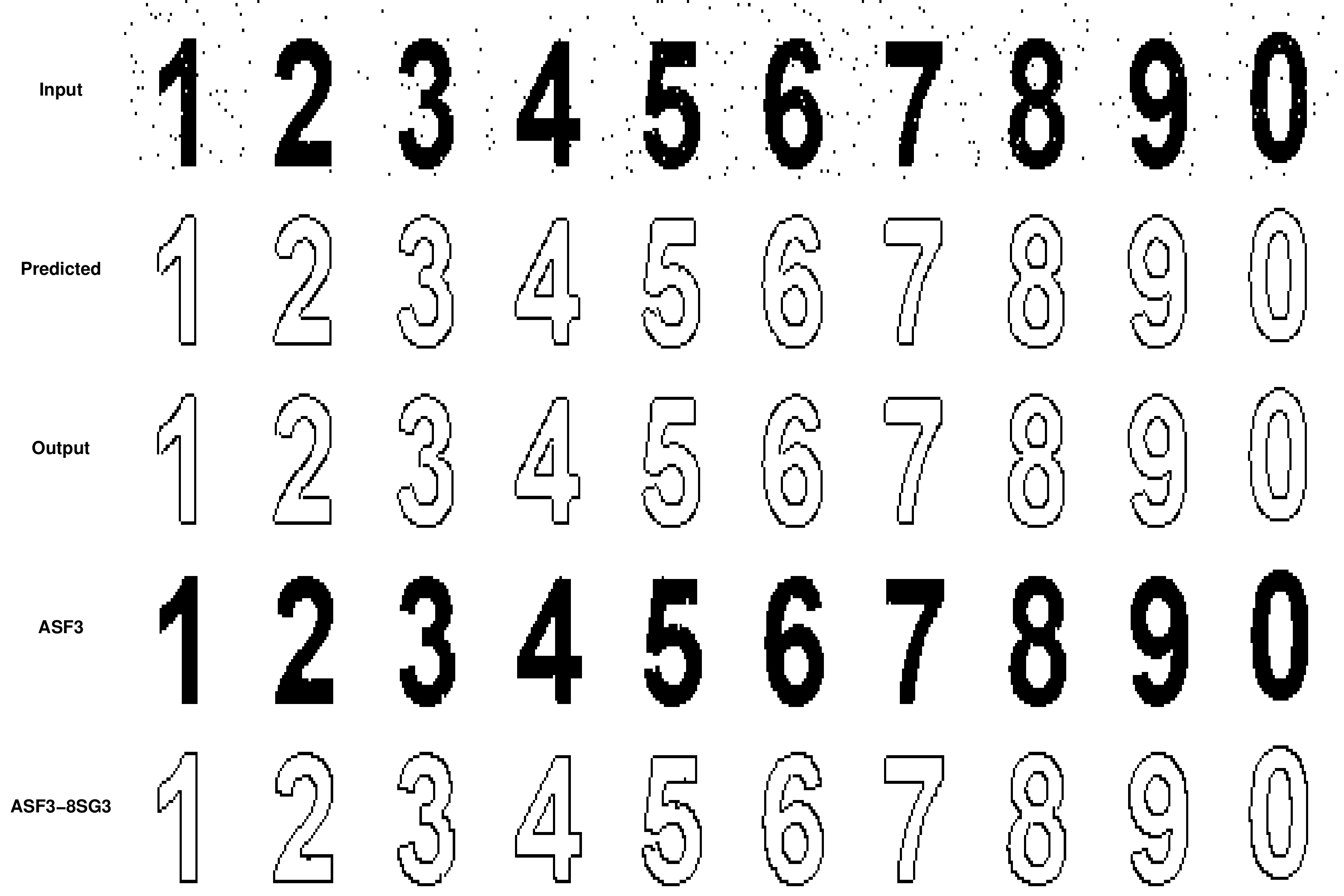}
	\caption{\footnotesize Input and output images in the validation sample, images predicted by the fine trained \deleted{ASF3-16SG3-16SG3} \added{ASF3-8SG3-8SG3} architecture, image obtained after applying the ASF3 layer of this architecture and image obtained after applying the \deleted{ASF3-16SG3} \added{ASF3-8SG3} layers of this architecture.} \label{fig_SLDA_best_digits}
\end{figure}

\section{Discussion}
\label{Sec_Discussion}

\subsection{Main contributions}

The method proposed in this paper merges the classical heuristic design of morphological operators with the estimation of parameters via machine learning. The method proposes the definition of a class of W-operators through a DMNN architecture and the training of its parameters through a discrete stochastic combinatorial optimization algorithm analogous to the stochastic gradient descent algorithm. An important characteristic of this approach is that the class of W-operators can be quite general (e.g., all W-operators with at most $k$ elements in the basis) or specialized (e.g., filters and granulometric filters) according to expected properties of the target operator, i.e., prior information. The semantic expressed by algebraic properties of classes of operators is a \deleted{differential} \added{distinctive feature} relative to other methods, such as deep neural networks. Some benefits of the proposed approach are: smaller sample complexity, interpretation of the designed operator and smaller computational complexity (due to constraints on the class of operators and the stochasticity of the training algorithm). The semantics of the designed operator being comprehensible open doors for further research in the context of DMNN.

More specifically, this paper formalizes the DMNN and proposes lattice \deleted{gradient} descent algorithms to train CDMNN. Although CDMNN are a special case of classical designs of set operators, that combine erosions and dilations via composition, infimum, supremum and complement, its formalization enabled a general algorithm to learn W-operators from data via a greedy search of a lattice. Furthermore, the CDMNN are transparent by design and fully interpretable, since the basis of the operators realized by them can be computed and from it their properties can be deduced. Finally, prior information about the problem at hand can be inserted into a CDMNN architecture to obtain better results and to account for features of the domain of application. 

\subsection{Algorithmic and application perspectives}

We have not exhausted the study of DMNN for binary image analysis, but rather presented its main features, and we leave many important questions for future research. First, the L\deleted{G}DA and SL\deleted{G}DA are not exclusive to train CDMNN and can be readily generalized to any DMNN that is parameterized by a vector of sets and intervals, or any other object contained in a lattice or a poset. \added{For instance, a recent work \cite{marcondes2023algorithm} has proposed a LDA and SLDA to train unconstrained sequential DMNN (cf. $\mathcal{A}_{1}$ in Figure \ref{fig_sequential_architecture})}. Nevertheless, the complexity of the algorithm may increase when $(\mathcal{F},\leq)$ does not have a lattice structure, as can be the case of restricted CDMNN, since the computation of the neighborhoods $\mathcal{N}(\mathcal{C})$ may be costly. An interesting topic for future research is, based on prior information and domain knowledge, to design specific DMNN to solve problems of interest and implement efficient algorithms based on the L\deleted{G}DA for their training. 

An efficient implementation of the general L\deleted{G}DA and SL\deleted{G}DA in Algorithms \ref{A1} and \ref{A2} to train CDMNN is also a topic for future research. The implementation of an algorithm to calculate the window and the basis of a MCG, as discussed in Remark \ref{remark_basis}, is also necessary. Once an efficient implementation is available, it would be interesting to further study empirical features illustrated in the application. For instance, it is necessary to better understand the effect of the architecture complexity and the batch size on the generalization quality of the trained CDMNN in order to build optimized architectures and properly select hyperparameters such as the number of epochs and batch size. 

\added{The numerical experiments evidenced that some stochasticity increases the performance of the algorithm and may avoid getting stuck in bad local minima. On the one hand, the SLDA was able to attain the best minima with a batch size five, while with a batch size one it could not attain these minima, even though the algorithm was less sensible to the initial point. On the other hand, the best results for training with the LDA were obtained with a batch size of one evidencing that a batch size of five is not enough stochasticity when there is no neighbor sampling. These preliminary findings need to be better explored and the performance may be improved by developing more sophisticated methods to select the initial point.}

\added{We also leave for future studies the comparison of the DMNN with other methods for binary image processing. An important comparison is from the point of view of control and interpretability which is the aspect of DMNN that mostly differentiate it from usual methods based on MNN. For instance, with DMNN it is possible to constraint the class of operators based on domain knowledge and to mathematically deduce the properties of the learned operator, and it would be interesting to assess if this can also be done with other methods. Furthermore, it would be interesting to investigate if it is possible to more efficiently train DMNN by developing hybrid algorithms which combine lattice searches with continuous optimization algorithms for MNN such as that proposed in \cite{groenendijk2023geometric}.}

\subsection{Possible extensions}

We believe the DMNN could be extended to represent a broader class of morphological operators in more general complete lattices. For instance, the canonical decomposition results of \cite{banon1991minimal} were extended by \cite{banon1993decomposition} for general complete lattices and could in theory be the foundation of a DMNN architecture. In this instance, it may be necessary to combine the L\deleted{G}DA with a continuous optimization algorithm \added{(e.g., algorithms to train MNN \cite{dimitriadis2021advances,groenendijk2023geometric})} since the operators may not be represented by parameters in a lattice, but rather by a set of parameters that can be decomposed into a lattice of subsets of parameters. For example, for gray-scale images, the parameters of a vertex may be the window on which its operator is locally defined and the function representing the structural element whose domain is the window. The L\deleted{G}DA could be applied to select the window, while another algorithm would have to learn the structural element. The extension of DMNN to gray-scale images is a promising line of research which we \deleted{pretend to follow} \added{are currently working on}.

The main bottlenecks to extending DMNN to solve more general image processing problems are the definition of a lattice algebra under which the problem can be solved via a lattice operator and the development of practical algorithms to train DMNN under this lattice algebra. \added{These bottlenecks are further discussed in \cite{marcondes2023paradigm}.} For example, if one could develop a lattice algebra for RGB images (see \cite{angulo2003mathematical,pastore2016new,sun2023order} for examples of possible algebras) and design image transformation operators as lattice operators, then the results of \cite{banon1993decomposition} could be applied to define the architectures of a DMNN, and if one could develop algorithms to train them, then they could be applied to solve state-of-the-art problems. Such an extension seems a giant leap now, but having a controllable and interpretable analogue of deep learning for general image transformation tasks would be a major advance and even if MM methods cannot attain this degree of generalization, some insights and important results may be achieved along the way. Such a possibility is a motive in itself to further study MM and DMNN extensions.

\subsection{Prior information and Neural Architecture Search}

The quality of a DMNN is associated with the correct design of its architecture, in the sense of it representing operators that approximate a target operator, and with the efficiency of its training from both a computational and sample complexity perspective. Therefore, from an algebraic perspective, it is necessary to better understand how desired classes of operators can be efficiently represented by unrestricted CDMNN architectures since they can be more efficiently trained. Furthermore, it is necessary to properly insert prior information into the architecture, so it represents a collection of operators \textit{compatible} with the problem being solved, on which it is statistically possible, in the sense of statistical learning theory \cite{vapnik1998}, to learn with the available sample size.

However, when strong prior information is not available, it may not be possible to properly design a DMNN architecture, and \added{it} should be learned from data. A state-of-the-art topic in deep learning is architecture selection, or Neural Architecture Search (NAS) \cite{elsken2019neural}, whose methods are mainly based on computation demanding procedures to optimize the architecture based on data within an automatic machine learning (AutoML) paradigm \cite{he2021automl}. Such an automatic approach would be incompatible with MM, since it does not take into account prior information and the properties of morphological operators. 

A true MM method for architecture selection should, based on prior information about the problem at hand, consider a set of candidate DMNN architectures and then select one of them based on data. In some instances, the available prior information may not be enough to properly build one architecture, but may suffice to design a collection of candidate architectures that should be evaluated on data to select the \textit{best} one. These architectures should generate collections of operators with increasing complexity, and selecting them should also be a mean of avoiding over-fitting and controlling the generalization quality of the learned architecture. If these candidate architectures could be themselves organized in a lattice, then they could be selected by a L\deleted{G}DA. \deleted{algorithm.} The study of lattice-based architecture selection methods is an important line of research to control the complexity of the architecture relative to the sample size and obtain better practical results when strong prior information is absent. This is a line of research we are currently following \added{and the method proposed in \cite{marcondes2023algorithm} is an example of DMNN architecture selection.}

\subsection{Impact to Mathematical Morphology}

Mathematical Morphology is a classical paradigm for representing non-linear lattice transformations and had been a state-of-the-art method for image analysis, from the 1960s until recently, when it was overshadowed by the success of deep learning techniques. However, even though established MM methods cannot compete with deep learning from a performance point-of-view, MM has the great advantage of being a mathematically grounded image processing theory from which one can build methods that are completely controllable and interpretable, characteristics that deep learning lacks in general. Therefore, we argue, resonating the arguments of \cite{angulo2021some}, that one should not disregard deep learning techniques, that are extremely useful and have many qualities, and expect to obtain state-of-the-art methods based solely on classical MM methods, but should rather rethink MM in view of modern learning techniques. The method proposed in this paper follows this research agenda.

We believe that the DMNN is a \textit{true} MM method since it retains the control over the design and interpretability of results of classical MM methods, which is an advantage over CNN. For instance, even though CNN can easily solve virtually any problem of binary image analysis, they cannot, in general, be controlled to represent a restricted class of operators based on prior information or be interpreted to understand the properties of the learned operator. Both these tasks are natural to DMNN, since its architecture can be specified via the MM toolbox to represent any class of operators, and, once it is trained, the basis of the operator can be computed and its properties can be fully known. Therefore, the DMNN are better than CNN from a semantic point of view. Since, from an algorithmic perspective, the SL\deleted{G}DA makes the training of DMNN scalable, they are a viable mathematically grounded practical method which ought to be further explored.

Besides providing a path to bring MM to the deep learning era, this paper also merges two, sometimes competing, paradigms of MM: heuristic design and automatic design via machine learning. The advocates of the heuristic approach have argued that the machine learning approach did not take into account strong prior information and did not have as much control over the operator design, while the advocates of the machine learning paradigm have argued that machine learning methods were necessary to scale MM applications and obtain good practical results via a less time-consuming approach. In hindsight, we believe that both arguments are valid and hope with the DMNN approach to unite these paradigms and help to regain the relevance of MM in the deep learning era.

\section*{Acknowledgments}

D. Marcondes was funded by grants \#22/06211-2 and \#23/00256-7, São Paulo Research Foundation (FAPESP), and J. Barrera was funded by grants \#14/50937-1 and \#2020/06950-4, São Paulo Research Foundation (FAPESP).

\bibliographystyle{plain}
\bibliography{Ref}

\begin{thebibliography}{10}

\bibitem{angulo2021some}
Jesus Angulo.
\newblock Some open questions on morphological operators and representations in
  the deep learning era: A personal vision.
\newblock In {\em Discrete Geometry and Mathematical Morphology: First
  International Joint Conference, DGMM 2021, Uppsala, Sweden, May 24--27, 2021,
  Proceedings}, pages 3--19. Springer, 2021.

\bibitem{angulo2003mathematical}
Jes{\'u}s Angulo and Jean Serra.
\newblock Mathematical morphology in color spaces applied to the analysis of
  cartographic images.
\newblock {\em Proceedings of GEOPRO}, 3:59--66, 2003.

\bibitem{araujo2006modular}
R~de~A Ara{\'u}jo, Francisco Madeiro, Robson~P de~Sousa, and L{\'u}cio~FC
  Pessoa.
\newblock Modular morphological neural network training via adaptive genetic
  algorithm for designing translation invariant operators.
\newblock In {\em 2006 IEEE International Conference on Acoustics Speech and
  Signal Processing Proceedings}, volume~2, pages II--II. IEEE, 2006.

\bibitem{araujo2017morphological}
Ricardo de~A Ara{\'u}jo, Adriano~LI Oliveira, and Silvio Meira.
\newblock A morphological neural network for binary classification problems.
\newblock {\em Engineering Applications of Artificial Intelligence}, 65:12--28,
  2017.

\bibitem{arce2018differential}
Fernando Arce, Erik Zamora, Humberto Sossa, and Ricardo Barr{\'o}n.
\newblock Differential evolution training algorithm for dendrite morphological
  neural networks.
\newblock {\em Applied Soft Computing}, 68:303--313, 2018.

\bibitem{u-curve3}
Esmaeil Atashpaz-Gargari, Marcelo~S Reis, Ulisses~M Braga-Neto, Junior Barrera,
  and Edward~R Dougherty.
\newblock A fast branch-and-bound algorithm for {U}-curve feature selection.
\newblock {\em Pattern Recognition}, 73:172--188, 2018.

\bibitem{banon1991minimal}
Gerald Jean~Francis Banon and Junior Barrera.
\newblock Minimal representations for translation-invariant set mappings by
  mathematical morphology.
\newblock {\em SIAM Journal on Applied Mathematics}, 51(6):1782--1798, 1991.

\bibitem{banon1993decomposition}
Gerald Jean~Francis Banon and Junior Barrera.
\newblock Decomposition of mappings between complete lattices by mathematical
  morphology, {P}art {I}. general lattices.
\newblock {\em Signal Processing}, 30(3):299--327, 1993.

\bibitem{barrera1998mmach}
Junior Barrera, Gerald Jean~Franc Banon, Roberto~Alencar Lotufo, and Roberto
  Hirata~Jr.
\newblock {M}{M}ach: a mathematical morphology toolbox for the {K}horos system.
\newblock {\em Journal of Electronic Imaging}, 7(1):174--210, 1998.

\bibitem{barrera1997automatic}
Junior Barrera, Edward~R Dougherty, and Nina~Sumiko Tomita.
\newblock Automatic programming of binary morphological machines by design of
  statistically optimal operators in the context of computational learning
  theory.
\newblock {\em Journal of Electronic Imaging}, 6(1):54--67, 1997.

\bibitem{barrera2022mathematical}
Junior Barrera, Ronaldo~F Hashimoto, Nina~ST Hirata, R~Hirata~Jr, and Marcelo~S
  Reis.
\newblock From mathematical morphology to machine learning of image operators.
\newblock {\em S{\~a}o Paulo Journal of Mathematical Sciences}, 16(1):616--657,
  2022.

\bibitem{barrera1996set}
Junior Barrera and Guillermo~Pablo Salas.
\newblock Set operations on closed intervals and their applications to the
  automatic programming of morphological machines.
\newblock {\em Journal of Electronic Imaging}, 5(3):335--352, 1996.

\bibitem{barrera2000automatic}
Junior Barrera, Routo Terada, Roberto Hirata~Jr, and Nina~ST Hirata.
\newblock Automatic programming of morphological machines by {PAC} learning.
\newblock {\em Fundamenta Informaticae}, 41(1-2):229--258, 2000.

\bibitem{beucher1982watersheds}
Serge Beucher.
\newblock Watersheds of functions and picture segmentation.
\newblock In {\em ICASSP'82. IEEE International Conference on Acoustics,
  Speech, and Signal Processing}, volume~7, pages 1928--1931. IEEE, 1982.

\bibitem{brun2003design}
Marcel Brun, Edward~R Dougherty, Roberto Hirata~Jr, and Junior Barrera.
\newblock Design of optimal binary filters under joint
  multiresolution--envelope constraint.
\newblock {\em Pattern recognition letters}, 24(7):937--945, 2003.

\bibitem{brun2004nonlinear}
Marcel Brun, Roberto Hirata, Junior Barrera, and Edward~R Dougherty.
\newblock Nonlinear filter design using envelopes.
\newblock {\em Journal of Mathematical Imaging and Vision}, 21:81--97, 2004.

\bibitem{mmand}
Jon Clayden.
\newblock {\em mmand: Mathematical Morphology in Any Number of Dimensions},
  2023.
\newblock R package version 1.6.3.

\bibitem{davidson1992simulated}
Jennifer~L Davidson.
\newblock Simulated annealing and morphology neural networks.
\newblock In {\em Image Algebra and Morphological Image Processing III}, volume
  1769, pages 119--127. SPIE, 1992.

\bibitem{davidson1993morphology}
Jennifer~L Davidson and Frank Hummer.
\newblock Morphology neural networks: An introduction with applications.
\newblock {\em Circuits, Systems and Signal Processing}, 12(2):177--210, 1993.

\bibitem{davidson1990theory}
Jennifer~L Davidson and Gerhard~X Ritter.
\newblock Theory of morphological neural networks.
\newblock In {\em Digital Optical Computing II}, volume 1215, pages 378--388.
  SPIE, 1990.

\bibitem{de2000designing}
Robson~P de~Sousa, Jo{\~a}o~Marques de~Carvalho, Francisco~M de~Assis, and
  L{\'u}cio~FC Pessoa.
\newblock Designing translation invariant operations via neural network
  training.
\newblock In {\em Proceedings 2000 International Conference on Image Processing
  (Cat. No. 00CH37101)}, volume~1, pages 908--911. IEEE, 2000.

\bibitem{dellamonica2007exact}
Domingos Dellamonica, Paulo~JS Silva, Carlos Humes, Nina~ST Hirata, and Junior
  Barrera.
\newblock An exact algorithm for optimal {MAE} stack filter design.
\newblock {\em IEEE transactions on image processing}, 16(2):453--462, 2007.

\bibitem{dimitriadis2021advances}
Nikolaos Dimitriadis and Petros Maragos.
\newblock Advances in morphological neural networks: training, pruning and
  enforcing shape constraints.
\newblock In {\em ICASSP 2021-2021 IEEE International Conference on Acoustics,
  Speech and Signal Processing (ICASSP)}, pages 3825--3829. IEEE, 2021.

\bibitem{dougherty2003hands}
Edward~R Dougherty and Roberto~A Lotufo.
\newblock {\em Hands-on morphological image processing}, volume~59.
\newblock SPIE press, 2003.

\bibitem{elsken2019neural}
Thomas Elsken, Jan~Hendrik Metzen, and Frank Hutter.
\newblock Neural architecture search: A survey.
\newblock {\em The Journal of Machine Learning Research}, 20(1):1997--2017,
  2019.

\bibitem{ucurveParallel}
Gustavo Estrela, Marco~Dimas Gubitoso, Carlos~Eduardo Ferreira, Junior Barrera,
  and Marcelo~S Reis.
\newblock An efficient, parallelized algorithm for optimal conditional
  entropy-based feature selection.
\newblock {\em Entropy}, 22(4):492, 2020.

\bibitem{franchi2020deep}
Gianni Franchi, Amin Fehri, and Angela Yao.
\newblock Deep morphological networks.
\newblock {\em Pattern Recognition}, 102:107246, 2020.

\bibitem{grana2001some}
Manuel Grana and Bogdan Raducanu.
\newblock Some applications of morphological neural networks.
\newblock In {\em IJCNN'01. International Joint Conference on Neural Networks.
  Proceedings (Cat. No. 01CH37222)}, volume~4, pages 2518--2523. IEEE, 2001.

\bibitem{groenendijk2022morphpool}
Rick Groenendijk, Leo Dorst, and Theo Gevers.
\newblock Morphpool: Efficient non-linear pooling \& unpooling in {CNNs}.
\newblock {\em arXiv preprint arXiv:2211.14037}, 2022.

\bibitem{groenendijk2023geometric}
Rick Groenendijk, Leo Dorst, and Theo Gevers.
\newblock Geometric back-propagation in morphological neural networks.
\newblock {\em IEEE Transactions on Pattern Analysis and Machine Intelligence},
  2023.

\bibitem{he2021automl}
Xin He, Kaiyong Zhao, and Xiaowen Chu.
\newblock {AutoML}: A survey of the state-of-the-art.
\newblock {\em Knowledge-Based Systems}, 212:106622, 2021.

\bibitem{hirata2008multilevel}
Nina~ST Hirata.
\newblock Multilevel training of binary morphological operators.
\newblock {\em IEEE Transactions on pattern analysis and machine intelligence},
  31(4):707--720, 2008.

\bibitem{hirata1999design}
Nina~ST Hirata, Junior Barrera, and Edward~R Dougherty.
\newblock Design of statistically optimal stack filters.
\newblock In {\em XII Brazilian Symposium on Computer Graphics and Image
  Processing (Cat. No. PR00481)}, pages 265--274. IEEE, 1999.

\bibitem{hirata2021machine}
Nina~ST Hirata and George~A Papakostas.
\newblock On machine-learning morphological image operators.
\newblock {\em Mathematics}, 9(16):1854, 2021.

\bibitem{hirata2002incremental}
Nina Sumiko~Tomita Hirata, J{\'u}nior Barrera, Routo Terada, Edward~R
  Dougherty, H~Talbot, and R~Beare.
\newblock The incremental splitting of intervals algorithm for the design of
  binary image operators.
\newblock {\em Proceedings of the 6th ISMM}, pages 219--228, 2002.

\bibitem{hirata2000iterative}
Nina Sumiko~Tomita Hirata, Edward~R Dougherty, and Junior Barrera.
\newblock Iterative design of morphological binary image operators.
\newblock {\em Optical Engineering}, 39(12):3106--3123, 2000.

\bibitem{hirata2002segmentation}
Roberto Hirata~Jr, Junior Barrera, Ronaldo~F Hashimoto, Daniel~O Dantas, and
  Gustavo~H Esteves.
\newblock Segmentation of microarray images by mathematical morphology.
\newblock {\em Real-Time Imaging}, 8(6):491--505, 2002.

\bibitem{hirata2000aperture}
Roberto Hirata~Jr, Edward~R Dougherty, and Junior Barrera.
\newblock Aperture filters.
\newblock {\em Signal processing}, 80(4):697--721, 2000.

\bibitem{hu2022learning}
Yufei Hu, Nacim Belkhir, Jesus Angulo, Angela Yao, and Gianni Franchi.
\newblock Learning deep morphological networks with neural architecture search.
\newblock {\em Pattern Recognition}, 131:108893, 2022.

\bibitem{julca2017image}
Frank~D Julca-Aguilar and Nina~ST Hirata.
\newblock Image operator learning coupled with {CNN} classification and its
  application to staff line removal.
\newblock In {\em 2017 14th IAPR International Conference on Document Analysis
  and Recognition (ICDAR)}, volume~1, pages 53--58. IEEE, 2017.

\bibitem{marcondes2023paradigm}
Diego Marcondes and Junior Barrera.
\newblock The lattice overparametrization paradigm for the machine learning of
  lattice operators.
\newblock {\em arXiv preprint arXiv:2310.06639}, 2023.

\bibitem{marcondes2023algorithm}
Diego Marcondes, Mariana Feldman, and Junior Barrera.
\newblock An algorithm to train unrestricted sequential discrete morphological
  neural networks.
\newblock {\em arXiv preprint arXiv:2310.04584}, 2023.

\bibitem{matheron1974random}
Georges Matheron.
\newblock {\em Random sets and integral geometry}.
\newblock John Wiley \& Sons, 1974.

\bibitem{matheron}
Georges Matheron.
\newblock {\em Random sets and integral geometry}.
\newblock John Wiley \& Sons, 1975.

\bibitem{mondal2019morphological}
Ranjan Mondal, Pulak Purkait, Sanchayan Santra, and Bhabatosh Chanda.
\newblock Morphological networks for image de-raining.
\newblock In {\em Discrete Geometry for Computer Imagery: 21st IAPR
  International Conference, DGCI 2019, Marne-la-Vall{\'e}e, France, March
  26--28, 2019, Proceedings 21}, pages 262--275. Springer, 2019.

\bibitem{mondal2019morphological2}
Ranjan Mondal, Sanchayan Santra, Soumendu~Sundar Mukherjee, and Bhabatosh
  Chanda.
\newblock Morphological network: How far can we go with morphological neurons?
\newblock {\em arXiv preprint arXiv:1901.00109}, 2019.

\bibitem{montagner2017staff}
Igor~S Montagner, Nina~ST Hirata, and Roberto Hirata~Jr.
\newblock Staff removal using image operator learning.
\newblock {\em Pattern Recognition}, 63:310--320, 2017.

\bibitem{montagner2016kernel}
Igor~S Montagner, Roberto Hirata, Nina~ST Hirata, and St{\'e}phane Canu.
\newblock Kernel approximations for {W}-operator learning.
\newblock In {\em 2016 29th SIBGRAPI Conference on Graphics, Patterns and
  Images (SIBGRAPI)}, pages 386--393. IEEE, 2016.

\bibitem{monteiro2008brief}
Alexandre Monteiro~da Silva and Peter Sussner.
\newblock A brief review and comparison of feedforward morphological neural
  networks with applications to classification.
\newblock In {\em Artificial Neural Networks-ICANN 2008: 18th International
  Conference, Prague, Czech Republic, September 3-6, 2008, Proceedings, Part II
  18}, pages 783--792. Springer, 2008.

\bibitem{pastore2016new}
Juan Pastore, Agustina Bouchet, Marcel Brun, and Virginia Ballarin.
\newblock New windows based color morphological operators for biomedical image
  processing.
\newblock In {\em Journal of Physics: Conference Series}, volume 705, page
  012023. IOP Publishing, 2016.

\bibitem{pessoa1996morphological}
Lucio~FC Pessoa and Petros Maragos.
\newblock Morphological/rank neural networks and their adaptive optimal design
  for image processing.
\newblock In {\em 1996 IEEE International Conference on Acoustics, Speech, and
  Signal Processing Conference Proceedings}, volume~6, pages 3398--3401. IEEE,
  1996.

\bibitem{pessoa2000neural}
Lucio~FC Pessoa and Petros Maragos.
\newblock Neural networks with hybrid morphological/rank/linear nodes: a
  unifying framework with applications to handwritten character recognition.
\newblock {\em Pattern Recognition}, 33(6):945--960, 2000.

\bibitem{R}
{R Core Team}.
\newblock {\em {R}: A Language and Environment for Statistical Computing}.
\newblock R Foundation for Statistical Computing, Vienna, Austria, 2021.

\bibitem{reis2018}
Marcelo~S Reis, Gustavo Estrela, Carlos~Eduardo Ferreira, and Junior Barrera.
\newblock Optimal {B}oolean lattice-based algorithms for the {U}-curve
  optimization problem.
\newblock {\em Information Sciences}, 2018.

\bibitem{u-curve1}
Marcelo Ris, Junior Barrera, and David~C Martins.
\newblock {U}-curve: A branch-and-bound optimization algorithm for {U}-shaped
  cost functions on {B}oolean lattices applied to the feature selection
  problem.
\newblock {\em Pattern Recognition}, 43(3):557--568, 2010.

\bibitem{ritter2003morphological}
Gerhard~X Ritter, Laurentiu Iancu, and Gonzalo Urcid.
\newblock Morphological perceptrons with dendritic structure.
\newblock In {\em The 12th IEEE International Conference on Fuzzy Systems,
  2003. FUZZ'03.}, volume~2, pages 1296--1301. IEEE, 2003.

\bibitem{ritter1996introduction}
Gerhard~X Ritter and Peter Sussner.
\newblock An introduction to morphological neural networks.
\newblock In {\em Proceedings of 13th International Conference on Pattern
  Recognition}, volume~4, pages 709--717. IEEE, 1996.

\bibitem{ruder2016overview}
Sebastian Ruder.
\newblock An overview of gradient descent optimization algorithms.
\newblock {\em arXiv preprint arXiv:1609.04747}, 2016.

\bibitem{santos2010information}
Carlos~S Santos, Nina~ST Hirata, and Roberto Hirata.
\newblock An information theory framework for two-stage binary image operator
  design.
\newblock {\em Pattern Recognition Letters}, 31(4):297--306, 2010.

\bibitem{schmitt1989mathematical}
M~Schmitt.
\newblock Mathematical morphology and artificial intelligence: an automatic
  programming system.
\newblock {\em Signal Processing}, 16(4):389--401, 1989.

\bibitem{serra1983image1}
Jean Serra.
\newblock {\em Image Analysis and Mathematical Morphology. Volume 1}.
\newblock Academic Press, Inc., 1984.

\bibitem{serra1983image}
Jean Serra.
\newblock {\em Image Analysis and Mathematical Morphology. Volume 2:
  Theoretical Advances}.
\newblock Academic Press, Inc., 1988.

\bibitem{serra1992overview}
Jean Serra and Luc Vincent.
\newblock An overview of morphological filtering.
\newblock {\em Circuits, Systems and Signal Processing}, 11:47--108, 1992.

\bibitem{sossa2014efficient}
Humberto Sossa and Elizabeth Guevara.
\newblock Efficient training for dendrite morphological neural networks.
\newblock {\em Neurocomputing}, 131:132--142, 2014.

\bibitem{sun2023order}
Shanqian Sun, Yunjia Huang, Kohei Inoue, and Kenji Hara.
\newblock Order space-based morphology for color image processing.
\newblock {\em Journal of Imaging}, 9(7):139, 2023.

\bibitem{sussner2009constructive}
Peter Sussner and Estevao~Laureano Esmi.
\newblock Constructive morphological neural networks: some theoretical aspects
  and experimental results in classification.
\newblock {\em Constructive neural networks}, pages 123--144, 2009.

\bibitem{van1996fast}
Marc Van~Droogenbroeck and Hugues Talbot.
\newblock Fast computation of morphological operations with arbitrary
  structuring elements.
\newblock {\em Pattern recognition letters}, 17(14):1451--1460, 1996.

\bibitem{vapnik1998}
Vladimir Vapnik.
\newblock {\em Statistical learning theory. 1998}, volume~3.
\newblock Wiley, New York, 1998.

\bibitem{vincent1991efficient}
Luc~M Vincent.
\newblock Efficient computation of various types of skeletons.
\newblock In {\em Medical Imaging V: Image Processing}, volume 1445, pages
  297--311. SPIE, 1991.

\end{thebibliography}

\end{document}